\documentclass{article}

\usepackage[round]{natbib}

\usepackage[utf8]{inputenc} 
\usepackage[T1]{fontenc}    
\usepackage{hyperref}       
\usepackage{url}            
\usepackage{booktabs}       
\usepackage{amsfonts}       
\usepackage{nicefrac}       
\usepackage{microtype}      

\title{No-regret Algorithms for Multi-task Bayesian Optimization}

\usepackage{multicol}
\usepackage[utf8]{inputenc}
\usepackage{times}
\usepackage{graphicx} 
\usepackage{subfigure} 
\usepackage{amsmath}
\usepackage{algorithm}
\usepackage{algorithmic}
\usepackage{amssymb,amsmath,amsfonts}
\usepackage{mathrsfs}
\usepackage{color}
\usepackage{dsfont}
\newcommand{\norm}[1]{\left\lVert#1\right\rVert}
\newcommand{\expect}[1]{\mathbb{E}\left[{#1}\right]}
\newcommand{\prob}[1]{\mathbb{P}\left[{#1}\right]}
\newcommand{\given}{\; \big\vert \;} 
\newcommand{\bydef}{:=}
\newcommand{\inner}[2]{\langle #1, #2 \rangle}

\usepackage{subfigure}
\usepackage{mathnotations2}
\graphicspath{{./Figures/}}

\newtheorem{mytheorem}{Theorem}

\newtheorem{mylemma}{Lemma}

\newtheorem{remark}{Remark}
\newcommand{\BlackBox}{\rule{1.5ex}{1.5ex}}
\newenvironment{proof}{\par\noindent{\bf Proof\ }}{\hfill\BlackBox\\[2mm]}
\newcommand{\beq}{\begin{equation}}
\newcommand{\eeq}{\end{equation}}
\newcommand{\beqn}{\begin{equation*}}
\newcommand{\eeqn}{\end{equation*}}
\newcommand{\beqa}{\begin{eqnarray}}
\newcommand{\eeqa}{\end{eqnarray}}
\newcommand{\beqan}{\begin{eqnarray*}}
\newcommand{\eeqan}{\end{eqnarray*}}

\newcommand{\argmax}{\mathop{\mathrm{argmax}}}

\DeclareMathOperator{\Tr}{trace}

\DeclareMathOperator{\Dg}{diag}

\author{
  Sayak Ray Chowdhury \hspace{10pt} Aditya Gopalan\\
  Indian Institute of Science
  }

\begin{document}

\maketitle

\begin{abstract}

We consider multi-objective optimization (MOO) of an unknown vector-valued function in the non-parametric Bayesian optimization (BO) setting, with the aim being to learn points on the Pareto front of the objectives. Most existing BO algorithms do not model the fact that the multiple objectives, or equivalently, tasks can share similarities, and even the few that do lack rigorous, finite-time regret guarantees that capture explicitly inter-task structure. In this work, we address this problem by modelling inter-task dependencies using a multi-task kernel and develop two novel BO algorithms based on random scalarizations of the objectives. Our algorithms employ vector-valued kernel regression as a stepping stone and belong to the upper confidence bound class of algorithms. Under a smoothness assumption that the unknown vector-valued function is an element of the reproducing kernel Hilbert space associated with the multi-task kernel, we derive worst-case regret bounds for our algorithms that explicitly capture the similarities between tasks. We numerically benchmark our algorithms on both synthetic and real-life MOO problems, and show the advantages offered by learning with multi-task kernels.
\end{abstract}

\section{Introduction}
\label{sec:intro}

Bayesian optimization is a popular approach for optimizing a black-box function with expensive, noisy evaluations, having been extensively applied in various applications such as hyper-parameter tuning \citep{snoek2012practical}, sensor selection \citep{GarOsbRob10:BOsensor}, synthetic gene design \citep{gonzalez2015bayesian}, etc. In many practical scenarios, one is required to optimize {\em multiple} objectives together, and moreover, these objectives
can be conflicting in nature. 
For example, consider drug discovery, where each function evaluation is a costly laboratory experiment and its output is a measurement of both the potency and side-effects of a candidate drug \citep{Paria2019AFF}. These two objectives are typically conflicting in nature, since one would like to maximize the potency of drug while also keeping its side-effects to a minimum. Other examples include tradeoffs such as bias and variance, accuracy and calibration \citep{guo2017calibration}, accuracy and fairness \citep{zliobaite2015relation} etc. These problems can be framed as that of
optimizing a vector-valued function $f=(f_1,\ldots,f_n)$, where each of its components is a real-valued function and corresponds to a particular objective or task.
Since one often cannot optimize all $f_i$'s simultaneously, most multi-objective optimization (MOO) approaches aim to recover the set of {Pareto optimal}
points, where intuitively a point is Pareto optimal if there is no way to improve on all objectives
simultaneously \citep{knowles2006parego,ponweiser2008multiobjective}. Popular BO strategies in this regard include Predictive
Entropy Search \citep{hernandez2016predictive}, max-value entropy search \citep{belakaria2019max}, Pareto active learning \citep{zuluaga2013active}, expected hypervolume improvement \citep{emmerich2008computation}, sequential uncertainty reduction \citep{picheny2015multiobjective} and scalarization based approaches \citep{roijers2013survey}. {Random scalarizations}, in particular, have been shown to be flexible enough to model user preferences in capturing the whole or a part of the \emph{Pareto front} \citep{Paria2019AFF}. 





Most multi-objective BO approaches maintain $n$ different {Gaussian processes} (GPs) \citep{rasmussen2003gaussian}, one for each task or objective $f_i$. However, in general, the tasks share some underlying structure, and cannot be treated as unrelated objects. By making use of this structure, one might benefit significantly by learning the tasks simultaneously as opposed to learning them independently. For example, consider predicting consumer preferences simultaneously based on their past history \citep{evgeniou2005learning}. Each task is to learn the preference of a particular consumer, and the tasks are related since people with similar tastes tend to buy similar items. Other examples include simultaneous estimation of many related indicators in economic forecasting \citep{greene2003econometric}, predicting tumour behaviour from multiple related diseases \citep{rifkin2003analytical} etc. 
However, assuming similarities in a set of tasks and blindly learning them together can be detrimental \citep{caruana1997multitask}. Hence, it is important to have a model that will benefit the learning in case of related tasks and will not hurt performance when the tasks are unrelated. This can be achieved by maintaining a \emph{multi-task} GP over $f$, which directly
induces correlations between tasks \citep{bonilla2008multi}. 
In the context of BO, \citet{swersky2013multi} empirically demonstrate the utility of this model in a number of applications, and \citet{astudillo2019bayesian} provide an asymptotic convergence analysis under a special setting of composite objective functions and noise-free evaluations. However, a formal finite time regret analysis showing the effectiveness of multi-task GPs over independent GPs in the context of noisy MOO has not been rigorously pursued. Against this backdrop, we make the following contributions:
\begin{itemize}

\item We develop two novel BO algorithms -- multi-task kernelized bandits (MT-KB) and multi-task budgeted kernelized bandits (MT-BKB) -- that are based on random scalarizations, and can leverage similarities between tasks to optimize them
more efficiently. 

\item Our algorithms use vector-valued kernel ridge regression as a building block and follow the general template of the upper-confidence-bound class of algorithms. Also, MT-BKB is the first algorithm that employs the \emph{Nystr\"{o}m approximation} in the context of \emph{multi-task kernels}. 

\item Under the assumption that the objective function has smoothness compatible with a joint kernel on its domain and components, we derive (scalarization induced) regret bounds for our algorithms that {\em explicitly capture the inter-task structure}. These are the first worst-case (frequentist) regret bounds for multi-objective BO, and are proved by deriving a novel concentration inequality for the estimate of the vector-valued objective function, which might be of independent interest. 

\item Finally, our algorithms are simple to implement when the kernel decouples between tasks and domain, and we report numerical results on synthetic as well as real-world based
datasets, for which the algorithms are seen to perform favourably.

\end{itemize}

\textbf{Related work.} In the field of geostatistics \citep{wackernagel2013multivariate}, and more recently in supervised learning \citep{liu2018remarks}, multi-task GPs and associated kernels have gained a lot of traction. Also, a lot of work has been done in the context of vector-valued learning with kernel methods \citep{micchelli2005learning,baldassarre2012multi,grunewalder2012conditional}, and this paper complements the literature by considering an online learning setting. A simple version of multi-objective black box optimization -- in the form of online learning in finite multi-armed
bandits (MABs) -- has been considered in \citep{drugan2013designing,drugan2014scalarization}. This paper, in effect, generalize these works to the more challenging setting of infinite-armed  bandits, which has been studied extensively in the single task setting \citep{srinivas2010gaussian,chowdhury2017kernelized,scarlett2017lower}.

\section{Problem statement}

We consider the problem of maximizing a vector-valued function $f(x)=[f_1(x),\ldots,f_n(x)]^{\top}$ over a compact domain $\cX \subset \Real^d$. At each
round $t$, a learner queries $f$ at a single point $x_t \in \cX$, and observes a noisy output $y_t=f(x_t) + \epsilon_t$, where $\epsilon_t \in \Real^n$ is a zero-mean \emph{sub-Gaussian} random vector conditioned on $\cF_{t-1}$, the $\sigma$-algebra generated by the random variables $\lbrace x_s, \epsilon_s\rbrace_{s = 1}^{t-1}$ and $x_t$. By this we mean that there exists a $\sigma \geq 0$, such that
\beqn
\forall \alpha \in \Real^n, \; \forall t \geq 1,\quad \expect{\exp(\alpha^{\top} \epsilon_t) \given \cF_{t-1}} \leq \exp\big(\sigma^2\norm{\alpha}_2^2/2\big)~.
\eeqn
The query point $x_t$ at round $t$ is chosen
causally depending upon the history $\lbrace(x_s,y_s)\rbrace_{s=1}^{t-1}$ of query and output sequences available up to
round $t-1$. Since one cannot optimize all $f_i$'s simultaneously, the learner's aim is to find the set of \emph{Pareto-optimal} points, denoted by $\cX_f$. A point $x$ is said to be Pareto dominated by $x'$ if $f(x) \prec f(x')$, where for any $u,v \in \Real^n$, $u \prec v$ denotes that
$u_i\leq v_i$ for all $i \in [n]$ and $u_j < v_j$ for some $j \in [n]$.\footnote{We denote by $[n]$ the set $\lbrace 1,2,\ldots,n \rbrace$.} A point is Pareto optimal if it is not Pareto dominated by any other points, i.e., $x \in \cX_f$ if $f(x) \nprec f(x')$ for all $x' \neq x$. The \emph{Pareto front} of $f$ is denoted by $f(\cX_f)$, where for any set $\cA$, $f(\cA)\bydef \lbrace f(x)|x \in \cA \rbrace$.

\paragraph{Random scalarizations and Pareto optimality}
A common approach to solve the multi-objective optimization problem is by converting the objective vectors $f(x)$ into single objective scalars using a \emph{scalarization function} $s_\lambda : \Real^n \ra \Real$, parameterized by a weight vector $\lambda \in \Lambda \subset \Real^n$. Similar to \citet{Paria2019AFF}, we assume random scalarizations, i.e., access to a (known) distribution $P_\lambda$ with its support on $\Lambda$. Thus, instead of maximizing a single scalarized objective, we aim to maximize over a set of scalarizations weighted by the distribution $P_\lambda$. (Note that the Dirac-delta distribution $P_\lambda$\footnote{A Dirac-delta is a probability distribution that puts mass $1$ on exactly one point in the probability space.} yields a deterministic  scalarization.)
We also assume that, for all $\lambda$, the scalarization function
$s_\lambda$ is $L_\lambda$-Lipschitz in the $\ell_2$-norm, i.e., 
\beqn
\forall u,v \in \Real^n,\quad \left\lvert s_\lambda(u)-s_\lambda(v)\right\rvert \leq L_\lambda\norm{u-v}_2~.
\eeqn
Commonly used scalarization functions include the linear scalarization $s_\lambda(y)=\sum_{i=1}^{n}\lambda_iy_i$ and the Chebyshev scalarization $s_\lambda(y)=\min_{i \in [n]}\lambda_i(y_i-z_i)$, where $z \in \Real^n$ is a reference point and $\lambda$ lies in the set $\Lambda = \lbrace \lambda \succ 0:\norm{\lambda}_1=1 \rbrace$ \citep{nakayama2009sequential}. Apart from being Lipschitz, another important property these scalarizations have is monotonicity in all co-ordinates, i.e., $s_\lambda(u) < s_\lambda(v)$ whenever $u \prec v$.
Monotonicity ensures that $x^\star_\lambda := \argmax_{x \in \cX}s_\lambda \left(f(x)\right)$, the maximizer of the scalarized objective, is a Pareto optimal point, since otherwise if $f(x_\lambda^\star) \prec f(x)$ for some $x \neq x_\lambda^\star$, we have $s_\lambda\left(f(x^\star_\lambda)\right) < s_\lambda\left(f(x)\right)$ yielding a contradiction. Therefore, $P_\lambda$ defines a probability distribution over the Pareto optimal set $\cX_f$, and thus, in turn, over the Pareto front $f(\cX_f)$. Hence, the distribution $P_\lambda$
provides flexibility to sample from the entire or a part of the Pareto front depending on the application \citep{Paria2019AFF}.

\paragraph{Performance metric} Given a budget of $T$ rounds, our goal is to find a set $\cX_T :=\lbrace x_1,\ldots,x_T\rbrace \subset \cX$ such that $f(\cX_T)$ well approximates the high probability regions of the Pareto front $f(\cX_f)$. This can be achieved, as shown in \citet{Paria2019AFF}, by minimizing the \emph{Bayes regret}, defined as 
\beqn
R_{B}(T)=\mathbb{E}_{ P_\lambda}\left[r_\lambda(T)\right], \;\; \text{where} \;\; r_\lambda(T) = s_\lambda \left(f(x_\lambda^\star)\right)-\max_{x \in \cX_T}s_\lambda \left(f(x)\right)~.
\eeqn
To see this, we note that it requires $r_\lambda(T)$ to be low for all $\lambda \in \Lambda$ that has high mass, to achieve a low Bayes regret. Now, by definition, $r_\lambda(T)=0$ if $x^\star_\lambda \in \cX_T$, and also, by monotonicity, $x_{\lambda}^\star \in \cX_f$. Then, by the Lipschitz continuity, a low value of $R_B(T)$ will essentially imply $f(\cX_T)$ to "span" the high probability regions of $f(\cX_f)$. A more classical performance measure is the (scalarized) \emph{cumulative regret} 
\beqn
R_C(T) = \sum_{t=1}^{T}\mathbb{E}_{P_\lambda}\big[s_{\lambda_t}\left(f(x^\star_{\lambda_t})\right)-s_{\lambda_t}\left(f\left(x_t\right)\right)\big]~,
\eeqn 
where each $\lambda_t$ is independent and $P_\lambda$ distributed. If $\Lambda$ is a bounded set and the scalarization $s_\lambda$ is also Lipschitz in $\lambda$, then one can show that $R_B(T) \leq \frac{1}{T}R_C(T) +o(1)$ \citep{Paria2019AFF}. 
A sub-linear growth of $R_C(T)$ with $T$ then implies that $R_B(T)\! \ra \!0$ as $T\!\ra\! \infty$.



\paragraph{Regularity assumptions} Attaining non-trivial regret bound is impossible in general for arbitrary vector-valued functions $f$, thus some regularity assumptions are in order. We call a mapping $\Gamma:\cX \times \cX \ra \Real^{n \times n}$, a \emph{multi-task kernel} \footnote{In its more general form, this definition can be lifted from $\Real^n$ to any arbitrary Hilbert space $\cH$ \citep{caponnetto2008universal}.} on $\cX$ if $\Gamma(x,x')^{\top}=\Gamma(x',x)$ for any $x,x' \in \cX$, and it is positive definite, i.e., for any $m \in \Nat$, $\lbrace x_i \rbrace_{i=1}^{m} \subseteq \cX$ and $\lbrace y_i \rbrace_{i=1}^{m} \subseteq \Real^n$ it holds
\beqn
\sum_{i,j=1}^{m}y_i^{\top}\Gamma(x_i,x_j)y_j \geq 0~.
\eeqn
Given a continuous (relative to the induced matrix norm) multi-task kernel $\Gamma$ on $\cX$, there exists a unique (modulo an isometry) vector-valued reproducing kernel Hilbert space (RKHS) of vector-valued continuous functions $g:\cX \to \mathbb{R}^n$, with $\Gamma$ as its reproducing kernel \citep{carmeli2010vector}. We denote this RKHS as $\cH_\Gamma(\cX)$, with the corresponding inner product $\inner{\cdot}{\cdot}_\Gamma$. Then, for every $x \in \cX$, there exists a bounded linear operator $\Gamma_x:\Real^n \ra \cH_\Gamma(\cX)$ such that the following holds:
\beqn
\forall x' \in \cX, \quad \Gamma(x,x')=\Gamma_x^{\top}\Gamma_{x'} \;\;\text{and}\;\; \forall g \in \cH_\Gamma(\cX),\quad g(x)=\Gamma_x^{\top}g~. 
\eeqn
Here, $\Gamma_x^{\top}$ denotes the adjoint of $\Gamma_x$ (with a slight abuse of notation), and it is the unique operator that satisfies the following:
\beqn
\forall g \in \cH_{\Gamma}(\cX),\; \forall y \in \Real^n,\quad \inner{\Gamma_x^{\top}g}{y}_{2} =\inner{g}{\Gamma_x y}_{\Gamma}~.
\eeqn
We assume that the objective function $f$ is an element of the RKHS $\cH_\Gamma(\cX)$ and its norm associated to $\cH_\Gamma(\cX)$ is bounded, i.e., there exists a $b < \infty$ such that $\norm{f}_\Gamma \leq b$.
This is a measure of smoothness of $f$, since, by the reproducing property
\beqn
\forall x,x' \in \cX,\quad \norm{f(x) - f(x')}_2  \leq \norm{f}_\Gamma \norm{\Gamma_x-\Gamma_{x'}},
\eeqn
where $\norm{\Gamma_x}:=\sup_{\norm{y}_2 \leq 1}\norm{\Gamma_x y}_{\Gamma}$ denotes the operator norm.
Further, we assume that there exists a $\kappa < \infty$ such that
$\norm{\Gamma(x,x)} \leq \kappa$ for all $x \in \cX$. Note that in the single-task setting ($n=1$), the kernel $\Gamma$ is scalar-valued and the RKHS $\cH_\Gamma(\cX)$ consists of real-valued functions. In this case, the bounded norm assumption holds for stationary kernels, e.g., the \emph{squared exponential} (SE) kernel and the \emph{Mat\'{e}rn} kernel \citep{srinivas2010gaussian,chowdhury2017kernelized}.


\paragraph{Examples of multi-task (MT) kernels}
It is possible to construct MT kernels using scalar kernels $k:\cX \times \cX \to \Real^{+}$. \citet{evgeniou2005learning} consider the kernel \beqn
\Gamma(x,x')=k(x,x')\left(\omega I_n+(1-\omega)1_n/n\right)~,
\eeqn
where $I_n$ is the $n \times n$ identity matrix, $1_n$ is the $n\times n$ all-one matrix and $\omega \in [0,1]$ is a parameter that governs the similarity level between components of $f$. The choice $\omega = 1$ corresponds to assuming that all tasks are unrelated and possible similarity among them is not exploited. Conversely, $\omega = 0$ is equivalent to assuming that all tasks are identical and can be explained by the same function. \citet{swersky2013multi} consider a more general class of kernels known as the \emph{intrinsic coregionalization model} (ICM), which includes the aforementioned kernel as a special case. The kernels are of the form 
\beqn
\Gamma(x,x')=k(x,x')B~,
\eeqn
where $B$ is an $n\times n$ p.s.d. matrix that encodes the inter-task structure. This class of kernels is called \emph{separable} since it
allows to decouple the contribution of input and output in the covariance structure \citep{alvarez2011kernels}. We consider
stationary scalar kernels $k$ with unit variances -- to avoid redundancy in the parameterization -- since the variances
can be captured fully by $B$ \citep{bonilla2008multi}. The main advantage of ICM is that one can use the eigen-system of $B$ to define a new coordinate system where $\Gamma$ becomes block diagonal, reducing the computational burden to a great extent. The diagonal MT kernel $\Gamma(x,x')=\Dg\left(k_1(x,x'),\ldots,k_n(x,x')\right)$ has the same advantage, but corresponds to treating each task independently using different scalar kernels $k_j$. However, in general, a MT kernel will not be diagonal, and moreover cannot be reduced to a diagonal one by linearly transforming the output space. For example, it is impossible to reduce the kernel $\Gamma(x,x')=\sum_{j=1}^{M}k_j(x,x')B_j$, $M \neq 1$, to a diagonal one, unless all the $n \times n$ matrices $B_j$ are simultaneously diagonalizable \citep{caponnetto2008universal}.

\section{Our approach}

We follow the general template of upper confidence bound (UCB) class of BO algorithms \citep{srinivas2010gaussian,chowdhury2017kernelized} suitably adapted to the multi-task setting. At each round $t$, we randomly sample a weight vector $\lambda_t$ from the distribution $P_\lambda$, and compute a multi-task acquisition
function $u_t:\cX \to \Real$ to act as an UCB for the unknown
function $f$, based on the random scalarization $s_{\lambda_t}$. Whenever $u_t(x)$ is a valid UCB, i.e., $s_{\lambda_t}\left(f(x)\right) \leq u_t(x)$, and it converges to $s_{\lambda_t}\left(f(x)\right)$
``sufficiently" fast, then selecting candidates that are optimal with respect to $u_t$ leads to low (scalarized) regret,
i.e., the scalarized objective $s_{\lambda_t}\left(f(x_{t})\right)$ at $x_{t} \in \argmax_{x \in \cX} u_t(x)$ tends to $s_{\lambda_t}\left(f(x_{\lambda_t}^\star)\right)$ as $t$ increases. The intuition behind our approach, at a high level, is that the set $f(\cX_t)$ tends to the high probability regions of the Pareto front as $t$ increases.
It now remains to design a principled multi-task acquisition function $u_t$ based on the scalarization $s_{\lambda_t}$, and in what follows, we shall describe two algorithms for that.


\subsection{Algorithm 1: Multi-task kernelized bandits (MT-KB)}
\label{subsec:MT-KB}

Given the data $\lbrace(x_i,y_i)\rbrace_{i=1}^{t} \subset \cX \times \Real^n$, we first aim to find an estimate of $f$ by solving a vector-valued regression problem: 
\beqn
\min_{f \in \cH_\Gamma(\cX)}\;\sum_{i=1}^{t}\norm{y_i-f(x_i)}_2^2+\eta \norm{f}_\Gamma^2~,
\eeqn
where $\eta > 0$ is a regularizing parameter. \citet{micchelli2005learning} show that the solution of this minimization problem can be written as
\beqn
\mu_t=\sum_{i=1}^{t}\Gamma_{x_i}\alpha_i~.
\eeqn
Here, $\lbrace\alpha_i \rbrace_{i=1}^{t} \subseteq \Real^n$ is the unique solution of the linear system of equations
\beqn
\sum_{i=1}^{t} \left(\Gamma(x_j,x_i)+\eta \delta_{j,i}\right)\alpha_i=y_j~,\quad 1 \leq j \leq t~,
\eeqn
where $\delta_{j,i}$ denotes the Kronecker-delta function. Now, by the reproducing property, we have
\beqn
    \mu_t(x)=\Gamma_x^{\top}\mu_t=G_{t}(x)^{\top}(G_t+\eta I_{nt})^{-1}Y_{t}~,
    \label{eqn:post-mean}
\eeqn
where the kernel matrix $G_t=[\Gamma(x_i,x_j)]_{i,j=1}^{t}$ is a $t \times t$ block matrix with each block being an $n\times n$ matrix (so that $G_t$ is an $nt \times nt$ matrix), $Y_{t}=\left[y_1^{\top},\ldots,y_t^{\top}\right]^{\top}$ is an $nt \times 1$ vector with the outputs concatenated, and $G_t(x)=\left[\Gamma(x,x_1)^{\top},\ldots,\Gamma(x,x_t)^{\top}\right]^{\top}$ is an $nt \times n$ matrix. Notice that $G_t(x)$ can be interpreted as an embedding of a point $x$ supported over the points $x_1,\ldots,x_t$ observed so far.
Now, if an arm $x$ is sufficiently unexplored, the estimate $\mu_t(x)$ will, in general, have high variance. One natural way of specifying the uncertainty around $\mu_t(x)$ is the following multi-task kernel:
\beq
	\Gamma_t(x,x')=\Gamma(x,x')-G_t(x)^{\top}(G_t+\eta I_{nt})^{-1}G_t(x')~, \quad x,x' \in \cX~.
\label{eqn:post-cov}	
\eeq
To see this, we draw a connection to multi-task Gaussian processes (MT-GPs) \citep{liu2018remarks}. Let $f\sim \cG \cP(0,\Gamma)$ be a sample from a zero-mean MT-GP  with covariance function $\Gamma$ (i.e., $\expect{f_i(x)}=0$ and $\expect{f_i(x)f_j(x')}=\Gamma(x,x')_{ij}$ for all $i,j \in [n]$ and $x,x'\in \cX$), and assume that the observation noise vectors $\lbrace\epsilon_t\rbrace_{t\geq 1}$ are independent and $\cN(0,\eta I_n)$ distributed. Then the posterior distribution of $f$ conditioned on the data $\lbrace(x_i,y_i)\rbrace_{i=1}^{t}$ is also a MT-GP with mean $\mu_t$ and covariance $\Gamma_t$, yielding a natural uncertainty model.
Now, inspired by the optimism-in-face-of-uncertainty principle, we compute the acquisition function for the next round as
\beq
u_{t+1}(x)=s_{\lambda_{t+1}}\left(\mu_{t}(x)\right)+L_{\lambda_{t+1}}\beta_{t}\norm{\Gamma_{t}(x,x)}^{1/2}~,
\label{eqn:UCB-rule}
\eeq
where $L_{\lambda}$ is the Lipschitz constant of the scalarization $s_{\lambda}$. As a result, selecting the arm $x_{t+1}$ with the highest $u_{t+1}$ inherently trades off exploitation, i.e., picking points with high (scalarized) reward $s_{\lambda_{t+1}}\left(\mu_{t}(x)\right)$, with exploration, i.e., picking points with high uncertainty $\norm{\Gamma_{t}(x,x)}^{1/2}$. The parameter $\beta_t$ balances between these two objectives, and needs be tuned properly to guarantee low regret.
The pseudo-code of MT-KB is given in Algorithm \ref{algo:MT-KB}.

\begin{algorithm}[t]
\renewcommand\thealgorithm{1}
\caption{Multi-task kernelized bandits (MT-KB)}\label{algo:MT-KB}
\begin{algorithmic}
\STATE \textbf{Require:} Kernel $\Gamma$, distribution $P_\lambda$, scalarization $s_\lambda$, time budget $T$, parameters $\eta$, $\lbrace\beta_t\rbrace_{t=0}^{T-1}$
\STATE Initialize $\mu_0(x)=0$ and $\Gamma_0(x,x')=\Gamma(x,x')$
\FOR{round $t = 1, 2, 3, \ldots,T$}
\STATE Sample weight vector $\lambda_t \sim P_{\lambda}$
\STATE Compute acquisition function $u_{t}(x)=s_{\lambda_{t}}\left(\mu_{t-1}(x)\right)+L_{\lambda_{t}}\beta_{t-1}\norm{\Gamma_{t-1}(x,x)}^{1/2}$
\STATE Select point $x_{t} \in \argmax_{x \in \cX} u_t(x)$
\STATE Get vector-valued output $y_t=f(x_t)+\epsilon_t$
\STATE Compute
\beqn
G_t(x)=\left[\Gamma(x_1,x)^{\top},\ldots,\Gamma(x_t,x)^{\top}\right]^{\top},\;G_t=[\Gamma(x_i,x_j)]_{i,j=1}^{t},\; Y_{t}=\left[y_1^{\top},\ldots,y_t^{\top}\right]^{\top}
\eeqn
\STATE Update
\beqan
\mu_t(x)&=&G_{t}(x)^{\top}(G_t+\eta I_{nt})^{-1}Y_{t}\\
\Gamma_t(x,x)&=&\Gamma(x,x)-G_t(x)^{\top}(G_t+\eta I_{nt})^{-1}G_t(x)
\eeqan
\ENDFOR
\end{algorithmic}
\addtocounter{algorithm}{-1}
\end{algorithm}


\paragraph{Computational complexity} Maximizing the acquisition function $u_t(x)$ over $\cX$ is in general NP-hard even for a single task, since it is a highly non-convex function. To simplify the exposition, in what follows, we will assume that an efficient oracle to optimize $u_t(x)$, such as DIRECT \citep{brochu2010tutorial}, is provided to us, and the per step cost comes only from computing $u_t(x)$. Now, the cost of computing $u_t(x)$ is dominated by the cost of inversion of the $nt \times nt$ kernel matrix, and thus in principle scales as $O(n^3t^3)$.\footnote{This can be reduced to $O(n^3t^2)$ using Schur's complement, but at an additional storage cost of $O(n^2t^2)$.} We note that the cubic dependency with time $t$ is present even in the single-task ($n=1$) setting \citep{shahriari2015taking} and in this case, in fact, MT-KB reduces to the GP-UCB algorithm \citep{srinivas2010gaussian}. 
\begin{remark}
The diagonal MT kernel $\Gamma(x,x')=\Dg\left(k_1(x,x'),\ldots,k_n(x,x')\right)$ corresponds to treating each task independently and the problem reduces to inverting $n$ kernel matrices yielding a per-step cost of $O(nt^3)$ for MT-KB. This is similar to the prior works \citep{hernandez2016predictive,Paria2019AFF,belakaria2019max} which assume that each task $f_i$ is sampled independently from the scalar Gaussian process $\cG\cP(0,k_i)$. 
\end{remark}

One common approach to improve computational scalability in kernel methods is the Nystr\"{o}m approximation \citep{drineas2005nystrom}, which restricts the embeddings $G_t(x)$ and the kernel matrix $G_t$ to be supported on a subset (dictionary) $\cD_t$ of selected points. However, this can lead to
sub-optimal choices and large regret
if $\cD_t$ is not sufficiently accurate. This brings about a trade-off between
larger and more accurate dictionaries, or smaller and more efficient ones. The BKB algorithm solves this for single-task BO \citep{calandriello2019gaussian}. We now generalize BKB for  multiple tasks to improve over the $O(n^3t^3)$ cost of MT-KB.

\subsection{Algorithm 2: Multi-task budgeted kernelized bandits (MT-BKB)} 
\label{subsec:MT-BKB}

The central idea behind this algorithm is to evaluate an approximate acquisition function $\tilde{u}_t(x)$, which remains a valid UCB over the scalarized function $s_{\lambda_t}\left(f(x)\right)$ and at the same time is sufficiently close to $u_t(x)$ to ensure low regret. Given the data $\lbrace(x_i,y_i)\rbrace_{i=1}^{t}$, we start with an empty dictionary $\cD_t=\emptyset$ and iterate over the set $\lbrace x_1,\ldots,x_t\rbrace$ to update $\cD_t$ as follows. For each candidate $x_i$, we compute an inclusion probability $p_{t,i}$, and add $x_i$ to $\cD_t$ with probabability $p_{t,i}$.
The inclusion probabilities $p_{t,i}$ need to be set suitably so that the dictionary is small enough without compromising on its accuracy. Once the sampling is over, let $\cD_t$ be given by the set $\lbrace x_{i_1},\ldots,x_{i_{m_t}}\rbrace$, where $m_t$ is the size of $\cD_t$ and $ i_j \in [t]$ for each $j \in [m_t]$. Given the dictionary $\cD_t$, let $\tilde{G}_{t}(x)=\big[\Gamma(x_{i_1},x)^{\top}/\sqrt{p_{t,i_1}},\ldots,\Gamma(x_{i_{m_t}},x)^{\top}/\sqrt{p_{t,i_{m_t}}}\big]^{\top}$ be the $nm_t \times n$ embedding of $x$ supported over all points in $\cD_t$ and $\tilde{G}_{t}=\left[\Gamma(x_{i_u},x_{i_v})/\sqrt{p_{t,i_u}p_{t,i_v}}\right]_{u,v=1}^{m_t}$ be the corresponding $nm_t \times nm_t$ kernel matrix, properly reweighted by the inclusion probabilities.
Then we compute the Nystr\"{o}m embeddings as 
\beqn
\tilde{\Phi}_t(x)=\big(\tilde{G}_t^{1/2}\big)^{+}\tilde{G}_{t}(x)~,
\eeqn
where $(\cdot)^{+}$ denotes the pseudo-inverse. We now use these embeddings to approximate $\mu_t$ and $\Gamma_t$ as
\beqan
\tilde{\mu}_t(x)&=&\tilde{\Phi}_t(x)^{\top}(\tilde{V}_t+\eta I_{nm_t})^{-1}\sum_{s=1}^{t}\tilde{\Phi}_t(x_s)y_s\;,\\
\tilde{\Gamma}_t(x,x')&=&\Gamma(x,x')-\tilde{\Phi}_t(x)^{\top}\tilde{\Phi}_t(x')+\eta \tilde{\Phi}_t(x)^{\top} (\tilde{V}_t+\eta I_{nm_t})^{-1}\tilde{\Phi}_t(x')\;,
\eeqan
where $\tilde{V}_t=\sum_{s=1}^{t}\tilde{\Phi}_t(x_s)\tilde{\Phi}_t(x_s)^{\top}$ is an $nm_t \times nm_t$ matrix. Finally, similar to (\ref{eqn:UCB-rule}), we compute the acquisition function for the next round as
\beqn
\tilde{u}_{t+1}(x)=s_{\lambda_{t+1}}\left(\tilde{\mu}_{t}(x)\right) + L_{\lambda_{t+1}}\tilde{\beta}_{t}\lVert\tilde{\Gamma}_{t}(x,x)\rVert^{1/2}~,
\eeqn
with $\tilde{\beta}_{t}$ governing the exploration-exploitation tradeoff. The inclusion probabilities for the next round are computed as $p_{t+1,i}=\min \big\lbrace q \lVert\tilde{\Gamma}_{t}(x_i,x_i)\rVert, 1 \big\rbrace$, where $q \geq 1$ is a parameter trading-off
the size of the dictionary and accuracy of the approximation. We note here that constructing $\cD_t$
based on approximate posterior variance sampling is well-studied for scalar kernels \citep{alaoui2015fast}, and in this work, we introduce it for the first time for MT kernels.  The pseudo-code of MT-BKB is given in Algorithm \ref{algo:MT-BKB}.

\begin{algorithm}[t]
\renewcommand\thealgorithm{2}
\caption{Multi-task budgeted kernelized bandits (MT-BKB)}\label{algo:MT-BKB}
\begin{algorithmic}
\STATE \textbf{Require:} Kernel $\Gamma$, distribution $P_\lambda$, scalarization $s_\lambda$, time budget $T$, parameters $\eta$, $q$, $\lbrace\tilde{\beta}_t\rbrace_{t=0}^{T-1}$
\STATE Initialize $\tilde{\mu}_0(x)=0$ and $\tilde{\Gamma}_0(x,x')=\Gamma(x,x')$
\FOR{round $t = 1, 2, 3, \ldots,T$}
\STATE Sample weight vector $\lambda_t \sim P_{\lambda}$
\STATE Compute acquisition function $\tilde{u}_{t}(x)=s_{\lambda_{t}}\left(\tilde{\mu}_{t-1}(x)\right)+L_{\lambda_{t}}\tilde{\beta}_{t-1}\lVert\tilde{\Gamma}_{t-1}(x,x)\rVert^{1/2}$
\STATE Select point $x_{t} \in \argmax_{x \in \cX} \tilde{u}_t(x)$
\STATE Get vector-valued output $y_t=f(x_t)+\epsilon_t$
\STATE Initialize dictionary $\cD_t=\emptyset$ 
\FOR{$i = 1, 2, 3, \ldots, t$}
\STATE Set inclusion probability $p_{t,i}=\min \left\lbrace q \lVert\tilde{\Gamma}_{t-1}(x_i,x_i)\rVert, 1 \right\rbrace$
\STATE Draw $z_{t,i}\sim\text{Bernoulli}(p_{t,i})$
\IF{$z_{t,i}=1$}
\STATE Update $\cD_t=\cD_t \cup \lbrace x_i \rbrace$ 
\ENDIF
\ENDFOR
\STATE Set $m_t=\abs{\cD_t}$, enumerate $\cD_t=\lbrace x_{i_1},\ldots,x_{i_{m_t}}\rbrace$ and compute 
\beqn
\tilde{G}_{t}(x)=\left[\frac{1}{\sqrt{p_{t,i_1}}}\Gamma(x_{i_1},x)^{\top},\ldots,\frac{1}{\sqrt{p_{t,i_{m_t}}}}\Gamma(x_{i_{m_t}},x)^{\top}\right]^{\top},\;  \tilde{G}_{t}=\left[\frac{1}{\sqrt{p_{t,i_u}p_{t,i_v}}}\Gamma(x_{i_u},x_{i_v})\right]_{u,v=1}^{m_t}
\eeqn
\STATE Find Nystr\"{o}m embeddings $\tilde{\Phi}_t(x)=\left(\tilde{G}_t^{1/2}\right)^{+}\tilde{G}_{t}(x)$
\STATE Compute $\tilde{V}_t =\sum_{s=1}^{t}\tilde{\Phi}_t(x_s)\tilde{\Phi}_t(x_s)^{\top}$ and update
\beqan
\tilde{\mu}_t(x)&=&\tilde{\Phi}_t(x)^{\top}(\tilde{V}_t+\eta I_{nm_t})^{-1}\sum\nolimits_{s=1}^{t}\tilde{\Phi}_t(x_s)y_s\\
\tilde{\Gamma}_t(x,x)&=&\Gamma(x,x)-\tilde{\Phi}_t(x)^{\top}\tilde{\Phi}_t(x)+\eta \tilde{\Phi}_t(x)^{\top} (\tilde{V}_t+\eta I_{nm_t})^{-1}\tilde{\Phi}_t(x)
\eeqan
\ENDFOR
\end{algorithmic}
\addtocounter{algorithm}{-1}
\end{algorithm}



\paragraph{Computational complexity}
Computing the dictionary involves a linear search over all selected points while the inclusion probabilities are computed already at the previous round, and thus requires $O(t)$ time per step. The Nystr\"{o}m embeddings $\tilde{\Phi}_t(x)$ can be computed in $O(n^3m_t^3)$ time, since an inversion of the matrix $\tilde{G}_t$ is required. By using these embeddings, $\tilde{V}_t$ can now be computed and inverted in $O(n^2m_t^2t)$ and $O(n^3m_t^3)$ time, respectively. Since, in general, $m_t \leq t$, the total per step cost of computing the acquisition function $\tilde{u}_t(x)$ is now $O(n^3m_t^2t)$ as opposed to the $O(n^3t^3)$ cost of MT-KB. The computational advantage of MT-BKB is clearly visible when the dictionary size $m_t$ is near constant at every step, i.e., when $m_t=\tilde{O}(1)$, where $\tilde{O}(\cdot)$ hides constant and $\log$ factors. We shall see in Section \ref{subsec:regret-bound} that this holds, for example, for the intrinsic coregionalization model (ICM) with the squared exponential kernel in its scalar part.

\subsection{Improved computational complexity for ICM kernels} 

The computational cost of our algorithms can be greatly reduced for ICM kernels $\Gamma(x,x') = k(x, x')B$. Let $\lbrace \xi_i\rbrace_{i=1}^{n}$ be the eigenvalues of $B$ with corresponding orthonormal eigenvectors $\lbrace u_i \rbrace_{i=1}^{n}$. We then have the kernel matrix $G_t=\sum_{i=1}^{n}\xi_i K_t \otimes u_iu_i^{\top}$ and the output vector $Y_{t}=\sum_{i=1}^{n}Y_{t}^{i}\otimes u_i$, where $\otimes$ denotes the Kronecker product, $K_t=[k(x_i,x_j)]_{i,j=1}^{t}$ is the kernel matrix of the scalar kernel $k$ and $Y_{t}^{i}=[y_1^{\top}u_i,\ldots,y_t^{\top}u_i]^{\top}$. Plugging these into (\ref{eqn:post-mean}) and (\ref{eqn:post-cov}), and using properties of Kronecker product, we now obtain 
\beqan
\mu_t(x)&=&\sum_{i=1}^{n}\xi_i k_t(x)^{\top}(\xi_iK_t+\eta I_t)^{-1}Y_{t}^{i} u_i~,\\
\lVert\Gamma_t(x,x)\rVert &=& \max_{1 \leq i \leq n}\xi_i\left(k(x,x)- \xi_ik_t(x)^{\top}(\xi_iK_t+\eta I_t)^{-1}k_t(x)\right)~,
\eeqan
where $k_t(x)=[k(x_1,x),\ldots,k(x_t,x)]^{\top}$.
We see that the eigen-decomposition of $B$ needs to be computed only once at the beginning and then, in the new coordinate system, we essentially have to solve $n$ independent problems. Specifically, at round $t$, we need to project the vector-valued output $y_t$ to all coordinates and compute $n$ matrix-vector multiplications of size $t$. However, since the kernel matrix $K_t$ is rescaled by the eigenvalues $\xi_i$, we have to perform only one $t \times t$ inversion.
Hence, the per-step time complexity of MT-KB is now $O\left(n^2+(n+t)t^2\right)$ as opposed to $O(n^3t^3)$ for general MT kernels. Similarly, the per-step cost of MT-BKB can be substantially improved to $O\left(n^2+(n+m_t)m_tt\right)$ from the $O(n^3m_t^2t)$ cost in general. 
Therefore, the kernels of this form allow for a near-linear (in time $t$) per-step cost of MT-BKB at the price of the eigen-decomposition of $B$. (We defer the details to appendix \ref{app:complexity}.)

\section{Theoretical results}

We now present the first theoretical result of this work, a concentration inequality for the estimate of the unknown multi-task objective function $f$, which is then used to prove the regret bounds for our algorithms. (Complete proofs of all results presented in this section are deferred to the appendix.)

\begin{mytheorem}[Multi-task concentration inequality] 
Let $f \in \cH_\Gamma(\cX)$ and the noise vectors $\lbrace\epsilon_t\rbrace_{t \geq 1}$ be $\sigma$-sub-Gaussian. Then, for any $\eta > 0$ and $\delta \in (0,1]$, with probability at least $1-\delta$, the following holds uniformly over all $x \in \cX$ and $t \geq 1:$
	\beqn
		\norm{f(x)-\mu_{t}(x)}_2 \leq \left(\norm{f}_\Gamma+\frac{\sigma}{\sqrt{n}}\sqrt{2\log(1/\delta)+\log\det(I_{nt}+\eta^{-1}G_t)}\right)\norm{\Gamma_{t}(x,x)}^{1/2}.
	\eeqn
	
	\label{thm:concentration}
\end{mytheorem}

The significance of this bound can be better understood by studying the log-determinant term, and for this, we again draw a connection to MT-GPs. If $f \sim \cG\cP(0,\Gamma)$ and $\epsilon_t \sim \cN(0, \eta I_n)$ i.i.d., then the \emph{mutual information} between $f$ and the outputs $Y_t$ is exactly equal to $\frac{1}{2}\log\det(I_{nt}+\eta^{-1}G_t))$, and it is a measure for the reduction in the
uncertainty or, equivalently, the information gain about $f$. Note that while we use GPs to describe the uncertainty in estimating the unknown function $f$, the bound is \emph{frequentist} and does not need any \emph{Bayesian} assumption about $f$. Similar to the single-task setting \citep{durand2018streaming}, the bound is proved by deriving a new self-normalized concentration inequality for martingales in the $\ell_2$ space.\footnote{Theorem \ref{thm:concentration} can even be generalized to the regime of infinite-task learning \citep{kadri2016operator,brault2019infinite}, where the observations lie in a Hilbert space $\cH$, and thus can be of independent interest. The only technical assumption that one will need is that the multi-task kernel $\Gamma(x,x)$ has a finite trace, which trivially holds in the finite-task setting.} We note here that 
\citet{astudillo2019bayesian} consider the much simpler setting of noise-free outputs and their bound can be re-derived as a special case of Theorem \ref{thm:concentration}.


\begin{remark} The multi-task kernel $\Gamma$ can be seen as a scalar kernel, $\Gamma(x,x')_{ij}=k\left((x,i),(x',j)\right)$, $i,j \in [n]$, and $G_t$ as an $nt \times nt$ kernel matrix of $k$ evaluated at points $(x_s,i)$, $s \in [t]$, $i \in [n]$. In this case, one can use \citet[Theorem 2]{chowdhury2017kernelized} to derive concentration bounds for each task $f_i$ separately and combine them together to obtain a result similar to Theorem \ref{thm:concentration} but with a notable change -- $\norm{\Gamma_t(x,x)}$ being replaced by $\Tr\left(\Gamma_t(x,x)\right)$. Thus, in general, we prove a tighter concentration inequality which eventually leads to a $O(\sqrt{n})$ factor saving in the final regret bound.   
\end{remark}



\subsection{Regret bounds}
\label{subsec:regret-bound}

Theorem \ref{thm:concentration} allows for a principled way to tune the confidence radii (i.e., $\beta_t$ and $\tilde{\beta}_t$) of our algorithms and achieve low regret. We now present the regret bound of MT-KB, which, to the best of our knowledge, is the first frequentist regret guarantee for multi-task BO under any general MT kernel.




\begin{mytheorem}[Cumulative regret of MT-KB]
	Let $f \in \cH_\Gamma(\cX)$, $\norm{f}_\Gamma \leq b$ and $\norm{\Gamma(x,x)} \leq \kappa$ for all $x \in \cX$. Let the scalarization function $s_\lambda$ be $L_\lambda$-Lipschitz, $L_\lambda \leq L$ for all $\lambda \in \Lambda$ and the noise vectors $\lbrace \epsilon_t \rbrace_{t \geq 1}$ be $\sigma$-sub-Gaussian. Then, for any $\eta >0$ and $\delta \in (0,1]$, MT-KB with
	\beqn
	\beta_t=b+\frac{\sigma}{\sqrt{\eta}}\sqrt{2\log(1/\delta)+\sum_{s=1}^{t}\log \det\left(I_n+\eta^{-1}\Gamma_{s-1}(x_s,x_s)\right)}~,
	\eeqn
	enjoys, with probability at least $1-\delta$, the regret bound
	\beqn
		R_C^{\text{MT-KB}}(T) \leq 2L\left(b+\frac{\sigma}{\sqrt{\eta}}\sqrt{\left(2\log(1/\delta)+\gamma_{nT}(\Gamma,\eta)\right)}\right)\sqrt{(1+\kappa/\eta)T\sum\nolimits_{t=1}^{T}\norm{\Gamma_{t}(x_t,x_t)}}\;,
	\eeqn
	where $\gamma_{nT}(\Gamma,\eta):=\max_{\cX_T \subset \cX}\frac{1}{2}\log\det\left(I_{nT}+\eta^{-1}G_T\right)$ denotes the maximum information gain.
	\label{thm:cumulative-regret}
\end{mytheorem}

Theorem \ref{thm:cumulative-regret}, along with the upper bound $\sum\nolimits_{t=1}^{T}\norm{\Gamma_{t}(x_t,x_t)} \leq 2 \eta\gamma_{nT}(\Gamma,\eta)$, yields the more compact regret bound $\tilde{O}\big(b\sqrt{ T\gamma_{nT}(\Gamma,\eta)}+\gamma_{nT}(\Gamma,\eta)\sqrt{T}\big)$. We note here that the bound for single-task case \citep{chowdhury2017kernelized} can be recovered by setting $n=1$. Furthermore, since the single-task bound is shown to be tight upto a poly-logarithmic factor \citep{scarlett2017lower}, our bound, we believe, is also tight in terms of dependence on $T$.
Now, we instantiate Theorem \ref{thm:cumulative-regret} for the special case of separable kernels to point out the novel insights and improvements that our analysis unearths as compared to existing work.  

\begin{mylemma}[Inter-task structure in regret bound]
Let $B$ be an $n \times n$ p.s.d. matrix and $\Gamma(x,x')=k(x,x')B$. Let $\norm{\Gamma(x,x)} \leq \kappa$ and $k(x,x)=1$ for all $x \in \cX$. Then the following holds:
\beqan
\gamma_{nT}(\Gamma,\eta) &\leq & \sum_{i \in [n]:\xi_i > 0}\gamma_T(k,\eta/\xi_i)\;,\\  \sum_{t=1}^{T}\norm{\Gamma_{t}(x_t,x_t)} &\leq & 2\eta \max\lbrace\kappa,1\rbrace\gamma_T(k,\eta)\;,
\eeqan
where $\xi_1,\ldots,\xi_n$ are the eigenvalues of $B$ and $\gamma_T(k,\alpha):=\max_{\cX_T \subset \cX}\frac{1}{2}\log\det\left(I_{T}+\alpha^{-1}K_T\right)$, $\alpha > 0$, is the maximum information gain associated with the scalar kernel $k$.
\label{lem:bound-ICM}
\end{mylemma}

Lemma \ref{lem:bound-ICM}, along with Theorem \ref{thm:cumulative-regret}, leads to a regret bound that explicitly encodes the amount of similarity between tasks in terms of the spectral properties of $B$.
For example, consider the case $B= \omega I_n+(1-\omega) 1_n/n$, $\omega \in [0,1]$, which has one eigenvalue equal to $1$ and all others equal to $\omega$. In this case, we obtain $\gamma_{nT}(\Gamma,\eta)\leq\gamma_{T}(k,\eta)+(n-1)\gamma_{T}(k,\eta/\omega)$. Now $\gamma_T(k,\eta/\omega)$ is an increasing function in $\omega$, and in fact, $\gamma_T(k,\eta/\omega)=0$ when $\omega=0$.
Hence, a low value of $\omega$, i.e., a high amount of similarity between tasks, yields a low cumulative regret and vice-versa. (A numerical example is shown in Figure \ref{fig:cumulative-regret-plot}: (a) using the squared exponential kernel as $k$.) Moreover, for the extreme two cases of $\omega=0$ (all tasks identical) and $\omega=1$ (all tasks unrelated), the regret bounds are $\tilde{O}(\gamma_T(k,\eta)\sqrt{T})$ and $\tilde{O}(\gamma_T(k,\eta)\sqrt{nT})$, respectively. The bounds clearly assert that similar objectives can be learnt much faster together rather than learning them separately. To the best of our knowledge, this intuitive but important observation is not captured by any of the existing regret analysis \citep{zuluaga2013active,Paria2019AFF,belakaria2019max}.

\begin{remark} Existing works model each task independently by means of a diagonal multi-task kernel $\Gamma(x,x')=\Dg\left(k_1(x,x'),\ldots,k_n(x,x')\right)$ and prove regret bounds for this special setting. In contrast, Theorem \ref{thm:cumulative-regret} is applicable to any general multi-task kernel, and in the special case of diagonal kernel, yields, along with Lemma \ref{lem:bound-ICM}, a regret bound of $\tilde{O}(\max_{i}\gamma_T(k_i,\eta)\sqrt{nT})$. This bound, together with the discussion above, suggest that whereas on the one hand MT-KB exploits similarities between tasks efficiently, its performance on the other hand does not suffer when the tasks are unrelated. 
Another important point to note here is that we analyze the frequentist (worst-case) regret, which is a stronger notion of regret compared to the Bayesian one (defined as the expected cumulative regret under a prior distribution of $f$) as considered in previous works \citep{Paria2019AFF,belakaria2019max}. 
\end{remark}

We now present regret and complexity guarantees for MT-BKB, which, to the best of our knowledge, are first of their kinds for multi-task BO under kernel or GP approximation.

\begin{mytheorem}[Analysis of MT-BKB]
For any $\eta>0$, $\epsilon \in (0,1)$ and $\delta \in (0,1]$, let $\rho=(1+\epsilon)/(1-\epsilon)$ and $q = 6\rho\log (4T/\delta)/\epsilon^2$. Then, under the same hypothesis as Theorem \ref{thm:cumulative-regret}, if we run MT-BKB with 
\beqn
\tilde{\beta}_t 
 =b\left(1+1/\sqrt{1-\epsilon}\right)+\frac{\sigma}{\sqrt{\eta}}\sqrt{2\log(2/\delta)+\rho\sum_{s=1}^{t}\log \det\big(I_n+ \eta^{-1}\tilde{\Gamma}_{s-1}(x_s,x_s)\big)}~,
 \eeqn
 then, with probability at least $1-\delta$, the following holds:
\beqan
 R_C^{\text{MT-BKB}}(T) &\leq & 2\rho^{3/2}R_C^{\text{MT-KB}}(T)~, \\ 
 \forall t \in [T], \quad m_t &\leq &   6\rho q(1+\kappa/\eta)\sum\nolimits_{s=1}^{t}\norm{\Gamma_{s}(x_s,x_s)}~. 
 \eeqan

\label{thm:cumulative-regret-kernel-approx}
\end{mytheorem} 

Theorem \ref{thm:cumulative-regret-kernel-approx} shows that MT-BKB can achieve an order-wise similar regret scaling as MT-KB (up to a constant factor), but only at a fraction of the computational cost. To see this, we again consider the kernel $\Gamma(x,x')=k(x,x')B$. In this case, Theorem \ref{thm:cumulative-regret-kernel-approx} and Lemma \ref{lem:bound-ICM} together imply that the dictionary size $m_t$ is $\tilde{O}\left(\gamma_t(k,\eta)\right)$. Now $\gamma_t$
is itself bounded for specific scalar kernels $k$, e.g., it is
$O\left((\ln t)^{d}\right)$ for the squared exponential kernel \citep{srinivas2010gaussian}, yielding $m_t$ to be $\tilde{O}(1)$. This leads to a near-linear (in time $t$) per-step cost for MT-BKB compared to the cubic cost for MT-KB. Further, it is worth noting that MT-BKB can adapt to any desired
accuracy level $\epsilon$ of the Nystr\"{o}m approximation. A low value of $\epsilon$ corresponds to high desired accuracy and MT-BKB adapts to it by inducing more and more points in the dictionary, yielding accurate embeddings and thus, in turn, low regret. Conversely, if one is willing to compromise on the accuracy (given by a high value of $\epsilon$), then MT-BKB can greatly reduce the size of the dictionary, yielding a low time complexity. 
The analysis follows in the footsteps of \citet{calandriello2019gaussian}, but is carefully generalized to consider multi-task kernels. The regret bound is crucially achieved by showing that $\Gamma_t(x,x)/\rho \preceq \tilde{\Gamma}_t(x,x) \preceq \rho\Gamma_t(x,x)$, i.e., MT-BKB's variance estimates are always almost close to the exact ones ($A \succeq B$ denotes that the matrix $A-B$ is p.s.d.). This not only helps us avoid variance starvation which is known to happen with classical sparse GP approximations \citep{wang2018batched}, but also, allows us to set $\tilde{\beta}_t$ efficiently and in a data-adaptive way.


 

\section{Experiments}
\label{sec:experiments} 

In order to investigate the practical benefits offered by learning with multi-task kernels, we compare MT-KB and MT-BKB with single-task algorithms that enjoy regret guarantees under RKHS smoothness assumptions. Specifically, we consider GP-UCB \citep{chowdhury2017kernelized} and its Nystr\"{o}m approximation BKB \citep{calandriello2019gaussian} as baselines, where each task is learnt independently and inter-task structure is not exploited. We call these baselines {\em independent task kernelized bandits (IT-KB)} and {\em budgeted kernelized bandits (IT-BKB)}, respectively. Whenever the objective is not explicitly generated from an RKHS, we also compare with MOBO \citep{Paria2019AFF}, which has better regret performance than other methods \citep{knowles2006parego,ponweiser2008multiobjective,emmerich2008computation,hernandez2016predictive} that model each task with an independent GP. In all simulations, we set $\eta=0.1$, $\delta=0.1$ and $\epsilon=0.5$, and use the Chebyshev scalarization. Similar to \citep{Paria2019AFF}, we sample from $P_\lambda$ as $\lambda=\alpha/\lVert \alpha \rVert_1$, where $\alpha_i=\norm{u}_1/u_i$, $i \in [n]$, and $u$ is sampled uniformly from $[0,1]^n$. We compare the algorithms on the following MOO problems and plot mean and standard deviation (over $10$ independent trials) of the time-average cumulative regret $\frac{1}{T}R_C(T)$ in Fig. \ref{fig:cumulative-regret-plot}: (b)-(f). (More details in appendix \ref{app:experiments}.)

\begin{figure}[t!]
\centering
\subfigure[Tasks with varying similarities]{\includegraphics[height=1.15in,width=1.75in]{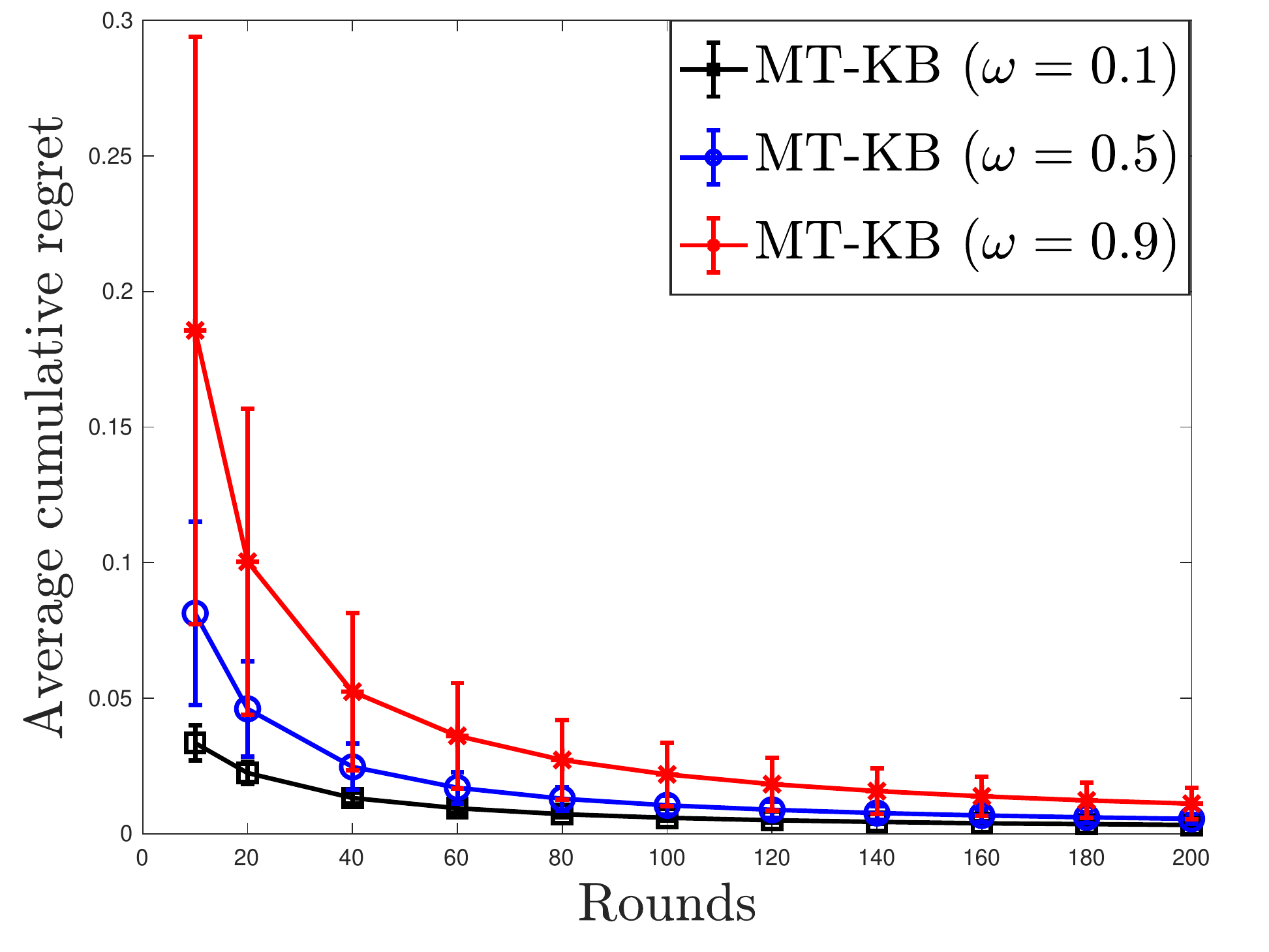}}
\subfigure[RKHS function ($2$ tasks)]{\includegraphics[height=1.15in,width=1.75in]{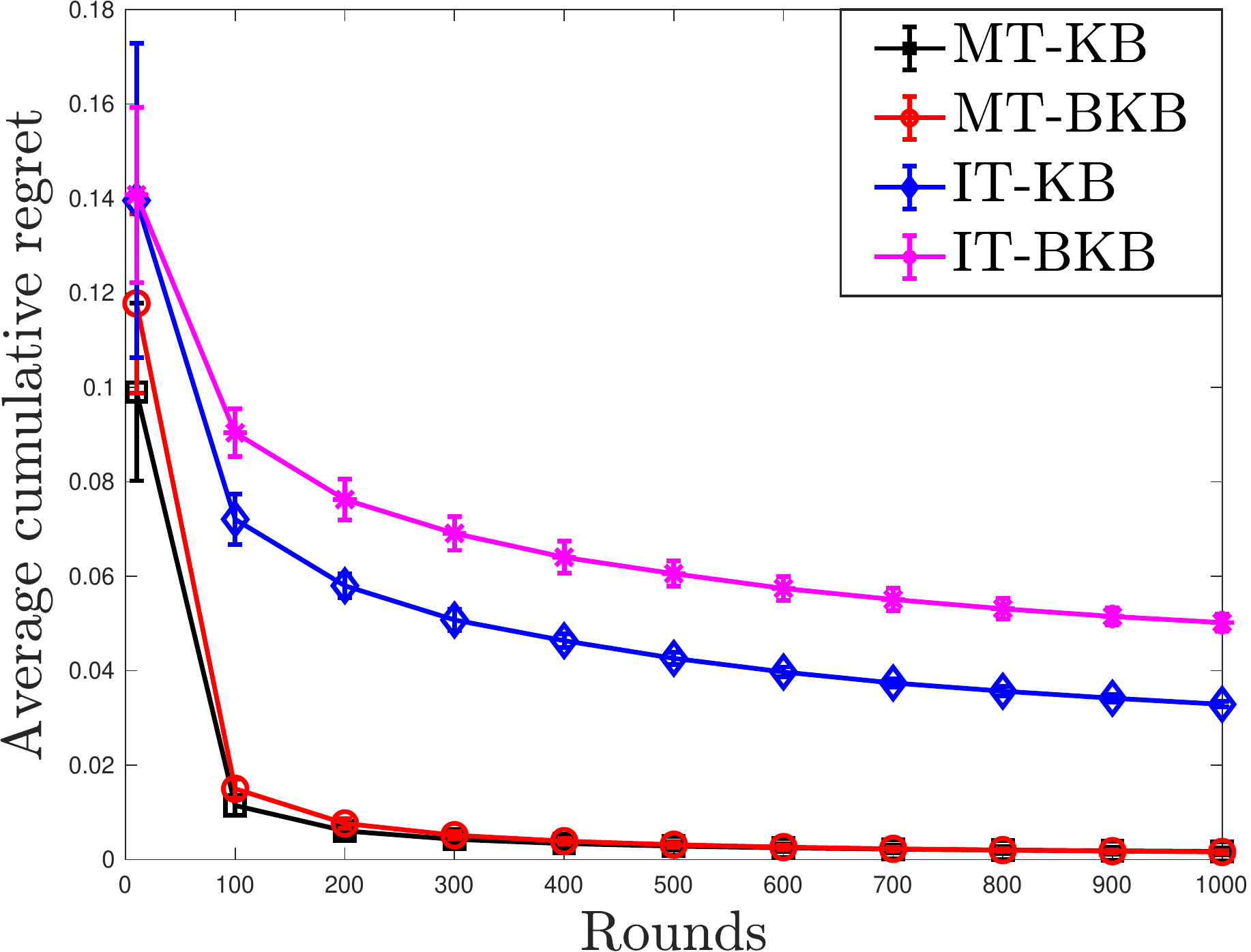}}
\subfigure[RKHS function ($20$ tasks)]{\includegraphics[height=1.15in,width=1.75in]{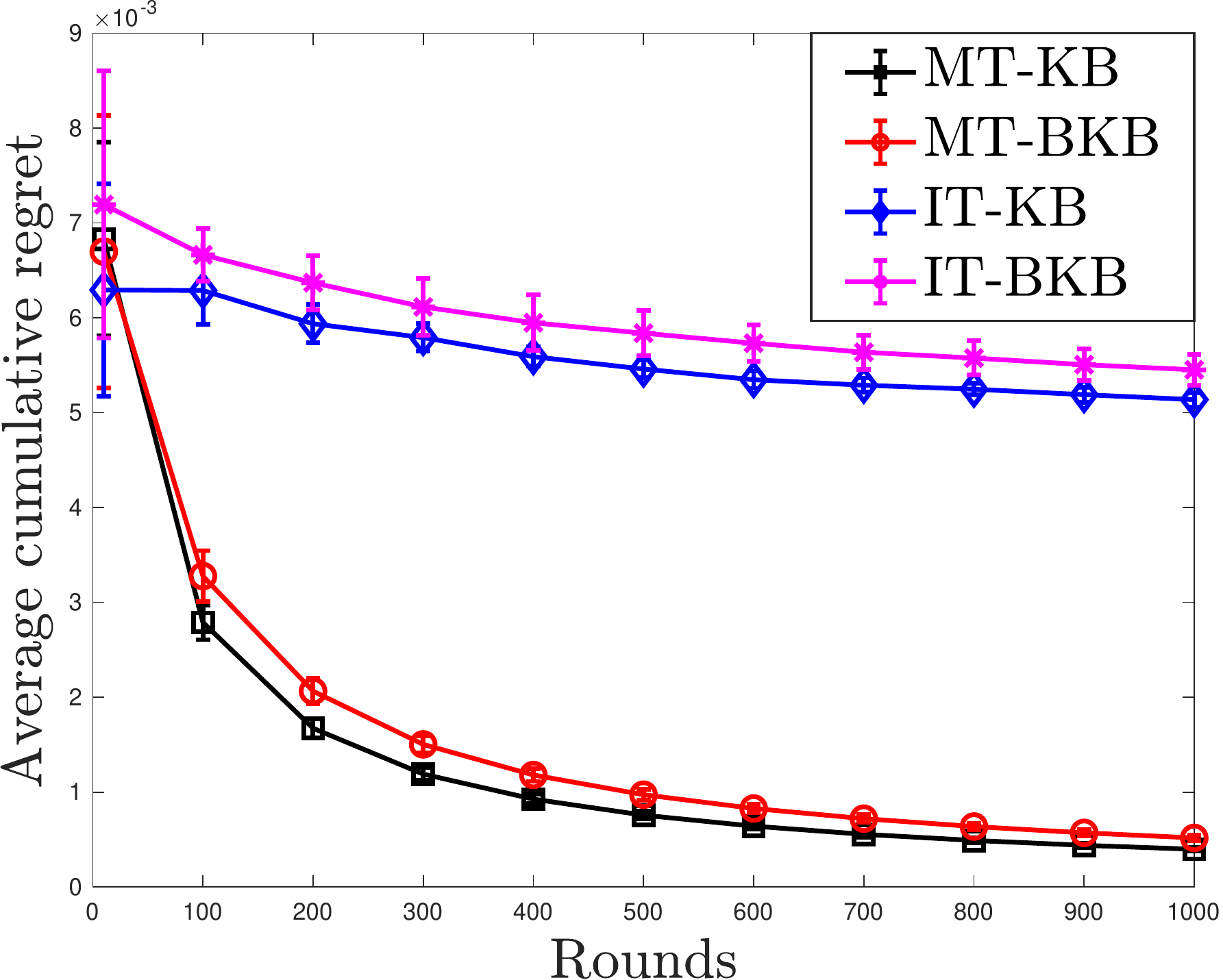}}
\subfigure[Perturbed sine function]{\includegraphics[height=1.15in,width=1.75in]{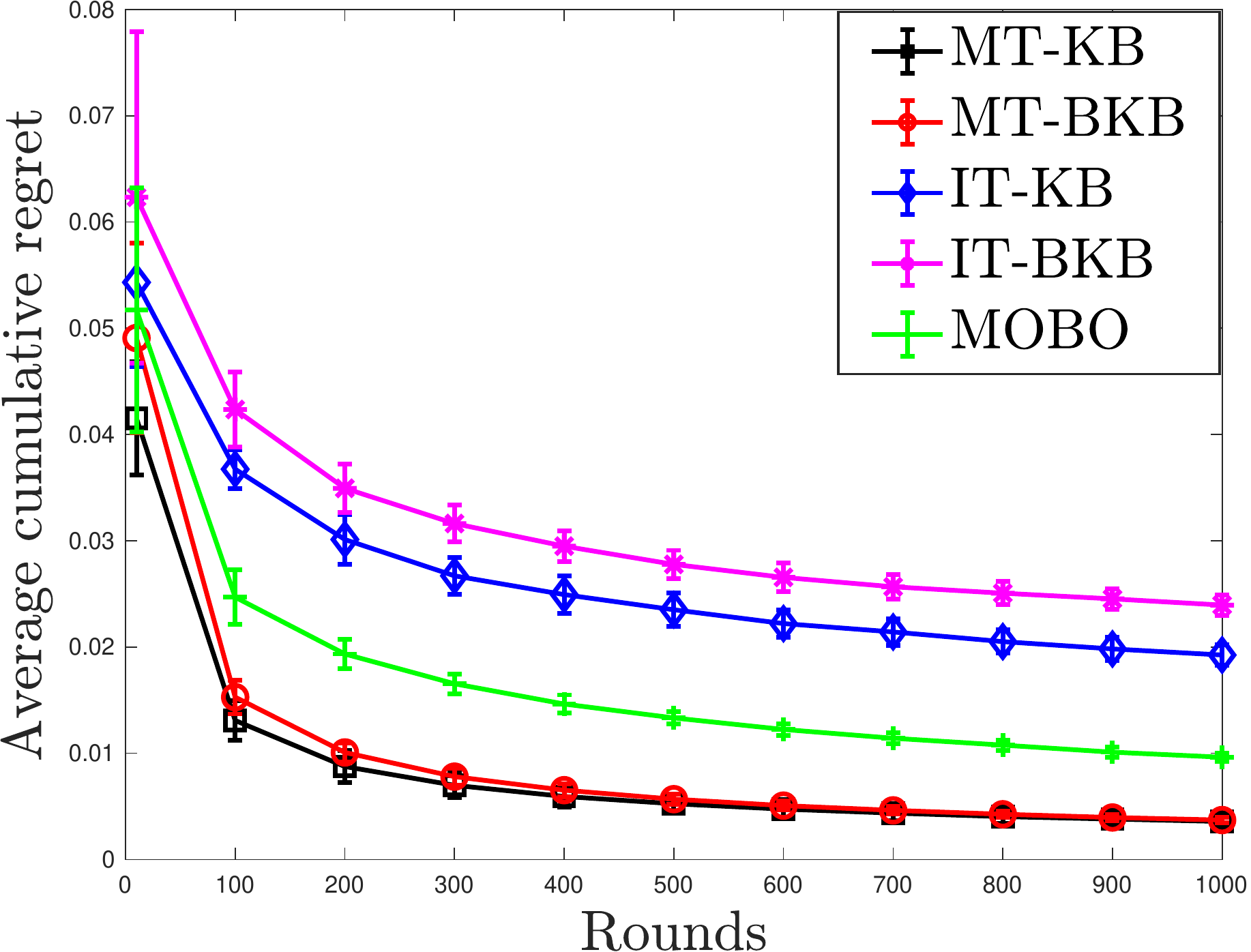}}
\subfigure[Shifted Branin-Hoo]{\includegraphics[height=1.15in,width=1.75in]{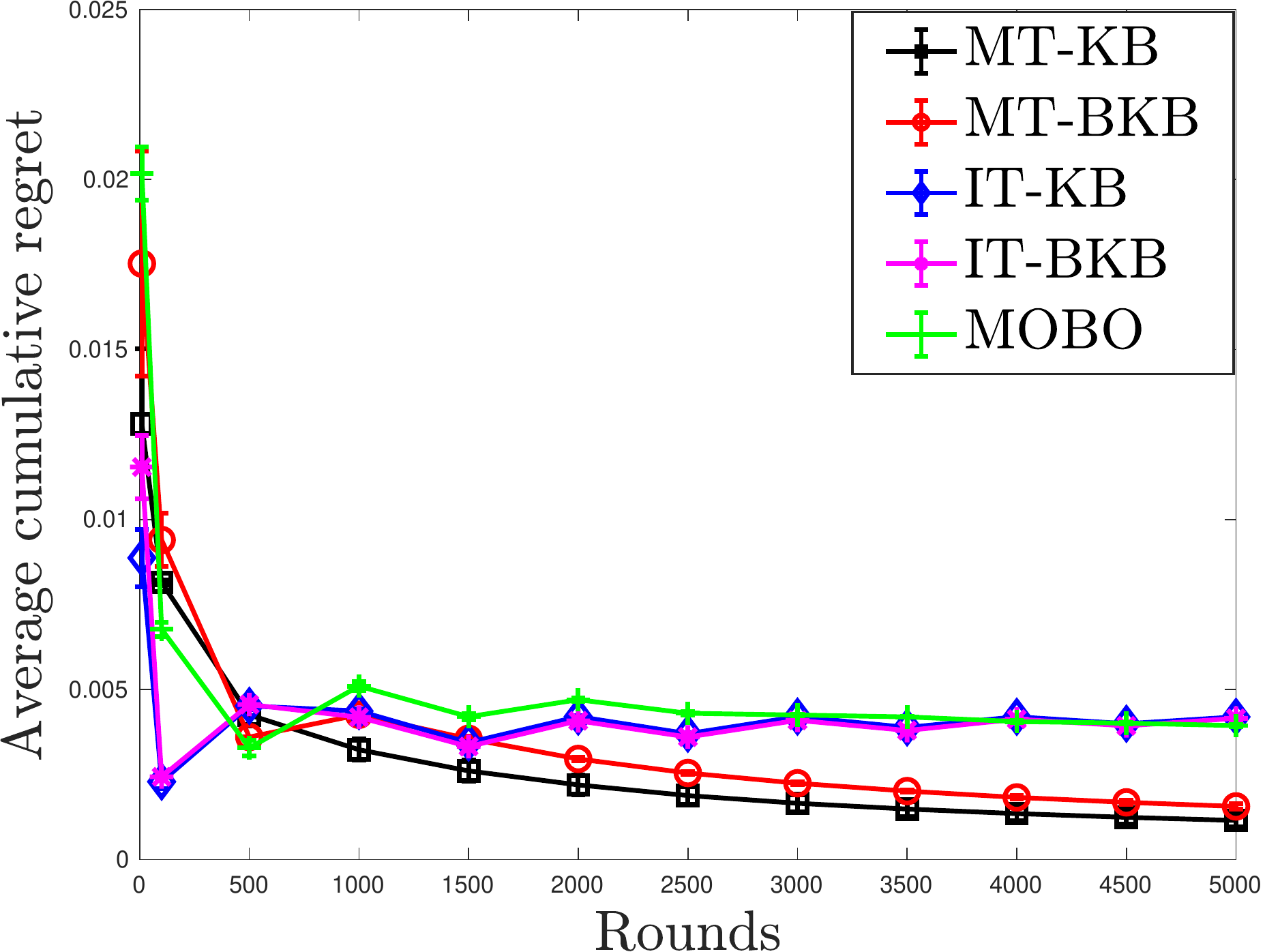}}
\subfigure[Sensor measurements]{\includegraphics[height=1.15in,width=1.75in]{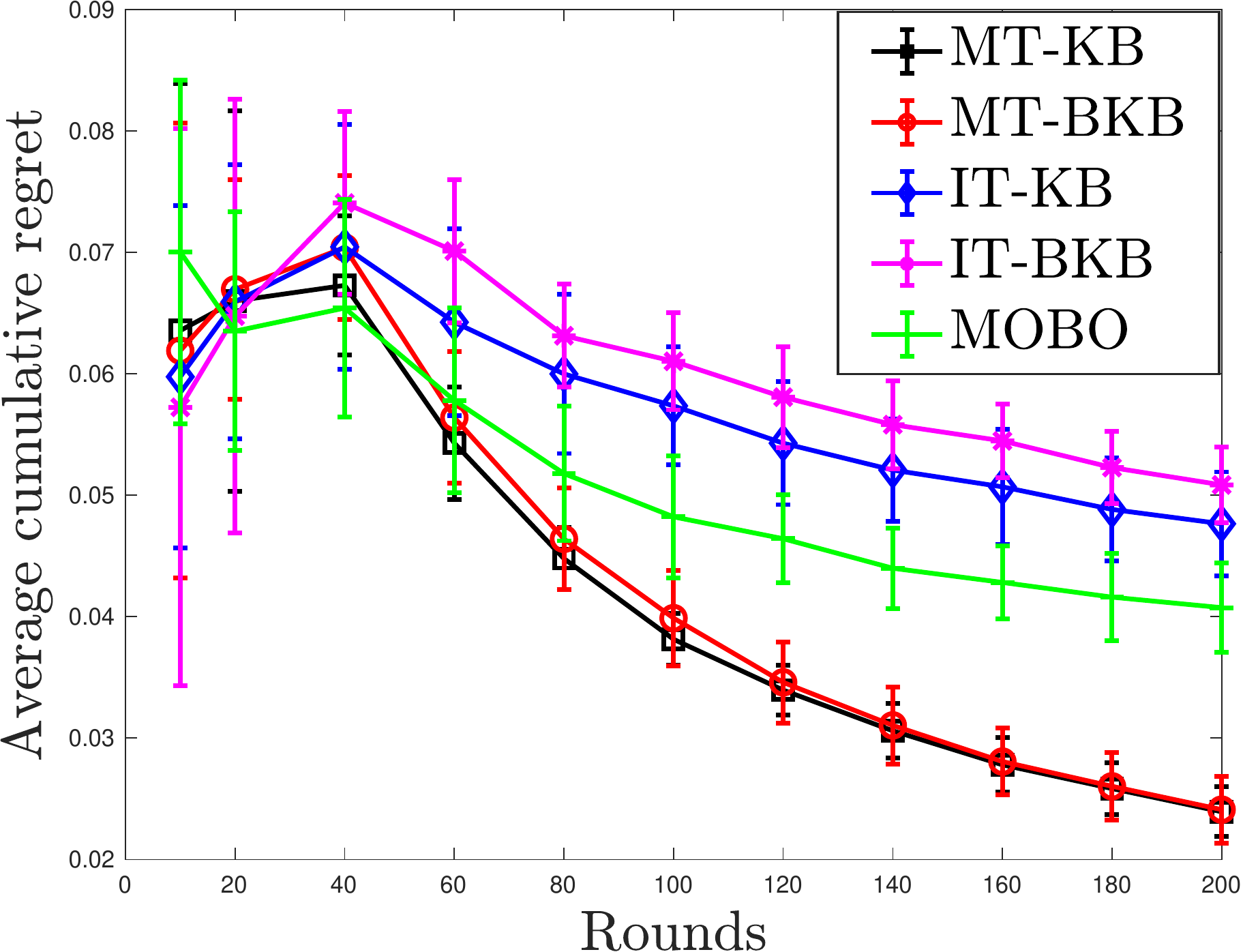}}
\caption{(a) Regret performance of MT-KB under varying inter-task similarities, (b)-(f) Comparison of average cumulative regret of MT-KB and MT-BKB with IT-KB, IT-BKB and MOBO on different MOO problems.
}
\label{fig:cumulative-regret-plot} 
\end{figure}

\paragraph{RKHS function} We generate a vector-valued RKHS element as $f(\cdot)=\sum_{i\leq 50}\Gamma(\cdot,x_i)c_i$, where the domain $\cX$ is an $0.01$-net of the interval $[0,1]$, each $x_i \in \cX$ and each $c_i$ is uniformly sampled from $[-1,1]^n$. We consider the ICM kernel $\Gamma(x,x')=k(x,x')B$ adopting a SE kernel with lengthscale $0.2$ for its scalar part and set $B = A^{\top}A$, where the elements of the $n \times n$ matrix $A$ is uniformly sampled from $[0,1]$. We set $\kappa$ as the largest eigenvalue of $B$ and bound the RKHS norm of $f$ using $b=\max_{x} \norm{f(x)}_2/\kappa$. The noise vectors are taken i.i.d. $\cN(0,\sigma^2I_n)$, $\sigma=0.1$. We compare the algorithms for $n=2$ and $n=20$ tasks. We observe that learning with MT kernels is much faster than learning the tasks independently -- even more so when no. of tasks are higher (Fig. \ref{fig:cumulative-regret-plot}: (b), (c)).

\paragraph{Perturbed sine function} We study a setting similar to \citep{baldassarre2012multi}, where $\cX$ is an $0.01$-net of the interval $[0,1]$ and we have $n=4$ tasks. Each task is given by a function $f_i(x) = \sin(2\pi x)+ 0.6f^{\text{pert}}_i(x)$ corrupted
by Gaussian noise of variance $0.01$. Each perturbation function $f^{\text{pert}}_i$ is a weighted sum of three Gaussians of width $0.1$ centered at $x_1 = 0.05$, $x_2 = 0.4$ and $x_3 = 0.7$, where task-specific weights are carefully chosen in order to yield tasks that are related by the common function, but also have local differences. We use the kernel $\Gamma_{\omega}(x,x')=k(x,x')\left(\omega I_n  + (1-\omega)1_n/n\right)$ that imposes a common similarity among all components and results are shown for $\omega=0.4$ (Fig. \ref{fig:cumulative-regret-plot}: (d)).

\paragraph{Shifted Branin-Hoo} The Branin-Hoo function, defined over a subset of $\Real^2$, is a common benchmark for BO \citep{jones2001taxonomy}. We consider $9$ shifted Branin-Hoo's as related tasks, where the $i$-th task is a translation of the function by $i\%$ along either axis, and run algorithms with the kernel $\Gamma_{\omega}$, $\omega=0.5$ (Fig. \ref{fig:cumulative-regret-plot}: (e)).

\paragraph{Sensor measurements}  
We take temperature, light and humidity measurements from 54 sensors collected in the Intel Berkeley lab \citep{srinivas2010gaussian} in the context of MOO. We have $3$ tasks, one for each variable, and each task $f_i(x)$ is given by the empirical mean of $50\%$ of the readings recorded at the sensor placed at location $x$. We take remaining readings to estimate an ICM kernel and run our algorithms with this kernel. Specifically, for its scalar part, we fit an SE kernel on sensor locations, and for its matrix part, we estimate inter-task similarities as $B=\frac{1}{m}R^{\top}K^{-1}R$, where $m$ denotes number of readings, $R$ is an $m \times 3$ matrix of readings for all tasks and $K$ is the $m \times m$ gram matrix of SE kernel. The idea is to de-correlate $R$ with $K^{-1}$ first so that only correlation with respect to $B$ is left. Further, we compute the empirical variance of sensor readings for each task and take the largest of those as $\sigma^2$. We see that the regret performance of MT-KB and MT-BKB are much better than IT-KB, IT-BKB and MOBO that do not use the inter-task structure in the form of the matrix $B$ (Fig. \ref{fig:cumulative-regret-plot}: (f)).

\section{Concluding remarks}
To the best of our knowledge, we prove the first rigorous regret bounds for multi-task Bayesian optimization that capture inter-task dependencies. We have demonstrated the shortcoming of modelling each task independently without making use of task similarities, and developed algorithms using multi-task kernels, which perform well in practice. We believe that our regret bounds are tight in terms of dependence on the time horizon. However, whether the dependence on the inter-task structure is optimal or not remains an important open question. It would also be interesting to see whether our multi-task concentration can be applied to several other interesting settings, for example optimizing under heavy-tailed corruptions \citep{chowdhury2019bayesian}, with a batch of inputs \citep{desautels2014parallelizing}, learning with kernel mean embeddings \citep{Chowdhury2020ActiveLO}, modelling the tranisition structure of a Markov decision process \citep{chowdhury2019online} to name a few.

\bibliographystyle{plainnat}
\bibliography{main}

\begin{thebibliography}{50}
\providecommand{\natexlab}[1]{#1}
\providecommand{\url}[1]{\texttt{#1}}
\expandafter\ifx\csname urlstyle\endcsname\relax
  \providecommand{\doi}[1]{doi: #1}\else
  \providecommand{\doi}{doi: \begingroup \urlstyle{rm}\Url}\fi

\bibitem[Alaoui and Mahoney(2015)]{alaoui2015fast}
Ahmed Alaoui and Michael~W Mahoney.
\newblock Fast randomized kernel ridge regression with statistical guarantees.
\newblock In \emph{Advances in Neural Information Processing Systems}, pages
  775--783, 2015.

\bibitem[Alvarez et~al.(2011)Alvarez, Rosasco, and
  Lawrence]{alvarez2011kernels}
Mauricio~A Alvarez, Lorenzo Rosasco, and Neil~D Lawrence.
\newblock Kernels for vector-valued functions: A review.
\newblock \emph{arXiv preprint arXiv:1106.6251}, 2011.

\bibitem[Astudillo and Frazier(2019)]{astudillo2019bayesian}
Raul Astudillo and Peter Frazier.
\newblock Bayesian optimization of composite functions.
\newblock In \emph{International Conference on Machine Learning}, pages
  354--363, 2019.

\bibitem[Baldassarre et~al.(2012)Baldassarre, Rosasco, Barla, and
  Verri]{baldassarre2012multi}
Luca Baldassarre, Lorenzo Rosasco, Annalisa Barla, and Alessandro Verri.
\newblock Multi-output learning via spectral filtering.
\newblock \emph{Machine learning}, 87\penalty0 (3):\penalty0 259--301, 2012.

\bibitem[Belakaria et~al.(2019)Belakaria, Deshwal, and Doppa]{belakaria2019max}
Syrine Belakaria, Aryan Deshwal, and Janardhan~Rao Doppa.
\newblock Max-value entropy search for multi-objective bayesian optimization.
\newblock In \emph{Advances in Neural Information Processing Systems}, pages
  7823--7833, 2019.

\bibitem[Bonilla et~al.(2008)Bonilla, Chai, and Williams]{bonilla2008multi}
Edwin~V Bonilla, Kian~M Chai, and Christopher Williams.
\newblock Multi-task gaussian process prediction.
\newblock In \emph{Advances in neural information processing systems}, pages
  153--160, 2008.

\bibitem[Brault et~al.(2019)Brault, Lambert, Szabo, Sangnier, and d’Alche
  Buc]{brault2019infinite}
Romain Brault, Alex Lambert, Zoltan Szabo, Maxime Sangnier, and Florence
  d’Alche Buc.
\newblock Infinite task learning in rkhss.
\newblock In \emph{The 22nd International Conference on Artificial Intelligence
  and Statistics}, pages 1294--1302, 2019.

\bibitem[Brochu et~al.(2010)Brochu, Cora, and De~Freitas]{brochu2010tutorial}
Eric Brochu, Vlad~M Cora, and Nando De~Freitas.
\newblock A tutorial on bayesian optimization of expensive cost functions, with
  application to active user modeling and hierarchical reinforcement learning.
\newblock \emph{arXiv preprint arXiv:1012.2599}, 2010.

\bibitem[Calandriello et~al.(2019)Calandriello, Carratino, Lazaric, Valko, and
  Rosasco]{calandriello2019gaussian}
Daniele Calandriello, Luigi Carratino, Alessandro Lazaric, Michal Valko, and
  Lorenzo Rosasco.
\newblock Gaussian process optimization with adaptive sketching: Scalable and
  no regret.
\newblock \emph{In Conference on Learning Theory}, 2019.

\bibitem[Caponnetto et~al.(2008)Caponnetto, Micchelli, Pontil, and
  Ying]{caponnetto2008universal}
Andrea Caponnetto, Charles~A Micchelli, Massimiliano Pontil, and Yiming Ying.
\newblock Universal multi-task kernels.
\newblock \emph{Journal of Machine Learning Research}, 9\penalty0
  (Jul):\penalty0 1615--1646, 2008.

\bibitem[Carmeli et~al.(2010)Carmeli, De~Vito, Toigo, and
  Umanit{\'a}]{carmeli2010vector}
Claudio Carmeli, Ernesto De~Vito, Alessandro Toigo, and Veronica Umanit{\'a}.
\newblock Vector valued reproducing kernel hilbert spaces and universality.
\newblock \emph{Analysis and Applications}, 8\penalty0 (01):\penalty0 19--61,
  2010.

\bibitem[Caruana(1997)]{caruana1997multitask}
Rich Caruana.
\newblock Multitask learning.
\newblock \emph{Machine learning}, 28\penalty0 (1):\penalty0 41--75, 1997.

\bibitem[Chowdhury et~al.(2020)Chowdhury, dos Santos~de Oliveira, and
  Ramos]{Chowdhury2020ActiveLO}
S.~R. Chowdhury, Rafael dos Santos~de Oliveira, and F.~Ramos.
\newblock Active learning of conditional mean embeddings via bayesian
  optimisation.
\newblock 2020.

\bibitem[Chowdhury and Gopalan(2017)]{chowdhury2017kernelized}
Sayak~Ray Chowdhury and Aditya Gopalan.
\newblock On kernelized multi-armed bandits.
\newblock In \emph{Proceedings of the 34th International Conference on Machine
  Learning-Volume 70}, pages 844--853. JMLR. org, 2017.

\bibitem[Chowdhury and Gopalan(2019{\natexlab{a}})]{chowdhury2019bayesian}
Sayak~Ray Chowdhury and Aditya Gopalan.
\newblock Bayesian optimization under heavy-tailed payoffs.
\newblock In \emph{Advances in Neural Information Processing Systems}, pages
  13790--13801, 2019{\natexlab{a}}.

\bibitem[Chowdhury and Gopalan(2019{\natexlab{b}})]{chowdhury2019online}
Sayak~Ray Chowdhury and Aditya Gopalan.
\newblock Online learning in kernelized markov decision processes.
\newblock In \emph{The 22nd International Conference on Artificial Intelligence
  and Statistics}, pages 3197--3205, 2019{\natexlab{b}}.

\bibitem[Desautels et~al.(2014)Desautels, Krause, and
  Burdick]{desautels2014parallelizing}
Thomas Desautels, Andreas Krause, and Joel~W Burdick.
\newblock Parallelizing exploration-exploitation tradeoffs in gaussian process
  bandit optimization.
\newblock \emph{Journal of Machine Learning Research}, 15:\penalty0 3873--3923,
  2014.

\bibitem[Drineas and Mahoney(2005)]{drineas2005nystrom}
Petros Drineas and Michael~W Mahoney.
\newblock On the nystr{\"o}m method for approximating a gram matrix for
  improved kernel-based learning.
\newblock \emph{journal of machine learning research}, 6\penalty0
  (Dec):\penalty0 2153--2175, 2005.

\bibitem[Drugan and Nowe(2013)]{drugan2013designing}
Madalina~M Drugan and Ann Nowe.
\newblock Designing multi-objective multi-armed bandits algorithms: A study.
\newblock In \emph{The 2013 International Joint Conference on Neural Networks
  (IJCNN)}, pages 1--8. IEEE, 2013.

\bibitem[Drugan and Now{\'e}(2014)]{drugan2014scalarization}
Madalina~M Drugan and Ann Now{\'e}.
\newblock Scalarization based pareto optimal set of arms identification
  algorithms.
\newblock In \emph{2014 International Joint Conference on Neural Networks
  (IJCNN)}, pages 2690--2697. IEEE, 2014.

\bibitem[Durand et~al.(2018)Durand, Maillard, and Pineau]{durand2018streaming}
Audrey Durand, Odalric-Ambrym Maillard, and Joelle Pineau.
\newblock Streaming kernel regression with provably adaptive mean, variance,
  and regularization.
\newblock \emph{The Journal of Machine Learning Research}, 19\penalty0
  (1):\penalty0 650--683, 2018.

\bibitem[Emmerich and Klinkenberg(2008)]{emmerich2008computation}
Michael Emmerich and Jan-willem Klinkenberg.
\newblock The computation of the expected improvement in dominated hypervolume
  of pareto front approximations.
\newblock \emph{Rapport technique, Leiden University}, 34:\penalty0 7--3, 2008.

\bibitem[Evgeniou et~al.(2005)Evgeniou, Micchelli, and
  Pontil]{evgeniou2005learning}
Theodoros Evgeniou, Charles~A Micchelli, and Massimiliano Pontil.
\newblock Learning multiple tasks with kernel methods.
\newblock \emph{Journal of machine learning research}, 6\penalty0
  (Apr):\penalty0 615--637, 2005.

\bibitem[Garnett et~al.(2010)Garnett, Osborne, and
  Roberts]{GarOsbRob10:BOsensor}
R.~Garnett, M.~A. Osborne, and S.~J. Roberts.
\newblock Bayesian optimization for sensor set selection.
\newblock In \emph{Proceedings of the 9th ACM/IEEE International Conference on
  Information Processing in Sensor Networks}, IPSN '10, pages 209--219, New
  York, NY, USA, 2010. ACM.

\bibitem[Gonzalez et~al.(2015)Gonzalez, Longworth, James, and
  Lawrence]{gonzalez2015bayesian}
Javier Gonzalez, Joseph Longworth, David~C James, and Neil~D Lawrence.
\newblock Bayesian optimization for synthetic gene design.
\newblock \emph{arXiv preprint arXiv:1505.01627}, 2015.

\bibitem[Greene(2003)]{greene2003econometric}
William~H Greene.
\newblock \emph{Econometric analysis}.
\newblock Pearson Education India, 2003.

\bibitem[Gr{\"u}new{\"a}lder et~al.(2012)Gr{\"u}new{\"a}lder, Lever,
  Baldassarre, Patterson, Gretton, and Pontil]{grunewalder2012conditional}
Steffen Gr{\"u}new{\"a}lder, Guy Lever, Luca Baldassarre, Sam Patterson, Arthur
  Gretton, and Massimilano Pontil.
\newblock Conditional mean embeddings as regressors.
\newblock In \emph{Proceedings of the 29th International Coference on
  International Conference on Machine Learning}, pages 1803--1810, 2012.

\bibitem[Guo et~al.(2017)Guo, Pleiss, Sun, and Weinberger]{guo2017calibration}
Chuan Guo, Geoff Pleiss, Yu~Sun, and Kilian~Q Weinberger.
\newblock On calibration of modern neural networks.
\newblock In \emph{Proceedings of the 34th International Conference on Machine
  Learning-Volume 70}, pages 1321--1330. JMLR. org, 2017.

\bibitem[Hern{\'a}ndez-Lobato et~al.(2016)Hern{\'a}ndez-Lobato,
  Hernandez-Lobato, Shah, and Adams]{hernandez2016predictive}
Daniel Hern{\'a}ndez-Lobato, Jose Hernandez-Lobato, Amar Shah, and Ryan Adams.
\newblock Predictive entropy search for multi-objective bayesian optimization.
\newblock In \emph{International Conference on Machine Learning}, pages
  1492--1501, 2016.

\bibitem[Jones(2001)]{jones2001taxonomy}
Donald~R Jones.
\newblock A taxonomy of global optimization methods based on response surfaces.
\newblock \emph{Journal of global optimization}, 21\penalty0 (4):\penalty0
  345--383, 2001.

\bibitem[Kadri et~al.(2016)Kadri, Duflos, Preux, Canu, Rakotomamonjy, and
  Audiffren]{kadri2016operator}
Hachem Kadri, Emmanuel Duflos, Philippe Preux, St{\'e}phane Canu, Alain
  Rakotomamonjy, and Julien Audiffren.
\newblock Operator-valued kernels for learning from functional response data.
\newblock \emph{The Journal of Machine Learning Research}, 17\penalty0
  (1):\penalty0 613--666, 2016.

\bibitem[Knowles(2006)]{knowles2006parego}
Joshua Knowles.
\newblock Parego: a hybrid algorithm with on-line landscape approximation for
  expensive multiobjective optimization problems.
\newblock \emph{IEEE Transactions on Evolutionary Computation}, 10\penalty0
  (1):\penalty0 50--66, 2006.

\bibitem[Liu et~al.(2018)Liu, Cai, and Ong]{liu2018remarks}
Haitao Liu, Jianfei Cai, and Yew-Soon Ong.
\newblock Remarks on multi-output gaussian process regression.
\newblock \emph{Knowledge-Based Systems}, 144:\penalty0 102--121, 2018.

\bibitem[Micchelli and Pontil(2005)]{micchelli2005learning}
Charles~A Micchelli and Massimiliano Pontil.
\newblock On learning vector-valued functions.
\newblock \emph{Neural computation}, 17\penalty0 (1):\penalty0 177--204, 2005.

\bibitem[Nakayama et~al.(2009)Nakayama, Yun, and Yoon]{nakayama2009sequential}
Hirotaka Nakayama, Yeboon Yun, and Min Yoon.
\newblock \emph{Sequential approximate multiobjective optimization using
  computational intelligence}.
\newblock Springer Science \& Business Media, 2009.

\bibitem[Paria et~al.(2019)Paria, Kandasamy, and P{\'o}czos]{Paria2019AFF}
Biswajit Paria, Kirthevasan Kandasamy, and B.~P{\'o}czos.
\newblock A flexible framework for multi-objective bayesian optimization using
  random scalarizations.
\newblock In \emph{UAI}, 2019.

\bibitem[Picheny(2015)]{picheny2015multiobjective}
Victor Picheny.
\newblock Multiobjective optimization using gaussian process emulators via
  stepwise uncertainty reduction.
\newblock \emph{Statistics and Computing}, 25\penalty0 (6):\penalty0
  1265--1280, 2015.

\bibitem[Ponweiser et~al.(2008)Ponweiser, Wagner, Biermann, and
  Vincze]{ponweiser2008multiobjective}
Wolfgang Ponweiser, Tobias Wagner, Dirk Biermann, and Markus Vincze.
\newblock Multiobjective optimization on a limited budget of evaluations using
  model-assisted $\mathcal{S}$-metric selection.
\newblock In \emph{International Conference on Parallel Problem Solving from
  Nature}, pages 784--794. Springer, 2008.

\bibitem[Rasmussen(2003)]{rasmussen2003gaussian}
Carl~Edward Rasmussen.
\newblock Gaussian processes in machine learning.
\newblock In \emph{Summer School on Machine Learning}, pages 63--71. Springer,
  2003.

\bibitem[Rifkin et~al.(2003)Rifkin, Mukherjee, Tamayo, Ramaswamy, Yeang,
  Angelo, Reich, Poggio, Lander, Golub, et~al.]{rifkin2003analytical}
Ryan Rifkin, Sayan Mukherjee, Pablo Tamayo, Sridhar Ramaswamy, Chen-Hsiang
  Yeang, Michael Angelo, Michael Reich, Tomaso Poggio, Eric~S Lander, Todd~R
  Golub, et~al.
\newblock An analytical method for multiclass molecular cancer classification.
\newblock \emph{Siam Review}, 45\penalty0 (4):\penalty0 706--723, 2003.

\bibitem[Roijers et~al.(2013)Roijers, Vamplew, Whiteson, and
  Dazeley]{roijers2013survey}
Diederik~M Roijers, Peter Vamplew, Shimon Whiteson, and Richard Dazeley.
\newblock A survey of multi-objective sequential decision-making.
\newblock \emph{Journal of Artificial Intelligence Research}, 48:\penalty0
  67--113, 2013.

\bibitem[Scarlett et~al.(2017)Scarlett, Bogunovic, and
  Cevher]{scarlett2017lower}
Jonathan Scarlett, Ilija Bogunovic, and Volkan Cevher.
\newblock Lower bounds on regret for noisy gaussian process bandit
  optimization.
\newblock In \emph{Conference on Learning Theory}, pages 1723--1742, 2017.

\bibitem[Shahriari et~al.(2015)Shahriari, Swersky, Wang, Adams, and
  De~Freitas]{shahriari2015taking}
Bobak Shahriari, Kevin Swersky, Ziyu Wang, Ryan~P Adams, and Nando De~Freitas.
\newblock Taking the human out of the loop: A review of bayesian optimization.
\newblock \emph{Proceedings of the IEEE}, 104\penalty0 (1):\penalty0 148--175,
  2015.

\bibitem[Snoek et~al.(2012)Snoek, Larochelle, and Adams]{snoek2012practical}
Jasper Snoek, Hugo Larochelle, and Ryan~P Adams.
\newblock Practical bayesian optimization of machine learning algorithms.
\newblock In \emph{Advances in neural information processing systems}, pages
  2951--2959, 2012.

\bibitem[Srinivas et~al.(2010)Srinivas, Krause, Kakade, and
  Seeger]{srinivas2010gaussian}
Niranjan Srinivas, Andreas Krause, Sham Kakade, and Matthias Seeger.
\newblock Gaussian process optimization in the bandit setting: no regret and
  experimental design.
\newblock In \emph{Proceedings of the 27th International Conference on
  International Conference on Machine Learning}, pages 1015--1022. Omnipress,
  2010.

\bibitem[Swersky et~al.(2013)Swersky, Snoek, and Adams]{swersky2013multi}
Kevin Swersky, Jasper Snoek, and Ryan~P Adams.
\newblock Multi-task bayesian optimization.
\newblock In \emph{Advances in neural information processing systems}, pages
  2004--2012, 2013.

\bibitem[Wackernagel(2013)]{wackernagel2013multivariate}
Hans Wackernagel.
\newblock \emph{Multivariate geostatistics: an introduction with applications}.
\newblock Springer Science \& Business Media, 2013.

\bibitem[Wang et~al.(2018)Wang, Gehring, Kohli, and Jegelka]{wang2018batched}
Zi~Wang, Clement Gehring, Pushmeet Kohli, and Stefanie Jegelka.
\newblock Batched large-scale bayesian optimization in high-dimensional spaces.
\newblock In \emph{International Conference on Artificial Intelligence and
  Statistics}, pages 745--754, 2018.

\bibitem[Zliobaite(2015)]{zliobaite2015relation}
Indre Zliobaite.
\newblock On the relation between accuracy and fairness in binary
  classification.
\newblock \emph{arXiv preprint arXiv:1505.05723}, 2015.

\bibitem[Zuluaga et~al.(2013)Zuluaga, Sergent, Krause, and
  P{\"u}schel]{zuluaga2013active}
Marcela Zuluaga, Guillaume Sergent, Andreas Krause, and Markus P{\"u}schel.
\newblock Active learning for multi-objective optimization.
\newblock In \emph{International Conference on Machine Learning}, pages
  462--470, 2013.

\end{thebibliography}

\newpage

\begin{appendix}
\begin{center}
\huge{Appendix}
\end{center}

\section{Computational complexity under ICM kernels}
\label{app:complexity}
In this section, we describe the time complexities of MT-KB and MT-BKB for the intrinsic coregionalization model (ICM) $\Gamma(x,x')=k(x,x')B$. As discussed earlier, we assume that an efficient oracle to optimize the acquisition function is provided to us, and the per step cost comes only from computing it. To this end, we first describe simplified model updates under ICM kernel using the eigen-system of $B$ and then detail out the time required for computing the updates. We note here that the eigen decomposition, which is $O(n^3)$, needs to be computed only once at the beginning and can be used at every step of the algorithms.
\paragraph{Per-step complexity of MT-KB}
Let $B=\sum_{i=1}^{n}\xi_iu_iu_i^{\top}$ denotes the eigen decomposition of the positive semi-definite matrix $B$. Then, $\Gamma(x,x)=\sum_{i=1}^{n}\xi_ik(x,x)u_iu_i^{\top}$. From the definition of the Kronecker product, we now have $G_t=\sum_{i=1}^{n}\xi_i K_t \otimes u_iu_i^{\top}$ and $G_t(x)=\sum_{i=1}^{n}\xi_i k_t(x) \otimes u_iu_i^{\top}$, where $K_t=\left[k(x_i,x_j)\right]_{i,j=1}^{t}$ and $k_t(x)=\left[k(x_1,x),\ldots,k(x_t,x)\right]^{\top}$. Since $\lbrace u_i \rbrace_{i=1}^{n}$ yields an orthonormal basis of $\Real^n$, the output $y_t \in \Real^n$ can be written as $y_t=\sum_{i=1}^{n}y_t^{\top}u_i\cdot u_i$. We then have $Y_t=\sum_{i=1}^{n}Y_{t}^{i}\otimes u_i$, where $Y_{t}^{i}=\left[y_1^{\top}u_i,\ldots,y_t^{\top}u_i\right]^{\top}$. We also note that $I_{nt}=\sum_{i=1}^{n}I_t\otimes u_iu_i^{\top}$, and, therefore
$G_t+\eta I_{nt}=\sum_{i=1}^{n}\left(\xi_iK_t+\eta I_t\right)\otimes u_iu_i^{\top}$. Now, let $K_t=\sum_{j=1}^{t}\alpha_jw_jw_j^{\top}$ denotes the eigen decomposition of the (positive semi-definite) kernel matrix $K_t$. We then have
\beq
G_t+\eta I_{nt}=\sum_{i=1}^{n}\sum_{j=1}^{t}(\xi_i\alpha_j+\eta)w_jw_j^{\top}\otimes u_iu_i^{\top}=\sum_{i=1}^{n}\sum_{j=1}^{t}(\xi_i\alpha_j+\eta)(w_j\otimes u_i)(w_j\otimes u_i)^{\top}.
\label{eqn:eig-dec}
\eeq
By the properties of tensor product $(w_j\otimes u_i)^{\top}(w_{j'}\otimes u_{i'})=(w_j^{\top}w_{j'})\cdot (u_i^{\top}u_{i'})$, which is equal to $1$ if $i=i'$, $j=j'$, and is equal to $0$ otherwise. Therefore, (\ref{eqn:eig-dec}) denotes the eigen decomposition of $G_t+\eta I_{nt}$. Hence
\beq
\left(G_t+\eta I_{nt}\right)^{-1}=\sum_{i=1}^{n}\sum_{j=1}^{t}\frac{1}{\xi_i\alpha_j+\eta}w_jw_j^{\top}\otimes u_iu_i^{\top}=\sum_{i=1}^{n}(\xi_iK_t+\eta I_t)^{-1}\otimes u_iu_i^{\top}.
\label{eqn:inverse}
\eeq
By the orthonormality of $\lbrace u_i \rbrace_{i=1}^{n}$ and the mixed product property of Kronecker product, we now obtain $(G_t+\eta I_{nt})^{-1}Y_{t} = \sum_{i=1}^{n}(\xi_iK_t+\eta I_t)^{-1}Y_t^i\otimes u_i$, and thus, in turn,
\beq
\mu_t(x)=G_{t}(x)^{\top}(G_t+\eta I_{nt})^{-1}Y_{t}= \sum_{i=1}^{n}\xi_ik_t(x)^{\top}(\xi_iK_t+\eta I_t)^{-1}Y_t^i \cdot u_i.
\label{eqn:eff-mean}
\eeq
Similarly, we get $G_{t}(x)^{\top}(G_t+\eta I_{nt})^{-1}G_t(x)=\sum_{i=1}^{n}\xi_i^2k_t(x)^{\top}(\xi_iK_t+\eta I_t)^{-1}k_t(x)\cdot u_iu_i^{\top}$ and therefore,
\beq
\norm{\Gamma_t(x,x)}= \max_{1 \leq i \leq n}\xi_i \left(k(x,x)- \xi_ik_t(x)^{\top}(\xi_iK_t+\eta I_t)^{-1}k_t(x)\right).
\label{eqn:eff-var}
\eeq
Let us now discuss the time required to compute $\mu_t(x)$ and $\norm{\Gamma_t(x,x)}$.
Given the eigen decomposition, updating $\lbrace Y_t^i \rbrace_{i=1}^{n}$ re-using those already computed at the previous step requires projecting the current output $y_t$ onto all coordinates, and thus, takes $O(n^2)$ time. Now, since the kernel matrix $K_t$ is rescaled by the eigenvalues $\xi_i$, we can find the eigen decomposition of $K_t$ once and reuse those to compute $\lbrace(\xi_iK_t+\eta I_t)^{-1}\rbrace_{i=1}^{n}$ in $O(t^3)$ time. Next, computing $n$ matrix-vector
multiplications and vector inner products of the form $k_t(x)^{\top}(\xi_iK_t+\eta I_t)^{-1}k_t(x)$ and $k_t(x)^{\top}(\xi_iK_t+\eta I_t)^{-1}Y_t^i$ take $O(nt^2)$ time. Finally, the sum in (\ref{eqn:eff-mean}) and the max in (\ref{eqn:eff-var}) can be computed in $O(n^2)$ and $O(n)$ time, respectively. Therefore, the overall cost to compute $\mu_t(x)$ and $\norm{\Gamma_t(x,x)}$ are $O\left(n^2+nt^2+t^3\right)=O\left(n^2+t^2(n+t)\right)$.

\paragraph{Per-step complexity of MT-BKB}
 Let $\tilde{\phi}_t(x)=\left(\tilde{K}_{t}^{1/2}\right)^{+}\tilde{k}_{t}(x) \in \Real^{m_t}$ denotes the Nystr\"{o}m embedding of the scalar kernel $k$, where $\tilde{k}_{t}(x)=\left[\frac{1}{\sqrt{p_{t,i_1}}}k(x_{i_1},x),\ldots,\frac{1}{\sqrt{p_{t,i_{m_t}}}}k(x_{i_{m_t}},x)\right]^{\top}$ and $\tilde{K}_{t}=\left[\frac{1}{\sqrt{p_{t,i_u}p_{t,i_v}}}k(x_{i_u},x_{i_v})\right]_{u,v=1}^{m_t}$. Then the eigen decomposition $B=\sum_{i=1}^{n}\xi_iu_iu_i^{\top}$ yields $\tilde{G}_t=\sum_{i=1}^{n}\xi_i \tilde{K}_t \otimes u_iu_i^{\top}$ and $\tilde{G}_t(x)=\sum_{i=1}^{n}\xi_i \tilde{k}_t(x) \otimes u_iu_i^{\top}$. A similar argument as in (\ref{eqn:eig-dec}) and (\ref{eqn:inverse}) now implies
$\left(\tilde{G}_t^{1/2}\right)^{+}=\sum_{i=1}^{n}\frac{1}{\sqrt{\xi_i}}\left(\tilde{K}_t^{1/2}\right)^{+} \otimes u_iu_i^{\top}$. Therefore, the Nystr\"{o}m embeddings for the multi-task kernel $\Gamma$ can be computed using the embeddings for the scalar kernel $k$ as
\beqn
\tilde{\Phi}_t(x)=\left(\tilde{G}_t^{1/2}\right)^{+}\tilde{G}_{t}(x)=\sum_{i=1}^{n}\sqrt{\xi_i}\left(\tilde{K}_{t}^{1/2}\right)^{+}\tilde{k}_{t}(x) \otimes u_iu_i^{\top}=\sum_{i=1}^{n}\sqrt{\xi_i}\tilde{\phi}_{t}(x) \otimes u_iu_i^{\top}. 
\label{eqn:nystrom-embedding}
\eeqn
We now have 
\beqn
\tilde{V}_t=\sum_{s=1}^{t}\tilde{\Phi}_t(x_s)\tilde{\Phi}_t(x_s)^{\top}=\sum_{s=1}^{t}\sum_{i=1}^{n}\xi_i\tilde{\phi}_t(x_s)\tilde{\phi}_t(x_s)^{\top}\otimes u_iu_i^{\top}=\sum_{i=1}^{n}\xi_i\tilde{v}_t \otimes u_iu_i^{\top},
\eeqn
where $\tilde{v}_t=\sum_{s=1}^{t}\tilde{\phi}_t(x_s)\tilde{\phi}_t(x_s)^{\top}$. A similar argument as in (\ref{eqn:eig-dec}) and (\ref{eqn:inverse}) then implies
\beqn
(\tilde{V}_t+\eta I_{nm_t})^{-1} = \sum_{i=1}^{n}\left(\xi_i\tilde{v}_t +\eta I_{m_t}\right)^{-1}\otimes u_iu_i^{\top}.
\eeqn
We further have 
\beqn
\sum_{s=1}^{t}\tilde{\Phi}_t(x_s)y_s=\sum_{s=1}^{t}\sum_{i=1}^{n}\sqrt{\xi_i}\cdot y_s^{\top}u_i\cdot\tilde{\phi}_{t}(x_s) \otimes u_i=\sum_{i=1}^{n}\sqrt{\xi_i}\left(\sum_{s=1}^{t}y_s^{\top}u_i\cdot \tilde{\phi}_t(x_s)\right)\otimes u_i.
\eeqn
Similar to (\ref{eqn:eff-mean}), we therefore obtain
\beq
\tilde{\mu}_t(x)=\sum_{i=1}^{n}\xi_i\tilde{\phi}_t(x)^{\top}\left(\xi_i\tilde{v}_t+\eta I_{m_t}\right)^{-1}\left(\sum_{s=1}^{t}y_s^{\top}u_i\cdot\tilde{\phi}_t(x_s)\right)\cdot u_i.
\label{eqn:eff-mean-approx}
\eeq
We now note that $\tilde{\Phi}_t(x)^{\top}\tilde{\Phi}_t(x)=\sum_{i=1}^{n}\xi_i \tilde{\phi}_t(x)^{\top}\tilde{\phi}_t(x)\cdot u_iu_i^{\top}$. Similar to (\ref{eqn:eff-var}), we then obtain
\beq
\norm{\tilde{\Gamma}_t(x,x)}=\max_{1 \leq i \leq n}\xi_i\left(k(x,x)-\tilde{\phi}_t(x)^{\top}\tilde{\phi}_t(x)+\eta \tilde{\phi}_t(x)^{\top} \left(\xi_i\tilde{v}_t+\eta I_{m_t}\right)^{-1}\tilde{\phi}_t(x)\right).
\label{eqn:eff-var-approx}
\eeq
We now discuss the time required to compute the scalar kernel embedding $\tilde{\phi}_t(x)$. Sampling the dictionary $\cD_t$, as we reuse the variances from the previous round, takes $O(t)$ time. We now compute the embedding $\tilde{\phi}_t(x)$ in $O(m_t^3+m_t^2)$ time, which corresponds to an inversion of $\tilde{K}_t^{1/2}$ and a matrix-vector product of dimension $m_t$, the size of the dictionary. Given the embedding function, let us now find the time required to compute $\tilde{\mu}_t(x)$ and $\lVert\tilde{\Gamma}_t(x,x)\rVert$. 
We first construct the matrix $\tilde{v}_t$ from scratch using all the points selected so far, which takes $O(m_t^2t)$ time. Then the inverses $\lbrace(\xi_i\tilde{v}_t+\eta I_{m_t})^{-1}\rbrace_{i=1}^{n}$ can be computed in $O(m_t^3)$ time and the matrix-vector multiplications $\lbrace(\xi_i\tilde{v}_t+\eta I_{m_t})^{-1}\tilde{\phi}_t(x)\rbrace_{i=1}^{n}$ in $O(nm_t^2)$ time. Similar to MT-KB, projecting the current output onto every direction takes $O(n^2)$ time. The projections can then be used to compute $n$ vectors of the form $\sum_{s=1}^{t}y_s^{\top}u_i\cdot\tilde{\phi}_t(x_s)$ in $O(nm_tt)$ time. Finally, $n$ vector inner products of dimension $m_t$ can be computed in $O(nm_t)$ time. Therefore, the overall cost to compute (\ref{eqn:eff-mean-approx}) and (\ref{eqn:eff-var-approx}) is  $O(n^2+nm_tt+nm_t^2+m_t^3+m_t^2t)=O\left(n^2+m_t t(n+t)\right)$, since the dictionary size $m_t \leq t$.


\section{Multi-task concentration}

We first introduce some notations. For any two Hilbert spaces $\cG$ and $\cH$ with respective inner products $\inner{\cdot}{\cdot}_{\cG}$ and $\inner{\cdot}{\cdot}_{\cH}$, we denote by $\cL(\cG,\cH)$ the space of all bounded linear operators from $\cG$ to $\cH$, with the operator norm $\norm{A}:=\sup_{\norm{g}_\cG \leq 1}\norm{Ag}_\cH$. We also denote, for any $A \in \cL(\cG,\cH)$, by $A^{\top}$ its adjoint, which is the unique operator such that $\inner{A^{\top}h}{g}_{\cG} =\inner{h}{Ag}_{\cH}$ for all $g \in \cG$, $h \in \cH$. In the case $\cG=\cH$, we denote $\cL(\cH)=\cL(\cH,\cH)$. We now review the following lemma \citep{rasmussen2003gaussian} about operators, which we will use several times.
\begin{mylemma}[Operator identities]
   \label{lem:dim-change}
   Let $A \in \cL(\cG,\cH)$. Then, for any $\eta > 0$, the following hold
   \beqan
   (A^{\top} A + \eta I)^{-1}A^{\top}&=&A^{\top}(AA^{\top}+\eta I)^{-1},\\
   I - A^{\top}(AA^{\top}+\eta I)^{-1}A &=& \eta (A^{\top} A + \eta I)^{-1}.
   \eeqan
   \end{mylemma}

We now present the main result of this appendix, which is stated and proved using the feature map of the multi-task kernel.
\paragraph{Feature map of multi-task kernel}
We assume the multi-task kernel $\Gamma$ to be continuous relative to the operator norm on $\cL(\Real^n)$, the space of bounded linear operators from $\Real^n$ to itself. Then the RKHS $\cH_\Gamma(\cX)$ associated with the kernel $\Gamma$ is a subspace of the space of continuous functions from $\cX$ to $\Real^n$, and hence, $\Gamma$ is a Mercer kernel \citep{carmeli2010vector}.
Let $\mu$ be a probability measure on the (compact) set $\cX$.
Since $\Gamma$ is a Mercer kernel on $\cX$ and $\sup_{x \in \cX}\norm{\Gamma(x,x)}<\infty$, the RKHS $\cH_\Gamma(\cX)$ is a subspace of $L^2(\cX,\mu;\Real^n)$, the Banach space of measurable functions $g:\cX \to \Real^n$ such that $\int_{\cX}\norm{g(x)}^2d\mu(x) < \infty$, with norm $\norm{g}_{L^2} = \left(\int_{\cX}\norm{g(x)}^2d\mu(x)\right)^{1/2}$. Since $\Gamma(x,x) \in \cL(\Real^n)$ is a compact operator\footnote{An operator $A \in \cL(\cH)$ is said to be compact if the image of each bounded set under $A$ is relatively compact.}, by the Mercer theorem for multi-task kernels \citep{carmeli2010vector}, there exists an at most countable sequence $\lbrace(\psi_i,\nu_i)\rbrace_{i \in \Nat}$ such that
\beqn
\begin{split}
  \Gamma(x,x')&=\sum_{i=1}^{\infty}\nu_i \psi_i(x)\psi_i(x')^{\top} \quad \text{and}\\
  \norm{g}_\Gamma^2 &= \sum_{i=1}^{\infty}\frac{\inner{g}{\psi_i}^2_{L^2}}{\nu_i}, \quad g \in L^2(\cX,\mu;\Real^n)\;,
\end{split}
\eeqn
where $\nu_i \geq 0$ for all $i$, $\lim_{i \ra \infty}\nu_i = 0$ and $\lbrace \psi_i:\cX \to \Real^n \rbrace_{i \in \Nat}$ is an orthonormal basis of $L^2(\cX,\mu;\Real^n)$.
In particular $g \in \cH_\Gamma(\cX)$ if and only if $\norm{g}_\Gamma < \infty$. Note that $\lbrace\sqrt{\nu_i}\psi_i\rbrace_{i \in \Nat}$ is an orthonormal basis of $\cH_\Gamma(\cX)$. Then, we can represent the objective function $f \in \cH_{\Gamma}(\cX)$ as 
\beqn
   f= \sum_{i=1}^{\infty}\theta^\star_i\sqrt{\nu_i}\psi_i
\eeqn
for some $\theta^\star:=(\theta^\star_1,\theta^\star_2,\ldots) \in \ell^2$, the Hilbert space of square-summable sequences of real numbers, such that $\norm{f}_\Gamma=\norm{\theta^\star}_2:=\left(\sum_{i=1}^{\infty}|\theta^\star_i|^2\right)^{1/2} < \infty$.
We now define a feature map
$\Phi : \cX \to \cL(\Real^n,\ell^2)$ of the multi-task kernel $\Gamma$ by 
\beqn
\Phi(x) y := \left(\sqrt{\nu_1}\psi_1(x)^{\top}y,\sqrt{\nu_2}\psi_2(x)^{\top}y,\ldots \right), \quad \forall \;x \in \cX, \; y \in \Real^n.
\eeqn
We then have $f(x)=\Phi(x)^{\top}\theta^\star$ and $\Gamma(x,x') = 
\Phi(x)^{\top}\Phi(x')$ for all $x,x' \in \cX$.

\paragraph{Martingale control in $\ell^2$ space}

Let us define $S_t=\sum_{s=1}^{t}\Phi(x_s)\epsilon_s$, where $ \epsilon_1,\ldots,\epsilon_t$ are the random noise vectors in $\Real^n$. Now consider $\cF_{t-1}$, the $\sigma$-algebra generated by the random variables $\lbrace x_s, \epsilon_s\rbrace_{s = 1}^{t-1}$ and $x_t$. Observe that $S_t$ is $\cF_{t}$-measurable and $\expect{S_t\given \cF_{t-1}} = S_{t-1}$. The process $\lbrace S_t \rbrace_{t\geq 1}$ is thus a martingale with values\footnote{We ignore issues of measurability here.} in the $\ell^2$ space. We now define a map $\Phi_{\cX_t} : \ell^2 \ra \Real^{nt}$ by
\beqn
\Phi_{\cX_t} \theta := \left[\left(\Phi(x_1)^{\top}\theta\right)^{\top},\ldots,\left(\Phi(x_t)^{\top}\theta\right)^{\top}\right]^{\top},\quad \forall \; \theta\in \ell^2.
\eeqn 
We also let $V_t:=\Phi_{\cX_t}^{\top}\Phi_{\cX_t}$ be a map from $\ell^2$ to itself and $I$ be the identity operator in $\ell^2$.
In Lemma \ref{lem:martingale-control}, we measure the deviation of $S_t$ by the norm weighted by $(V_t+\eta I)^{-1}$, which is itself derived from $S_t$. Lemma \ref{lem:martingale-control} represents the multi-task generalization of the result of \citet{durand2018streaming}, and we recover their result under the single-task setting ($n=1$).

\begin{mylemma}[Self-normalized martingale control]
    Let the noise vectors $\lbrace\epsilon_t\rbrace_{t \geq 1}$ be $\sigma$-sub-Gaussian. Then, for any $\eta > 0$ and $\delta \in (0,1]$, with probability at least $1-\delta$, the following holds uniformly over all $t \geq 1:$ 
\beqn
    \norm{S_t}_{(V_{t}+\eta I)^{-1}} \leq \sigma \sqrt{2\log\left(1/\delta\right)+ \log\det\left(I+\eta^{-1}V_t\right)}~.
\eeqn
\label{lem:martingale-control}
\end{mylemma}
\begin{proof}
For any sequence of real numbers $\theta=(\theta_1,\theta_2,\ldots)$ such that $\norm{\sum_{i=1}^{\infty}\theta_i\sqrt{\nu_i}\psi_i(x)}_2 < \infty$, let us define $\Phi(x)^{\top}\theta :=\sum_{i=1}^{\infty}\theta_i\sqrt{\nu_i}\psi_i(x)$ and
\beqn
    M_t^{\theta}=\prod_{s=1}^{t}D_s^{\theta},\quad D_s^{\theta} = \exp\left(\frac{\epsilon_s^{\top}\Phi(x_s)^{\top}\theta}{\sigma}-\frac{1}{2}\norm{\Phi(x_s)^{\top}\theta}_2^2\right)~.
\eeqn
Since the noise vectors $\lbrace \epsilon_t \rbrace_{t \geq 1}$ are conditionally $\sigma$-sub-Gaussian, i.e.,
\beqn
\forall \alpha \in \Real^n, \forall t \geq 1,\quad\expect{\exp(\epsilon_t^{\top}\alpha) \given \cF_{t-1}} \leq \exp\left(\sigma^2\norm{\alpha}_2^2/2\right),
\eeqn
we have $\expect{D^{\theta}_t|\mathcal{F}_{t-1}} \leq 1$ and hence $\expect{M_t^{\theta}| \cF_{t-1}} \leq M_{t-1}^{\theta}$. Therefore, it is immediate that $\lbrace M_t^{\theta} \rbrace_{t=0}^{\infty}$ is a non-negative super-martingale and actually satisfies $\expect{M_t^{\theta}} \leq 1$.

Now, let $\tau$ be a stopping time with respect to the filtration $\lbrace \mathcal{F}_t \rbrace_{t=0}^{\infty}$.
By the convergence theorem for non-negative super-martingales, $M_{\infty}^{\theta} = \lim\limits_{t\to\infty}M_t^{\theta}$ is almost surely well-defined, and thus $M_\tau^{\theta}$ is well-defined as well irrespective of whether $\tau < \infty$ or not. Let $Q_t^{\theta} = M_{\min\lbrace \tau,t \rbrace}^{\theta}$ be a stopped version of $\lbrace M_t^{\theta}\rbrace_t$. Then, by Fatou's lemma,
\begin{equation}
    \expect{M_\tau^{\theta}} = \expect{\liminf_{t \to \infty}Q_t^{\theta}}
\leq  \liminf_{t \to \infty} \expect{Q_t^{\theta}} = \liminf_{t \to \infty} \expect{M_{\min\lbrace \tau,t \rbrace}^{\theta}} \leq 1~,
\label{eqn:super-martingale}
\end{equation}
since the stopped super-martingale $\left\lbrace M_{\min\lbrace \tau,t \rbrace}^{\theta}\right\rbrace_{t \geq 1}$ is also a super-martingale.

Let $\cF_{\infty}$ be the $\sigma$-algebra generated by $\lbrace\cF_t\rbrace_{t=0}^{\infty}$, and $\Theta=(\Theta_1,\Theta_2,\ldots)$, $\Theta_i \sim \cN(0,1/\eta)$ be an infinite i.i.d. Gaussian random sequence which is independent of $\cF_\infty$. Since $\Gamma(x,x) \in \cL(\Real^n)$ has finite trace, we have
\begin{equation*}
    \expect{\norm{\sum_{i=1}^{\infty}\Theta_i\sqrt{\nu_i}\psi_i(x)}^2_2} = \frac{1}{\eta}\sum_{i=1}^{\infty}\nu_i\norm{\psi_i(x)}_2^2= \frac{1}{\eta} \Tr\left(\Gamma(x,x)\right)< \infty~.
\end{equation*}
Therefore,  $\norm{\sum_{i=1}^{\infty}\Theta_i\sqrt{\nu_i}\psi_i(x)}_2 < \infty$ almost surely and thus $M_t^{\Theta}$ is well-defined. Now, thanks to the sub-Gaussian property, $\expect{M_t^{\Theta}|\Theta} \leq 1$ almost surely, and thus $\expect{M_t^{\Theta}} \leq 1$ for all $t$.

Let
$M_t := \expect{M_t^{\Theta}| \mathcal{F}_{\infty}}$ be a mixture of non-negative super-martingales $M_t^{\Theta}$. Then $\lbrace M_t \rbrace_{t=0}^{\infty}$ is also a non-negative super-martingale adapted to the filtration $\{\mathcal{F}_t\}_{t=0}^{\infty}$. Hence, by a similar argument as in (\ref{eqn:super-martingale}), $M_\tau$ is almost surely well-defined and 
$\expect{M_\tau} = \expect{M_\tau^{\Theta}}  \leq 1$.
Let us now compute the mixture martingale $M_t$. We first note for any $\theta \in \ell^2$ that $M_t^{\theta}=\exp \left( \inner{\theta}{S_t/\sigma}_2-\frac{1}{2}\norm{\theta}^2_{V_t}\right)$. The difficulty however lies in the handling of possibly infinite dimension. To this end, we follow \citet{durand2018streaming} to consider the first $d$ dimensions for each $d \in \Nat$. Let $\Theta_d$ denote the restriction of $\Theta$ to the first $d$ components. Thus $\Theta_d\sim\cN(0,\frac{1}{\eta}I_d)$. Similarly, let $S_{t,d}$, $V_{t,d}$ and $M_{t,d}$ denote the corresponding restrictions of $S_t$, $V_t$ and $M_t$, respectively. Following the steps from \citet{chowdhury2017kernelized}, we then obtain that
\begin{align*}
    M_{t,d} &= \frac{\det(\eta I_d)^{1/2}}{(2\pi)^{d/2}}\int_{\Real^d} \exp \left( \inner{\alpha}{S_{t,d}/\sigma}_2-\frac{1}{2}\norm{\alpha}^2_{V_{t,d}}\right)\exp\left(-\frac{\eta}{2}\norm{\alpha}_2^2\right) \mathop{d\alpha} \nonumber\\
    & = \frac{1}{\det(I_d+\eta^{-1}V_{t,d})^{1/2}}\exp\left(\frac{1}{2\sigma^2}\norm{S_{t,d}}^2_{(V_{t,d}+\eta I_d)^{-1}}\right).
\end{align*}
Note that $M_{\tau,d}$ is also almost surely well defined and $\expect{M_{\tau,d}} \leq 1$ for all $d \in \mathbb{N}$. We now fix a $\delta \in (0,1]$. An application of Markov's inequality and Fatou's Lemma then yields
\begin{align*}
\mathbb{P}\left[\norm{S_\tau}^2_{(V_{\tau}+\eta I)^{-1}} > 2\sigma^2\log \left(\frac{\det(I+\eta^{-1}V_{\tau})^{1/2}}{\delta}\right)\right]&=\mathbb{P}\left[ \frac{\exp\left(\frac{1}{2\sigma^2}\norm{S_\tau}^2_{(V_{\tau}+\eta I)^{-1}}\right)}{\frac{1}{\delta} \det(I+\eta^{-1}V_{\tau})^{1/2}}> 1\right]\nonumber\\
    &=\mathbb{P}\left[\lim_{d \to \infty} \frac{\exp\left(\frac{1}{2\sigma^2}\norm{S_{\tau,d}}^2_{(V_{\tau,d}+\eta I_d)^{-1}}\right)}{\frac{1}{\delta} \det(I_d+\eta^{-1}V_{\tau,d})^{1/2}}> 1\right] \nonumber\\&\leq \mathbb{E}\left[\lim_{d \to \infty} \frac{\exp\left(\frac{1}{2\sigma^2}\norm{S_{\tau,d}}^2_{(V_{\tau,d}+\eta I_d)^{-1}}\right)}{\frac{1}{\delta} \det(I_d+\eta^{-1}V_{\tau,d})^{1/2}}\right]\nonumber\\
    & \leq \delta \lim_{d \to \infty}\expect{M_{\tau,d}} \leq \delta~.
\end{align*}
We now define a random stopping time $\tau$ following \citet{chowdhury2017kernelized}, by
 \begin{equation*}
     \tau=\min \left\lbrace t \geq 0: \norm{S_t}^2_{(V_{t}+\eta I)^{-1}} > 2\sigma^2\log \left(\frac{\det(I+\eta^{-1}V_{t})^{1/2}}{\delta}\right)\right \rbrace~.
 \end{equation*}
We then have
\beqn
\mathbb{P}\left[\exists \; t \geq 1:\;\norm{S_t}^2_{(V_{t}+\eta I)^{-1}} > 2\sigma^2\log \left(\frac{\det(I+\eta^{-1}V_{t})^{1/2}}{\delta}\right)\right] = \prob{\tau < \infty} \leq \delta~,
\eeqn
which concludes the proof.
\end{proof}


\subsection{Concentration bound for the estimate (Proof of Theorem \ref{thm:concentration})}
We first reformulate $\mu_t(x)$ in terms of the feature map $\Phi(x)$ as 
\begin{align*}
   \mu_t(x)&= G_t(x)^{\top}\left(G_t+\eta I_{nt}\right)^{-1}Y_t\\
   &=\Phi(x)^{\top}\Phi_{\cX_t}^{\top}\left(\Phi_{\cX_t}\Phi_{\cX_t}^{\top}+\eta I_{nt}\right)^{-1}Y_t\\
   &=\Phi(x)^{\top}\left(\Phi_{\cX_t}^{\top}\Phi_{\cX_t}+\eta I\right)^{-1}\Phi_{\cX_t}^{\top}Y_t\\
   &=\Phi(x)^{\top} \left(V_t+\eta I\right)^{-1}\sum_{s=1}^{t}\Phi(x_s)y_s\nonumber\\
   &=\Phi(x)^{\top} \left(V_t+\eta I\right)^{-1}\sum_{s=1}^{t}\Phi(x_s)(f(x_s)+\epsilon_s)\nonumber\\
   &=\Phi(x)^{\top} \left(V_t+\eta I\right)^{-1}\sum_{s=1}^{t}\Phi(x_s)\left(\Phi(x_s)^{\top}\theta^\star+\epsilon_s\right)\nonumber\\
   &=\Phi(x)^{\top}\theta^{\star}-\eta \Phi(x)^{\top} (V_t+\eta I)^{-1}\theta^{\star}+\Phi(x)^{\top} (V_t+\eta I)^{-1}S_t\\
   &=f(x) + \Phi(x)^{\top} (V_t+\eta I)^{-1}\left(S_t-\eta\theta^\star \right),
\end{align*}
where the third step follows from Lemma \ref{lem:dim-change}. We now obtain, from the definition of operator norm, the following
\begin{align*}
    \norm{f(x)-\mu_t(x)}_2 &\leq \norm{\Phi(x)^{\top}(V_t+\eta I)^{-1/2}} \norm{(V_t+\eta I)^{-1/2}\left(S_t-\eta \theta^\star\right)}_2\\
    &\leq \norm{(V_t+\eta I)^{-1/2}\Phi(x)}\left(\norm{S_t}_{(V_t+\eta I)^{-1}}+\eta \norm{\theta^\star}_{(V_t+\eta I)^{-1}} \right)\nonumber\\
    &\leq \norm{\Phi(x)^{\top}(V_t+\eta I)^{-1}\Phi(x)}^{1/2}\left(\norm{S_t}_{(V_t+\eta I)^{-1}}+\eta^{1/2} \norm{f}_{\Gamma}\right),
\end{align*}
where the last step is controlled as $\norm{\theta^\star}_{(V_t+\eta I)^{-1}} \leq \eta^{-1/2} \norm{\theta^\star}_2=\eta^{-1/2}\norm{f}_{\Gamma}$. A simple application of Lemma \ref{lem:dim-change} now yields
\beqa
\eta \Phi(x)^{\top}(V_t+\eta I)^{-1}\Phi(x)&=& \eta\Phi(x)^{\top}(\Phi_{\cX_t}^{\top}\Phi_{\cX_t}+\eta I)^{-1}\Phi(x)\nonumber\\
&=& \Phi(x)^{\top}\Phi(x)-\Phi(x)^{\top}\Phi_{\cX_t}^{\top}(\Phi_{\cX_t}\Phi_{\cX_t}^{\top}+\eta I_{nt})^{-1}\Phi_{\cX_t}\Phi(x)\nonumber\\
&=&\Gamma(x,x)-G_t(x)^{\top}(G_t+\eta I_{nt})^{-1}G_{t}(x)=\Gamma_t(x,x).
\label{eqn:pred-var-dual}
\eeqa
We then have $\norm{\Phi(x)^{\top}(V_t+\eta I)^{-1}\Phi(x)}^{1/2}= \eta^{-1/2}\norm{\Gamma_t(x,x)}^{1/2}$. We conclude the proof from Lemma \ref{lem:martingale-control} and using Sylvester's identity to get
\beqa
\det\left(I+\eta^{-1}V_t\right)= \det\left(I+\eta^{-1}\Phi_{\cX_t}^{\top}\Phi_{\cX_t}\right)
=\det\left(I_{nt}+\eta^{-1}\Phi_{\cX_t}\Phi_{\cX_t}^{\top}\right)
=\det \left(I_{nt}+\eta^{-1}G_t\right).
\label{eqn:det-dual}
\eeqa

\section{Regret analysis of MT-KB}
\subsection{Properties of predictive variance}
\begin{mylemma}[Sum of predictive variances]
For any $\eta > 0$ and $t \geq 1$,
\begin{align*}
\frac{1}{\eta}\sum_{s=1}^{t}\Tr\left(\Gamma_s(x_s,x_s)\right)&= \log\det\left(I_{nt}+\eta^{-1}G_t\right) = \sum_{s=1}^{t}\log \det\left(I_n+ \eta^{-1}\Gamma_{s-1}(x_s,x_s)\right).
\end{align*}
\label{lem:sum-of-pred-variances}
\end{mylemma}
\begin{proof}
For the first part, we observe from (\ref{eqn:pred-var-dual}) that
\begin{align*}
   \frac{1}{\eta}\sum_{s=1}^{t} \Tr\left(\Gamma_{s}(x_s,x_s)\right)
   &= \sum_{s=1}^{t}\Tr \left(\Phi(x_s)^{\top}(V_s+\eta I)^{-1}\Phi(x_s)\right)\\
   & = \sum_{s=1}^{t}\Tr \left((V_s+\eta I)^{-1}\Phi(x_s)\Phi(x_s)^{\top}\right)\\
   & = \sum_{s=1}^{t}\Tr \left((V_s+\eta I)^{-1}\left((V_s+\eta I)-(V_{s-1}+\eta I)\right)\right)\\
   & \leq \sum_{s=1}^{t}\log \left( \frac{\det(V_s + \eta I)}{\det(V_{s-1}+\eta I)}\right)\\
   &= \log \det\left(I+\eta^{-1}V_t\right)= \log\det \left(I_{nt}+\eta^{-1}G_t\right).
\end{align*}
Here, the last equality follows from (\ref{eqn:det-dual}). The inequality follows from the fact that for two p.d. matrices $A$ and $B$ such that $A-B$ is p.s.d.,
$\Tr \left(A^{-1}( A- B)\right) \leq \log \left( \frac{\det(A)}{\det(B)}\right)$ \citep{calandriello2019gaussian}.

For the second part, we obtain from Schur's determinant identity that
\beqan
&&\det\left(I_{nt}+\eta^{-1}G_t\right)\\&=&\det\left(I_{n(t-1)}+\eta^{-1}G_{t-1}\right) \times\\
&&\hspace{10 pt}\det\left(I_n+\eta^{-1}\Gamma(x_t,x_t)-\eta^{-1}G_{t-1}(x_t)^{\top}\left(I_{n(t-1)}+\eta^{-1}G_{t-1}\right)^{-1}\eta^{-1}G_{t-1}(x_t)\right)\\
&=& \det\left(I_{n(t-1)}+\eta^{-1}G_{t-1}\right) \det\left(I_n+\eta^{-1}\Gamma_{t-1}(x_t,x_t)\right)\\
&=&\ldots\\
&=&\prod_{s=1}^{t} \det\left(I_n+\eta^{-1}\Gamma_{s-1}(x_s,x_s)\right).
\eeqan
We conclude the proof by applying logarithm on both sides.
\end{proof}

\begin{mylemma}[Predictive variance geometry]
Let $\norm{\Gamma(x,x)} \leq \kappa$. Then, for any $\eta > 0$ and $t \geq 1$,
\beqn
    \Gamma_{t}(x,x)\preceq\Gamma_{t-1}(x,x)\preceq \left(1+\kappa/\eta\right)\Gamma_{t}(x,x).
\eeqn
\label{lem:pred-var-inequalities}
\end{mylemma}
\begin{proof}
Let us define $\overline{V}_{t}=V_t+\eta I$ for all $t \geq 0$. We then have from (\ref{eqn:pred-var-dual}) that
   \beqan
   \Gamma_{t}(x,x) &=&\eta\Phi(x)^{\top}\overline{V}_t^{-1}\Phi(x)\\
   &=&\eta\Phi(x)^{\top}\left(\overline{V}_{t-1}+\Phi(x_t)\Phi(x_t)^{\top}\right)^{-1}\Phi(x)\\
   &=& \eta\Phi(x)^{\top}\overline{V}_{t-1}^{-1}\Phi(x)-\\
   &&\quad \quad\eta \Phi(x)^{\top}\overline{V}_{t-1}^{-1}\Phi(x_t)\left(I_n+\Phi(x_t)^{\top}\overline{V}_{t-1}^{-1}\Phi(x_t)\right)^{-1}\Phi(x_t)^{\top}\overline{V}_{t-1}^{-1}\Phi(x)\\
   &=& \Gamma_{t-1}(x,x) - \eta^{-1}\Gamma_{t-1}(x_t,x)^{\top}\left(I_n+\eta^{-1}\Gamma_{t-1}(x_t,x_t)\right)^{-1}\Gamma_{t-1}(x_t,x)\\
   &\preceq& \Gamma_{t-1}(x,x).
   \eeqan
   Here in the third step, we have used the Sherman-Morrison formula and in the last step, we have used the positive semi-definite property of multi-task kernels.
To prove the second part, we first note that
\beqa
\frac{1}{\eta}\Gamma_t(x,x)&=&\Phi(x)^{\top}\left(\overline{V}_{t-1}+\Phi(x_t)\Phi(x_t)^{\top}\right)^{-1}\Phi(x)\nonumber\\
&=&\Phi(x)^{\top}\overline{V}_{t-1}^{-1/2}\left(I+\overline{V}_{t-1}^{-1/2}\Phi(x_t)\Phi(x_t)^{\top}\overline{V}_{t-1}^{-1/2}\right)^{-1}\overline{V}_{t-1}^{-1/2}\Phi(x)\;.
\label{eqn:pred-var-expansion}
\eeqa
Further, since $\norm{\Gamma(x,x)} \leq \kappa$, we have $\lambda_{\max}\left(\Gamma(x,x)\right) \leq \kappa$, and hence,
\beq
\Gamma_{t}(x,x) \preceq \Gamma_{t-1}(x,x) \preceq \Gamma_{t-2}(x,x) \preceq \ldots \Gamma_{0}(x,x)=\Gamma(x,x) \preceq \kappa I_n. 
\label{eqn:pred-var-recursion}
\eeq  
Since $\overline{V}_{t-1}^{-1/2}\Phi(x_t)\Phi(x_t)^{\top}\overline{V}_{t-1}^{-1/2}$ and $\Phi(x_t)^{\top}\overline{V}_{t-1}^{-1}\Phi(x_t)$ have same set of non-zero eigenvalues, we now obtain from (\ref{eqn:pred-var-recursion}) that $\overline{V}_{t-1}^{-1/2}\Phi(x_t)\Phi(x_t)^{\top}\overline{V}_{t-1}^{-1/2} \preceq \frac{\kappa}{\eta}I$. Then (\ref{eqn:pred-var-expansion}) implies that
\beqn
\Gamma_t(x,x) \succeq \eta\Phi(x)^{\top}\overline{V}_{t-1}^{-1}\Phi(x)/\left(1+\kappa/\eta\right)=\Gamma_{t-1}(x,x)/\left(1+\kappa/\eta\right),
\eeqn
which completes the proof.
\end{proof}

\subsection{Regret bound for MT-KB (Proof of Theorem \ref{thm:cumulative-regret})}
Since the scalarization functions $s_\lambda$ is $L_\lambda$-Lipschitz in the $\ell_2$ norm, we have
\beqan
    \left\lvert s_{\lambda_t}\left(f(x)\right) - s_{\lambda_t}\left(\mu_{t-1}(x)\right)\right\rvert \leq L_{\lambda_t}\norm{f(x)-\mu_{t-1}(x)}_2.
\eeqan
Since $\mu_0(x)=0$, $\Gamma_0(x,x)=\Gamma(x,x)$ and $\norm{f}_\Gamma \leq b$, we have
\beqn
\norm{f(x)-\mu_0(x)}_2=\norm{\Gamma_x^{\top}f}_2\leq \norm{f}_\Gamma\norm{\Gamma_x} = \norm{f}_\Gamma\norm{\Gamma_x^{\top}\Gamma_x}^{1/2} \leq b\norm{\Gamma_0(x,x)}^{1/2}.
\eeqn
Then, from Theorem \ref{thm:concentration} and Lemma \ref{lem:sum-of-pred-variances}, the following holds with probability at least $1-\delta$: 
\beq
    \forall t \geq 1, \forall x \in \cX, \quad  \left\lvert s_{\lambda_t}\left(f(x)\right) - s_{\lambda_t}\left(\mu_{t-1}(x)\right)\right\rvert \leq L_{\lambda_t}\beta_{t-1}\norm{\Gamma_{t-1}(x,x)}^{1/2},
    \label{eqn:concentration}
\eeq
where $\beta_t=b+\frac{\sigma}{\sqrt{\eta}}\sqrt{2\log(1/\delta)+\sum_{s=1}^{t}\log \det\left(I_n+\eta^{-1}\Gamma_{s-1}(x_s,x_s)\right)}$, $t \geq 0$.
We can now upper bound
the \emph{instantaneous regret} at time $t \geq 1$ as 
\beqan
r_{\lambda_t}(x_t)&:=&s_{\lambda_t}\left(f(x^\star_{\lambda_t})\right)-s_{\lambda_t}\left(f(x_t)\right)\\
&\leq & s_{\lambda_t}\left(\mu_{t-1}(x^\star_{\lambda_t})\right)+L_{\lambda_t}\beta_{t-1}\norm{\Gamma_{t-1}(x^\star_{\lambda_t},x^\star_{\lambda_t})}^{1/2}-s_{\lambda_t}\left(f(x_t)\right)\\
&\leq& s_{\lambda_t}\left(\mu_{t-1}(x_t)\right)+L_{\lambda_t}\beta_{t-1}\norm{\Gamma_{t-1}(x_t,x_t)}^{1/2}-s_{\lambda_t}\left(f(x_t)\right)\\
&\leq & 2L_{\lambda_t}\beta_{t-1}\norm{\Gamma_{t-1}(x_t,x_t)}^{1/2}.
\eeqan
Here in the first and third step, we have used (\ref{eqn:concentration}). The second step follows from the choice of $x_t$. 
Since $\beta_t$ is a monotonically increasing function in $t$ and $L_{\lambda_t} \leq L$ for all $t$, we have
\beqan
\sum_{t=1}^{T}r_{\lambda_t}(x_t) \leq 2L\beta_{T}\sum_{t=1}^{T}\norm{\Gamma_{t-1}(x_t,x_t)}^{1/2}
\leq 2L\beta_{T}\sqrt{(1+\kappa/\eta)T\sum_{t=1}^{T}\norm{\Gamma_{t}(x_t,x_t)}},
\eeqan
where the last step is due to the Cauchy-Schwartz inequality and Lemma \ref{lem:pred-var-inequalities}. We now obtain from Lemma \ref{lem:sum-of-pred-variances} that $\beta_{T} \leq  b+\frac{\sigma}{\sqrt{\eta}}\sqrt{2\left(\log(1/\delta)+\gamma_{nT}(\Gamma,\eta)\right)}$.
We conclude the proof by taking an expectation over $\lbrace\lambda_i\rbrace_{i=1}^{T} \sim P_{\lambda}$.




\subsection{Inter-task structure in regret for separable kernels (Proof of Lemma \ref{lem:bound-ICM})}
For separable multi-task kernels $\Gamma(x,x')=k(x,x')B$, the kernel matrix is given by $G_T=K_T \otimes B$, where $K_T$ is kernel matrix corresponding to the scalar kernel $k$ and $\otimes$ denotes the Kronecker product. Let $\lbrace\alpha_t\rbrace_{t=1}^{T}$ denote the eigenvalues of $K_T$. Then the eigenvalues of $G_T$ are given by $\alpha_t\xi_i$, $1 \leq t \leq T$, $1 \leq i \leq n$, where $\xi_i$'s are the eigenvalues of $B$. We now have
\beqan
\log\det(I_{nT}+\eta^{-1}G_T)&=&\sum_{t=1}^{T}\sum_{i=1}^{n}\log(1+\alpha_t \xi_i/\eta)\\ 
&=& \sum_{i \in [n]:\xi_i > 0}\sum_{t=1}^{T}\log(1+\alpha_t \xi_i/\eta)\\
&=&\sum_{i \in [n]:\xi_i > 0}\log\det\left(I_T+(\eta/\xi_i)^{-1}K_T\right).
\eeqan
Taking supremum over all possible subsets $\cX_T$ of $\cX$, we then obtain that $\gamma_{nT}(\Gamma,\eta) \leq  \sum_{i \in [n]:\xi_i > 0}\gamma_T(k,\eta/\xi_i)$. 

To prove the second part, we use the feature representation of the scalar kernel $k$. To this end, we let $\phi:\cX \to \ell^2$ be a feature map of the scalar kernel $k$, so that $k(x,x')=\phi(x)^{\top}\phi(x')$ for all $x,x' \in \cX$. We now define a map $\phi_{\cX_t} : \ell^2 \ra \Real^{t}$ by
\beqn
\phi_{\cX_t} \theta := \left[\phi(x_1)^{\top}\theta,\ldots,\phi(x_t)^{\top}\theta\right]^{\top},\quad \forall \; \theta\in \ell^2.
\eeqn 


We also let $v_t:=\phi_{\cX_t}^{\top}\phi_{\cX_t}$ be a map from $\ell^2$ to itself. For any $\alpha > 0$, we then obtain from Lemma \ref{lem:dim-change} that
\beqan
\alpha \;\phi(x)^{\top}(v_t+\alpha I)^{-1}\phi(x)&=& \alpha\;\phi(x)^{\top}(\phi_{\cX_t}^{\top}\phi_{\cX_t}+\alpha I)^{-1}\phi(x)\\
&=& \phi(x)^{\top}\phi(x)-\phi(x)^{\top}\phi_{\cX_t}^{\top}(\phi_{\cX_t}\phi_{\cX_t}^{\top}+\alpha I_t)^{-1}\phi_{\cX_t}\phi(x)\\
&=&k(x,x)-k_t(x)^{\top}(K_t+\alpha I_{t})^{-1}k_{t}(x),
\eeqan
where $k_t(x)=[k(x_1,x),\ldots,k(x_t,x)]^{\top}$ and $K_t=[k(x_i,x_j)]_{1,j=1}^{t}$. We then have from (\ref{eqn:eff-var}) that
\beqan
\norm{\Gamma_t(x,x)} &=& \max_{1 \leq i \leq n} \xi_i \left(k(x,x)-k_t(x)^{\top}\left(K_t+\frac{\eta}{\xi_i}I_t\right)^{-1}k_t(x)\right)\\
&=& \max_{1 \leq i \leq n}\xi_i \cdot \frac{\eta}{\xi_i} \phi(x)^{\top}\left(v_t+\frac{\eta}{\xi_i} I\right)^{-1}\phi(x)\\
& \leq & \eta \;\phi(x)^{\top}\left(v_t+\frac{\eta}{\kappa} I\right)^{-1}\phi(x).
\eeqan
Here, in the last step we have used that $\xi_i \leq \kappa$ for all $i \in [n]$. This holds from our hypothesis $\norm{\Gamma(x,x)} \leq \kappa$ and $k(x,x)=1$. We now observe that $\left(v_t+\frac{\eta}{\kappa}I\right)^{-1} \preceq \left(v_t+\eta I\right)^{-1}$ for $\kappa \leq 1$ and $\left(v_t+\frac{\eta}{\kappa}I\right)^{-1} \preceq \kappa \left(v_t+\eta I\right)^{-1}$ for $\kappa \geq 1$. Therefore
\beqn
\norm{\Gamma_t(x,x)}
\leq \eta\max\lbrace\kappa,1\rbrace  \phi(x)^{\top}\left(v_t+\eta I\right)^{-1}\phi(x). 
\eeqn
A simple application of Lemma \ref{lem:sum-of-pred-variances} for $n=1$ and $\Gamma(\cdot,\cdot)=k(\cdot,\cdot)$ now yields
\beqan
\sum_{t=1}^{T} \norm{\Gamma_t(x,x)} &\leq & \eta \max\lbrace\kappa,1\rbrace \sum_{t=1}^{T} \phi(x_t)^{\top}\left(v_t+\eta I\right)^{-1}\phi(x_t)\\ &=& \eta \max\lbrace\kappa,1\rbrace \log\det\left(I_T+\eta^{-1}K_T\right) \leq 2 \eta \max\lbrace\kappa,1\rbrace \gamma_T(k,\eta),
\eeqan
which completes the proof.

\subsection{Inter-task structure in regret for sum of separable kernels}
We now present a generalization of Lemma \ref{lem:bound-ICM} for multi-task kernels of the form $\Gamma(x,x')=\sum_{j=1}^{M}k_j(x,x')B_j$. This class of kernels is called the sum of separable (SoS) kernel and includes the diagonal kernel $\Gamma(x,x')=\Dg \left(k_1(x,x'),\ldots,k_n(x,x')\right)$ as a special case.

\begin{mylemma}[Inter-task structure in regret for SoS kernel]
 Let $\Gamma(x,x')=\sum_{j=1}^{M}k_j(x,x')B_j$ and $B_j \in \Real^{n\times n}$ be positive semi-definite. Then the following holds:
 \beqan
 \gamma_{nT}(\Gamma,\eta) &\leq & \sum_{j=1}^{M}\rho_{B_j}\max\lbrace\xi_{B_j},1\rbrace\gamma_T(k_j,\eta),\\ \sum_{t=1}^{T}\norm{\Gamma_t(x_t,x_t)} &\leq & 2\eta\sum_{j=1}^{M}\max\lbrace\xi_{B_j},1\rbrace\gamma_T(k_j,\eta),
 \eeqan
 where $\rho_{B_j}$ and $\xi_{B_j}$ denote the rank and the maximum eigenvalue of $B_j$, respectively and $\gamma_T(k_j)$ is the maximum information gain corresponding to scalar kernel $k_j$. Moreover, if $\Gamma(x,x')=\Dg \left(k_1(x,x'),\ldots,k_n(x,x')\right)$ and each $k_j$ is a stationary kernel, then
 \beqn
 \gamma_{nT}(\Gamma,\eta)\leq \sum_{j=1}^{n}\gamma_T(k_j,\eta), \quad \quad \sum_{t=1}^{T}\norm{\Gamma_t(x_t,x_t)} \leq 2\eta\;\max_{1 \leq j \leq n}\gamma_T(k_j,\eta).
 \eeqn
\label{lem:cumulative-regret-SoS-kernel}
\end{mylemma}

\begin{proof}
We let, for each scalar kernel $k_j$, a feature map $\phi_j:\cX \to \ell^2$, so that $k_j(x,x')=\phi_j(x)^{\top}\phi_j(x')$. We now define the feature map
$\Phi : \cX \to \cL(\Real^n,\ell^2)$ of the multi-task kernel $\Gamma(x,x')=\sum_{j=1}^{M}k_j(x,x')B_j$ by 
\beqn
\Phi(x)y := \left(\phi_1(x)\otimes B_1^{1/2}y,\ldots,\phi_M(x)\otimes B_M^{1/2}y\right), \quad \forall \; x \in \cX,\;y \in \Real^n\;,
\eeqn
with the inner product
\beqn
\Phi(x)^{\top}\Phi(x'):=\sum_{j=1}^{M}\left(\phi_j(x)\otimes B_j^{1/2}\right)^{\top}\left(\phi_j(x')\otimes B_j^{1/2}\right)=\sum_{j=1}^{M}\phi_j(x)^{\top}\phi_j(x')\cdot B_j.
\eeqn
We then have
\beqn
V_t := \sum_{s=1}^{t}\Phi(x_s)\Phi(x_s)^{\top}=\sum_{s=1}^{t}\sum_{j=1}^{M}\phi_j(x_s)\phi_j(x_s)^{\top} \otimes B_j = \sum_{j=1}^{M}v_{t,j}\otimes B_j,
\eeqn
where $v_{t,j}:=\sum_{s=1}^{t}\phi_j(x_s)\phi_j(x_s)^{\top}$. We further obtain from (\ref{eqn:pred-var-dual}) that
\beqan
\Gamma_t(x,x) = \sum_{j=1}^{M}\eta \left(\phi_j(x)\otimes B_j^{1/2}\right)^{\top}\left(\sum_{j=1}^{M}v_{t,j}\otimes B_j+\eta I\right)^{-1}\left(\phi_j(x)\otimes B_j^{1/2}\right).
\eeqan
Now each $B_j$ is a positive semi-definite matrix and so is $v_{t,j}\otimes B_j$. Hence, for for all $j \in [M]$, $\left(\sum_{j=1}^{M}v_{t,j}\otimes B_j+\eta I\right)^{-1} \preceq \left(v_{t,j}\otimes B_j+\eta I \right)^{-1}$. Therefore
\beq
\Gamma_t(x,x) \preceq \sum_{j=1}^{M}\eta \left(\phi_j(x)\otimes B_j^{1/2}\right)^{\top}\left(v_{t,j}\otimes B_j+\eta I\right)^{-1}\left(\phi_j(x)\otimes B_j^{1/2}\right)=\sum_{j=1}^{M}\Gamma_{t,j}(x,x),
\label{eqn:SoS}
\eeq
where $\Gamma_{t,j}(x,x):=\eta \left(\phi_j(x)\otimes B_j^{1/2}\right)^{\top}\left(v_{t,j}\otimes B_j+\eta I\right)^{-1}\left(\phi_j(x)\otimes B_j^{1/2}\right)$. Now, let $(\xi_{j,i},u_{j,i})$ denotes the $i$-th eigenpair of $B_j$. A similar argument as in (\ref{eqn:inverse}) then yields 
\beqn
\left(v_{t,j}\otimes B_j+\eta I\right)^{-1}=\sum_{i=1}^{n}\left(\xi_{j,i}v_{t,j}+\eta I\right)^{-1} \otimes u_{j,i}u_{j,i}^{\top}\;.
\eeqn
We then have from the mixed product property of Kronecker product and the orthonormality of $\lbrace u_{j,i} \rbrace_{i=1}^{n}$ that
\beqan
\Gamma_{t,j}(x,x) &=& \sum_{i=1}^{n}\eta \;\xi_{j,i}\phi_j(x)^{\top}\left(\xi_{j,i}v_{t,j}+\eta I\right)^{-1}\phi_j(x)\cdot u_{j,i}u_{j,i}^{\top}\\
&=&\sum_{i=1}^{n}\eta\; \phi_j(x)^{\top}\left(v_{t,j}+\frac{\eta}{\xi_{j,i}}I\right)^{-1}\phi_j(x)\cdot u_{j,i}u_{j,i}^{\top}\;.
\eeqan
Since $\left(v_t+\frac{\eta}{\xi_{j,i}}I\right)^{-1} \preceq \left(v_t+\eta I\right)^{-1}$ for $\xi_{j,i} \leq 1$ and $\left(v_t+\frac{\eta}{\xi_{j,i}}I\right)^{-1} \preceq \xi_{j,i}\left(v_t+\eta I\right)^{-1}$ for $\xi_{j,i} \geq 1$, we now have
\beqan
\Tr(\Gamma_{t,j}(x,x))
& \leq & \eta \sum\limits_{i \in [n] : \xi_{j,i} > 0}\max\lbrace\xi_{j,i},1\rbrace  \phi_j(x)^{\top}\left(v_{t,j}+\eta I\right)^{-1}\phi_j(x)\\
& \leq & \eta \;\rho_{B_j}\max\lbrace\xi_{B_j},1\rbrace \phi_j(x)^{\top}\left(v_{t,j}+\eta I\right)^{-1}\phi_j(x).
\eeqan
Similarly
\beqan
\norm{\Gamma_{t,j}(x,x)}
& \leq & \eta \; \max\limits_{1 \leq i \leq n}\max\lbrace\xi_{j,i},1\rbrace \phi_j(x)^{\top}\left(v_{t,j}+\eta I\right)^{-1}\phi_j(x)\\
& \leq & \eta \max\lbrace\xi_{B_j},1\rbrace  \phi_j(x)^{\top}\left(v_{t,j}+\eta I\right)^{-1}\phi_j(x). 
\eeqan
Let $K_{T,j}=[k_j(x_p,x_q)]_{p,q=1}^{T}$ denotes the kernel matrix corresponding to the scalar kernel $k_j$. An application of Lemma \ref{lem:sum-of-pred-variances} for $n=1$ and $\Gamma(\cdot,\cdot)=k_j(\cdot,\cdot)$ now yields
\beqan
\sum_{t=1}^{T}\Tr\left(\Gamma_{t,j}(x_t,x_t)\right) &\leq& \eta \; \rho_{B_j}\max\lbrace\xi_{B_j},1\rbrace\log\det\left(I_T+\eta^{-1}K_{T,j}\right)  \quad \text{and}\\
\sum_{t=1}^{T}\norm{\Gamma_{t,j}(x_t,x_t)} &\leq& \eta\; \max\lbrace \xi_{B_j},1\rbrace \log\det\left(I_T+\eta^{-1}K_{T,j}\right).
\eeqan
We then have from (\ref{eqn:SoS}) and Lemma \ref{lem:sum-of-pred-variances} that
\beqan
\log\det\left(I_{nT}+\eta^{-1}G_T\right)&=&\frac{1}{\eta}\sum_{t=1}^{T}\Tr\left(\Gamma_t(x_t,x_t)\right) \\&\leq& \frac{1}{\eta}\sum_{j=1}^{M}\sum_{t=1}^{T}\Tr\left(\Gamma_{t,j}(x_t,x_t)\right)\\
&\leq& \sum_{j=1}^{M} \rho_{B_j}\max\lbrace\xi_{B_j},1\rbrace \log\det\left(I_T+\eta^{-1}K_{T,j}\right).
\eeqan
Taking supremum over all possible subsets $\cX_T$ of $\cX$, we now obtain that $\gamma_{nT}(\Gamma,\eta)\leq \sum_{j=1}^{M}\rho_{B_j}\max\lbrace\xi_{B_j},1\rbrace\gamma_T(k_j,\eta)$.
We further have from (\ref{eqn:SoS}) that
\beqn
\sum_{t=1}^{T}\norm{\Gamma_{t}(x_t,x_t)}\leq \sum_{j=1}^{M}\sum_{t=1}^{T}\norm{\Gamma_{t,j}(x_t,x_t)} \leq 2\eta\sum_{j=1}^{M}\max\lbrace \xi_{B_j},1\rbrace\gamma_T(k_j,\eta),
\eeqn
which completes the proof for the first part.

For the diagonal kernel, $M=n$ and each $B_j$ is a diagonal matrix with $1$ in the $j$-th diagonal entry and $0$ in all others. In this case, we have
\beqn
\Gamma_t(x,x)=\eta \sum_{j=1}^{n}\phi_j(x)^{\top}\left(v_{t,j}+\eta I\right)^{-1}\phi_j(x)\cdot B_j\;.
\eeqn
We then have from Lemma \ref{lem:sum-of-pred-variances} that
\beqan
\log\det\left(I_{nT}+\eta^{-1}G_T\right)&=&\frac{1}{\eta}\sum_{t=1}^{T}\Tr\left(\Gamma_t(x_t,x_t)\right) \\&=& \sum_{t=1}^{T}\sum_{j=1}^{n}\phi_j(x_t)^{\top}\left(v_{t,j}+\eta I\right)^{-1}\phi_j(x_t)\cdot\Tr\left(B_j\right)\\
&=& \sum_{j=1}^{n}\sum_{t=1}^{T}\phi_j(x_t)^{\top}\left(v_{t,j}+\eta I\right)^{-1}\phi_j(x_t)\\&=&
\sum_{j=1}^{n}  \log\det\left(I_T+\eta^{-1}K_{T,j}\right).
\eeqan
Taking supremum over all possible subsets $\cX_T$ of $\cX$, we now obtain that $\gamma_{nT}(\Gamma,\eta)\leq \sum_{j=1}^{n}\gamma_T(k_j,\eta)$. We further have
\beqn
\norm{\Gamma_t(x,x)}= \max_{1 \leq j \leq n} \eta\;\phi_j(x)^{\top}\left(v_{t,j}+\eta I\right)^{-1}\phi_j(x)\;.
\eeqn
Let $j^{\star}(x)= \argmax_{1 \leq j \leq n}k_j(x,x)$. Since each $k_j$ is stationary, i.e., $k_j(x,x')=k_j(x-x')$, we have $j^\star(x)$ is independent of $x$. We now let $j^\star=j^\star(x)$ for all $x$. Then it can be easily checked that
\beqn
\norm{\Gamma_t(x,x)}=  \eta\;\phi_{j^\star}(x)^{\top}\left(v_{t,j^\star}+\eta I\right)^{-1}\phi_{j^\star}(x)\;.
\eeqn
We now obtain from Lemma \ref{lem:sum-of-pred-variances} that
\beqan
\sum_{t=1}^{T}\norm{\Gamma_t(x_t,x_t)}&=&  \eta\sum_{t=1}^{T}\phi_{j^\star}(x_t)^{\top}\left(v_{t,j^\star}+\eta I\right)^{-1}\phi_{j^\star}(x_t)\\
&=&\eta\; \log\det\left(I_T+\eta^{-1}K_{T,j^\star}\right)\leq 2\eta\;\max_{1 \leq j \leq n} \gamma_T(k_j,\eta)\;,
\eeqan
which completes the proof for the second part.
\end{proof}

\section{Analysis of MT-BKB}

\paragraph{Trading-off approximation accuracy and size} Given a dictionary $\cD_t=\lbrace x_{i_1},\ldots,x_{i_{m_t}}\rbrace$, we define a map $\Phi_{\cD_t}: \ell^2 \ra \Real^{nm_t}$ by
\beq
\Phi_{\cD_t} \theta := \left[\frac{1}{\sqrt{p_{t,i_1}}}\left(\Phi(x_{i_1})^{\top}\theta\right)^{\top},\ldots,\frac{1}{\sqrt{p_{t,i_{m_t}}}}\left(\Phi(x_{i_{m_t}})^{\top}\theta\right)^{\top}\right]^{\top},\quad \forall \; \theta \in \ell^2,
\label{eqn:dictionary-map}
\eeq
where $p_{t,i_j}=\min \left\lbrace q\norm{\tilde{\Gamma}_{t-1}(x_{i_j},x_{i_j})},1 \right\rbrace$ for all $j \in [m_t]$.
\begin{mylemma}[Approximation properties]
For any $T \geq 1$, $\epsilon \in (0,1)$ and $\delta \in (0,1]$, set $\rho=\frac{1+\epsilon}{1-\epsilon}$ and $q=\frac{6\rho\ln(2T/\delta)}{\epsilon^2}$. Then, for any $\eta > 0$, with probability at least $1-\delta$, the following hold uniformly over all $t\in [T]:$
\beqan
(1-\epsilon)\Phi_{\cX_t}^{\top}\Phi_{\cX_t}-\epsilon\eta I &\preceq &  \Phi_{\cD_t}^{\top}\Phi_{\cD_t}\preceq (1+\epsilon)\Phi_{\cX_t}^{\top}\Phi_{\cX_t}+\epsilon\eta I\;,\\
m_t &\leq & 6\rho q \left(1+\kappa/\eta\right)\sum_{s=1}^{t}\norm{\Gamma_s(x_s,x_s)}.
\eeqan
\label{lem:dictionary-prop}
\end{mylemma}
\begin{proof}
Let $S_t$ be an $nt$-by-$nt$ block diagonal matrix with $i$-th diagonal block $[S_{t}]_i=\frac{1}{\sqrt{p_{t,i}}}I_n$ if $x_i \in \cD_t$, and $[S_{t}]_i=0$ if $x_i \notin \cD_t$, $1 \leq i \leq t$. We then have $\Phi_{\cD_t}^{\top}\Phi_{\cD_t}= \Phi_{\cX_t}^{\top}S_t^{\top}S_t\Phi_{\cX_t}$. The proof now can be completed by following \citet[Theorem 1]{calandriello2019gaussian}.
\end{proof}

\begin{remark} Note that although tuning the approximation trade-off parameter $q$ requires the knowledge of the time horizon $T$ in advance, Lemma \ref{lem:dictionary-prop} is quite robust to the uncertainty on $T$. If the horizon is not known, then after the $T$-th step, one can increase $q$ according to the new desired horizon, and update the dictionary with this new value of $q$. Combining this with a standard doubling trick preserve the approximation properties \citep{calandriello2019gaussian}.
\end{remark}

\paragraph{Approximating the confidence set}
We now focus on the dictionary $\cD_t$ chosen by MT-BKB at each step and discuss a principled approach to compute the approximations $\tilde{\mu}_t(x)$ and $\tilde{\Gamma}_t(x,x)$.
To this end, we let 
\beq
P_t=\Phi_{\cD_t}^{\top}\left(\Phi_{\cD_t}\Phi_{\cD_t}^{\top}\right)^{+}\Phi_{\cD_t}
\label{eqn:proj-op}
\eeq
denote the symmetric orthogonal projection operator on the subspace of $\cL\left(\Real^n,\ell^2\right)$ that is spanned by $\Phi(x_{i_1}),\ldots,\Phi(x_{i_{m_t}})$. We also let $\hat{\Phi}_t(x)=P_t\Phi(x)$ denote the projection of $\Phi(x)$. We now define a map $\hat{\Phi}_{\cX_t} : \ell^2 \ra \Real^{nt}$ by 
\beqn
\hat{\Phi}_{\cX_t} \theta := \left[\left(\hat{\Phi}_t(x_1)^{\top}\theta\right)^{\top},\ldots,\left(\hat{\Phi}_t(x_t)^{\top}\theta\right)^{\top}\right]^{\top},\quad \forall \; \theta\in \ell^2.
\eeqn
We then have $\hat{\Phi}_{\cX_t}=\Phi_{\cX_t}P_t$ and $\hat{\Phi}_{\cX_t}\hat{\Phi}_{\cX_t}^{\top}=\Phi_{\cX_t}P_t\Phi_{\cX_t}^{\top}$.


\begin{mylemma}[Approximation as given by projection]
Let $\hat{V}_t := \hat{\Phi}_{\cX_t}^{\top}\hat{\Phi}_{\cX_t}$. Then, for any $\eta >0$ and $t\geq 1$, the following holds:
\beqan
\tilde{\mu}_t(x)&=&\Phi(x)^{\top}\left(\hat{V}_t+\eta I\right)^{-1}\sum_{s=1}^{t}\hat{\Phi}_t(x_s)y_s~,\\ \tilde{\Gamma}_t(x,x)&=&\eta \Phi(x)^{\top}\left(\hat{V}_t+\eta I\right)^{-1}\Phi(x)~. 
\eeqan
\label{lem:projection}
\end{mylemma}
\begin{proof}
We first note that
\beqan
\tilde{\Phi}_t(x)^{\top}\tilde{\Phi}_t(x')&=&\tilde{G}_t(x)^{\top}\tilde{G}_t^{+}\tilde{G}_t(x')=\Phi(x)^{\top}P_t\Phi(x').
\eeqan
We now define an $nt \times nm_t$ matrix $\tilde{\Phi}_{\cX_t}=\left[\tilde{\Phi}_t(x_1),\ldots,\tilde{\Phi}_t(x_t)\right]^{\top}$. We then have
\beq
\tilde{\Phi}_{\cX_t}\tilde{\Phi}_t(x)=\Phi_{\cX_t}P_t\Phi(x)=\hat{\Phi}_{\cX_t}\Phi(x),\quad \tilde{\Phi}_{\cX_t}\tilde{\Phi}_{\cX_t}^{\top}=\Phi_{\cX_t}P_t\Phi_{\cX_t}^{\top}=\hat{\Phi}_{\cX_t}\hat{\Phi}_{\cX_t}^{\top},
\label{eqn:projection}
\eeq
where $P_t$ is the projection operator as defined in (\ref{eqn:proj-op}).
We also have $\tilde{V}_t:=\sum_{s=1}^{t}\tilde{\Phi}_t(x_s)\tilde{\Phi}_t(x_s)^{\top}=\tilde{\Phi}_{\cX_t}^{\top}\tilde{\Phi}_{\cX_t}$. Therefore
\begin{align*}
  \tilde{\mu}_t(x)&=\tilde{\Phi}_t(x)^{\top}(\tilde{\Phi}_{\cX_t}^{\top}\tilde{\Phi}_{\cX_t}+\eta I_{nm_t})^{-1}\sum_{s=1}^{t}\tilde\Phi_t(x_s)y_s\\
  &=\tilde{\Phi}_t(x)^{\top}(\tilde{\Phi}_{\cX_t}^{\top}\tilde{\Phi}_{\cX_t}+\eta I_{nm_t})^{-1}\tilde{\Phi}_{\cX_t}^{\top}Y_t\\
  &=\tilde{\Phi}_t(x)^{\top}\tilde{\Phi}_{\cX_t}^{\top}(\tilde{\Phi}_{\cX_t}\tilde{\Phi}_{\cX_t}^{\top}+\eta I_{nt})^{-1}Y_t\\
  &=\Phi(x)^{\top}\hat{\Phi}_{\cX_t}^{\top}(\hat{\Phi}_{\cX_t} \hat{\Phi}_{\cX_t}^{\top}+\eta I_{nt})^{-1}Y_t\\
  &=\Phi(x)^{\top}(\hat{\Phi}_{\cX_t}^{\top}\hat{\Phi}_{\cX_t}+\eta I)^{-1}\hat{\Phi}_{\cX_t}^{\top}Y_t=\Phi(x)^{\top}(\hat{V}_t+\eta I)^{-1}\sum_{s=1}^{t}\hat\Phi_t(x_s)y_s\;,
\end{align*}
where in third and fifth step, we have used Lemma \ref{lem:dim-change}, and in fourth step, we have used (\ref{eqn:projection}). Further
\begin{align*}
 \tilde{\Gamma}_t(x,x)&=\Gamma(x,x)-\tilde{\Phi}_t(x)^{\top}\tilde{\Phi}_t(x)+\eta \tilde{\Phi}_t(x)^{\top} (\tilde{\Phi}_{\cX_t}^{\top}\tilde{\Phi}_{\cX_t}+\eta I_{nm_t})^{-1}\tilde{\Phi}_t(x)\\
 &=\Gamma(x,x)-\tilde{\Phi}_t(x)^{\top}\left(I_{nm_t}-\eta (\tilde{\Phi}_{\cX_t}^{\top}\tilde{\Phi}_{\cX_t}+\eta I_{nm_t})^{-1}\right)\tilde{\Phi}_t(x)\\
 &=\Gamma(x,x)-\tilde{\Phi}_t(x)^{\top}\tilde{\Phi}_{\cX_t}^{\top} (\tilde{\Phi}_{\cX_t}\tilde{\Phi}_{\cX_t}^{\top}+\eta I_{nt})^{-1}\tilde{\Phi}_{\cX_t}\tilde{\Phi}_t(x)\\
 &=\Phi(x)^{\top}\Phi(x)-\Phi(x)^{\top}\hat{\Phi}_{\cX_t}^{\top} (\hat{\Phi}_{\cX_t}\hat{\Phi}_{\cX_t}^{\top}+\eta I_{nt})^{-1}\hat{\Phi}_{\cX_t}\Phi(x)\\
 &=\Phi(x)^{\top}\left(I-\hat{\Phi}_{\cX_t}^{\top} (\hat{\Phi}_{\cX_t}\hat{\Phi}_{\cX_t}^{\top}+\eta I_{nt})^{-1}\hat{\Phi}_{\cX_t}\right)\Phi(x)\\
 &=\eta \Phi(x)^{\top}(\hat{\Phi}_{\cX_t}^{\top}\hat{\Phi}_{\cX_t}+\eta I)^{-1}\Phi(x)=\eta \Phi(x)^{\top}(\hat{V}_t+\eta I)^{-1}\Phi(x),
\end{align*}
where in third and sixth step, we have used Lemma \ref{lem:dim-change}, and in fourth step, we have used (\ref{eqn:projection}).
\end{proof}

\begin{mylemma}[Multi-task concentration under Nystr\"{o}m approximation]
	Let $f \in \cH_\Gamma(\cX)$ and the noise vectors $\lbrace\epsilon_t\rbrace_{t \geq 1}$ be $\sigma$-sub-Gaussian. Further, for any $\eta >0$, $\epsilon \in (0,1)$ and $t \geq 1$, let $(1-\epsilon)\Phi_{\cX_t}^{\top}\Phi_{\cX_t}-\epsilon\eta I \preceq  \Phi_{\cD_t}^{\top}\Phi_{\cD_t}\preceq (1+\epsilon)\Phi_{\cX_t}^{\top}\Phi_{\cX_t}+\epsilon\eta I$. Then, for any $\delta \in (0,1]$, with probability at least $1-\delta$, the following holds uniformly over all $x \in \cX$ and $t \geq 1$:
	\beqn
		\norm{f(x)-\tilde{\mu}_{t}(x)}_2 \leq \left(c_{\epsilon}\norm{f}_\Gamma+\frac{\sigma}{\sqrt{\eta}}\sqrt{2\log(1/\delta)+\log\det(I_{nt}+\eta^{-1}G_t)}\right)\norm{\tilde{\Gamma}_{t}(x,x)}^{1/2},
	\eeqn
	where $c_{\epsilon}=1+\frac{1}{\sqrt{1-\epsilon}}$.
	\label{lem:concentration-kernel-approx}
\end{mylemma}
\begin{proof}
Let us first define $\tilde{\alpha}_t(x):=\Phi(x)^{\top}\left(\hat{V}_t+\eta I\right)^{-1}\sum_{s=1}^{t}\hat{\Phi}_t(x_s)f(x_s)$, where $\hat{V}_t = \hat{\Phi}_{\cX_t}^{\top}\hat{\Phi}_{\cX_t}$. We now note that $f(x)=\Phi(x)^{\top}{\theta^\star}$ and $\tilde{\alpha}_t(x)=\Phi(x)^{\top}\left(\hat{V}_t+\eta I\right)^{-1}\hat{\Phi}_{\cX_t}^{\top}\Phi_{\cX_t}\theta^\star$ for some $\theta^\star \in \ell^2$, so that $\norm{f}_\Gamma=\norm{\theta^\star}_2$. We then have
\beqan
\norm{f(x)-\tilde{\alpha}_t(x)}_2 &=& \norm{\Phi(x)^{\top}\left(\theta^\star-\left(\hat{V}_t+\eta I\right)^{-1}\hat{\Phi}_{\cX_t}^{\top}\Phi_{\cX_t}\theta^\star\right)}_2\\
&\leq & \norm{\Phi(x)^{\top}\left(\hat{V}_t+\eta I\right)^{-1/2}}\norm{\theta^\star-\left(\hat{V}_t+\eta I\right)^{-1}\hat{\Phi}_{\cX_t}^{\top}\Phi_{\cX_t}\theta^\star}_{(\hat{V}_t+\eta I)}\\
&=&\norm{\Phi(x)^{\top}\left(\hat{V}_t+\eta I\right)^{-1}\Phi(x)}^{1/2}\norm{\left(\hat{V}_t+\eta I\right)\theta^\star-\hat{\Phi}_{\cX_t}^{\top}\Phi_{\cX_t}\theta^\star}_{\left(\hat{V}_t+\eta I\right)^{-1}}\\
&=&\eta^{-1/2}\norm{\tilde{\Gamma}_t(x,x)}^{1/2}\norm{\eta \theta^\star-\hat{\Phi}_{\cX_t}^{\top}\left(\Phi_{\cX_t}-\hat{\Phi}_{\cX_t}\right)\theta^\star}_{\left(\hat{V}_t+\eta I\right)^{-1}}\\
&\leq& \eta^{-1/2}\norm{\tilde{\Gamma}_t(x,x)}^{1/2}\left(\eta\norm{\theta^\star}_{\left(\hat{V}_t+\eta I\right)^{-1}}+\norm{\hat{\Phi}_{\cX_t}^{\top}\Phi_{\cX_t}\left(I-P_t\right)\theta^\star}_{\left(\hat{V}_t+\eta I\right)^{-1}}\right)\\
&\leq& \left(\norm{\theta^\star}_2+\eta^{-1/2}\norm{\left(\hat{V}_t+\eta I\right)^{-1/2}\hat{\Phi}_{\cX_t}^{\top}\Phi_{\cX_t}\left(I-P_t\right)\theta^\star}_2\right)\norm{\tilde{\Gamma}_t(x,x)}^{1/2}.
\eeqan
Here in the fourth step, we have used Lemma \ref{lem:projection} and in the second last step, we have used $\hat{\Phi}_{\cX_t}=\Phi_{\cX_t}P_t$, where $P_t$ is the projection operator as defined in (\ref{eqn:proj-op}).
The last step is controlled as $\norm{\theta^\star}_{(\hat{V}_t+\eta I)^{-1}} \leq \eta^{-1/2} \norm{\theta^\star}_2$.
We now have
\beqan
\norm{\left(\hat{V}_t+\eta I\right)^{-1/2}\hat{\Phi}_{\cX_t}^{\top}\Phi_{\cX_t}\left(I-P_t\right)\theta^\star}_2 &\leq& \norm{\left(\hat{V}_t+\eta I\right)^{-1/2}\hat{\Phi}_{\cX_t}^{\top}} \norm{\Phi_{\cX_t}\left(I-P_t\right)}\norm{\theta^\star}_2\\
&\leq& \norm{\Phi_{\cX_t}(I-P_t)\Phi_{\cX_t}^{\top}}^{1/2}\norm{\theta^\star}_2,
\eeqan
where we have used that $\norm{({\hat{V}_t+\eta I)^{-1/2}}\hat{\Phi}_{\cX_t}^{\top}} = \norm{\hat{\Phi}_{\cX_t}(\hat{\Phi}_{\cX_t}^{\top}\hat{\Phi}_{\cX_t}+\eta I)^{-1}\hat{\Phi}_{\cX_t}^{\top}}^{1/2} \leq 1$ and $(I-P_t)^2=I-P_t$.
We now observe from Lemma \ref{lem:dim-change} and our hypothesis $(1-\epsilon)\Phi_{\cX_t}^{\top}\Phi_{\cX_t}-\epsilon\eta I \preceq  \Phi_{\cD_t}^{\top}\Phi_{\cD_t}\preceq (1+\epsilon)\Phi_{\cX_t}^{\top}\Phi_{\cX_t}+\epsilon\eta I$ that
\beqn
I-P_t \preceq I - \Phi_{\cD_t}^{\top} (\Phi_{\cD_t}\Phi_{\cD_t}^{\top}+\eta I_{nm_t})^{-1}\Phi_{\cD_t}=\eta (\Phi_{\cD_t}^{\top} \Phi_{\cD_t}+\eta I)^{-1}
\preceq \frac{\eta}{1-\epsilon} (\Phi_{\cX_t}^{\top}\Phi_{\cX_t}+\eta I)^{-1},
\eeqn
and therefore,
$\norm{\Phi_{\cX_t}(I-P_t)\Phi_{\cX_t}^{\top}}^{1/2} \leq \sqrt{\frac{\eta}{1-\epsilon}}\norm{\Phi_{\cX_t}(\Phi_{\cX_t}^{\top}\Phi_{\cX_t}+\eta I)^{-1}\Phi_{\cX_t}^{\top}}^{1/2} \leq \sqrt{\frac{\eta}{1-\epsilon}}$.
Putting it all together, we now have
\beq
\norm{f(x)-\tilde{\alpha}_t(x)}_2 \leq \norm{\theta^\star}_2 \left(1+\frac{1}{\sqrt{1-\epsilon}}\right)\norm{\tilde{\Gamma}_t(x,x)}^{1/2}=c_{\epsilon}\norm{f}_\Gamma \norm{\tilde{\Gamma}_t(x,x)}^{1/2},
\label{eqn:combine-one}
\eeq
where we have used that $\norm{\theta^\star}_2=\norm{f}_\Gamma$ and $c_\epsilon=1+\frac{1}{\sqrt{1-\epsilon}}$.
We further obtain from Lemma \ref{lem:projection} that 
\beqan
\norm{\tilde{\mu}_t(x)-\tilde{\alpha}_t(x)}_2&=&\norm{\Phi(x)^{\top}\left(\hat{V}_t+\eta I\right)^{-1}\sum_{s=1}^{t}\hat{\Phi}_t(x_s)(y_s-f(x_s))}_2\\
&\leq&\norm{\Phi(x)^{\top}\left(\hat{V}_t+\eta I\right)^{-1/2}}\norm{\sum_{s=1}^{t}\hat{\Phi}_t(x_s)\epsilon_s}_{\left(\hat{V}_t+\eta I\right)^{-1}}\\
&=& \norm{\Phi(x)^{\top}\left(\hat{V}_t+\eta I\right)^{-1}\Phi(x)}^{1/2}\norm{\hat{\Phi}_{\cX_t}^{\top}E_t}_{\left(\hat{V}_t+\eta I\right)^{-1}}\\
&=& \eta^{-1/2}\norm{\tilde{\Gamma}_t(x,x)}^{1/2}\norm{\hat{\Phi}_{\cX_t}^{\top}E_t}_{\left(\hat{V}_t+\eta I\right)^{-1}},
\eeqan
where $E_t=\left[\epsilon_1^{\top},\ldots,\epsilon_t^{\top}\right]^{\top}$ denotes an $nt \times 1$ vector formed by concatenating the noise vectors $\epsilon_i,1\leq i \leq t$.
We now have
\beqan
\norm{\hat{\Phi}_{\cX_t}^{\top}E_t}_{\left(\hat{V}_t+\eta I\right)^{-1}}^2 &=& E_t^{\top}\hat{\Phi}_{\cX_t}\left(\hat{\Phi}_{\cX_t}^{\top}\hat{\Phi}_{\cX_t}+\eta I\right)^{-1}\hat{\Phi}_{\cX_t}^{\top}E_t\\
&=&E_t^{\top}\left(I_{nt}-\eta\left(\hat{\Phi}_{\cX_t}\hat{\Phi}_{\cX_t}^{\top}+\eta I_{nt}\right)^{-1}\right)E_t\\
&\leq &E_t^{\top}\left(I_{nt}-\eta\left(\Phi_{\cX_t}\Phi_{\cX_t}^{\top}+\eta I_{nt}\right)^{-1}\right)E_t\\
&=& E_t^{\top}\Phi_{\cX_t}\left(\Phi_{\cX_t}^{\top}\Phi_{\cX_t}+\eta I\right)^{-1}\Phi_{\cX_t}^{\top}E_t
=\norm{\Phi_{\cX_t}^{\top}E_t}_{\left(V_t+\eta I\right)^{-1}}^2,
\eeqan
where in second and fourth step, we have used Lemma \ref{lem:dim-change}, and in third step, we have used $\hat{\Phi}_{\cX_t}\hat{\Phi}_{\cX_t}^{\top}=\Phi_{\cX_t}P_t\Phi_{\cX_t}^{\top} \preceq \Phi_{\cX_t}\Phi_{\cX_t}^{\top}$. 
We then have
\beqa
\norm{\tilde{\mu}_t(x)-\tilde{\alpha}_t(x)}_2 &\leq& \eta^{-1/2}\norm{\sum_{s=1}^{t}\Phi(x_s)\epsilon_s}_{\left(V_t+\eta I\right)^{-1}}\norm{\tilde{\Gamma}_t(x,x)}^{1/2}\nonumber\\
&=& \eta^{-1/2}\norm{S_t}_{\left(V_t+\eta I\right)^{-1}}\norm{\tilde{\Gamma}_t(x,x)}^{1/2},
\label{eqn:combine-two}
\eeqa
where $S_t:=\sum_{s=1}^{t}\Phi(x_s)\epsilon_s$.
Combining (\ref{eqn:combine-one}) and (\ref{eqn:combine-two}) together, we now obtain
\beqan
\norm{f(x)-\tilde{\mu}_t(x)}_2 &\leq & \norm{f(x)-\tilde{\alpha}_t(x)}_2+\norm{\tilde{\alpha}_t(x)-\tilde{\mu}_t(x)}_2\\&\leq& \left(c_{\epsilon}\norm{f}_\Gamma +\eta^{-1/2}\norm{S_t}_{(V_t+\eta I)^{-1}}\right)\norm{\tilde{\Gamma}_t(x,x)}^{1/2}.
\eeqan
We now conclude the proof using Lemma \ref{lem:martingale-control}.
\end{proof}



\paragraph{Preventing variance starvation}
We now show that an accurate dictionary helps us avoid variance starvation in Nystr\"{o}m approximation.

\begin{mylemma}[Predictive variance control]
For any $\eta > 0$ and $\epsilon \in (0,1)$, let $\rho =(1+\epsilon)/(1-\epsilon)$ and $(1-\epsilon)\Phi_{\cX_t}^{\top}\Phi_{\cX_t}-\epsilon\eta I \preceq  \Phi_{\cD_t}^{\top}\Phi_{\cD_t}\preceq (1+\epsilon)\Phi_{\cX_t}^{\top}\Phi_{\cX_t}+\epsilon\eta I$. Then
\beqn
\frac{1}{\rho} \Gamma_t(x,x) \preceq \tilde{\Gamma}_t(x,x) \preceq		 \rho\Gamma_t(x,x).
\eeqn
\label{lem:approx-pred-var-bound}
\end{mylemma}

\begin{proof}
We first note that 
$\hat{\Phi}_{\cX_t}^{\top}\hat{\Phi}_{\cX_t}=P_t\Phi_{\cX_t}^{\top}\Phi_{\cX_t}P_t$, where $P_t$ is the projection operator as defined in (\ref{eqn:proj-op}). Then our hypothesis $(1-\epsilon)\Phi_{\cX_t}^{\top}\Phi_{\cX_t}-\epsilon\eta I \preceq  \Phi_{\cD_t}^{\top}\Phi_{\cD_t}\preceq (1+\epsilon)\Phi_{\cX_t}^{\top}\Phi_{\cX_t}+\epsilon\eta I$ can be re-formulated as
\beqn
\frac{1}{1+\epsilon}P_t\Phi_{\cD_t}^{\top}\Phi_{\cD_t}P_t - \frac{\epsilon\eta }{1+\epsilon} P_t \preceq \hat{\Phi}_{\cX_t}^{\top}\hat{\Phi}_{\cX_t} \preceq  \frac{1}{1-\epsilon}P_t\Phi_{\cD_t}^{\top}\Phi_{\cD_t}P_t + \frac{\epsilon\eta }{1-\epsilon} P_t.
\eeqn 
Since, by definition, $P_t\Phi_{\cD_t}^{\top}=\Phi_{\cD_t}^{\top}$ and $P_t \preceq I$, we have
\beqn
\frac{1}{1+\epsilon}\Phi_{\cD_t}^{\top}\Phi_{\cD_t} - \frac{\epsilon\eta }{1+\epsilon} \preceq \hat{\Phi}_{\cX_t}^{\top}\hat{\Phi}_{\cX_t} \preceq  \frac{1}{1-\epsilon}\Phi_{\cD_t}^{\top}\Phi_{\cD_t} + \frac{\epsilon\eta }{1-\epsilon},
\eeqn 
and, thus, in turn
\beqn
\frac{1}{1+\epsilon}\left(\Phi_{\cD_t}^{\top}\Phi_{\cD_t}+\eta I\right) \preceq \hat{\Phi}_{\cX_t}^{\top}\hat{\Phi}_{\cX_t} +\eta I \preceq  \frac{1}{1-\epsilon} \left(\Phi_{\cD_t}^{\top}\Phi_{\cD_t}+\eta I\right).
\eeqn
We now obtain from our hypothesis that
\beqn
\frac{1-\epsilon}{1+\epsilon}\left(\Phi_{\cX_t}^{\top}\Phi_{\cX_t}+\eta I\right) \preceq \hat{\Phi}_{\cX_t}^{\top}\hat{\Phi}_{\cX_t} +\eta I \preceq  \frac{1+\epsilon}{1-\epsilon} \left(\Phi_{\cX_t}^{\top}\Phi_{\cX_t}+\eta I\right).
\eeqn
This further implies that
\beqn
\frac{1-\epsilon}{1+\epsilon}\Phi(x)^{\top}(V_t+\eta I)^{-1}\Phi(x) \preceq \Phi(x)^{\top}\left(\hat{V}_t +\eta I\right)^{-1}\Phi(x) \preceq  \frac{1+\epsilon}{1-\epsilon} \Phi(x)^{\top}\left(V_t+\eta I\right)^{-1}\Phi(x),
\eeqn
which completes the proof.
\end{proof}

\subsection{Regret bound and dictionary size for MT-BKB (Proof of Theorem \ref{thm:cumulative-regret-kernel-approx})}


Since the scalarization functions $s_\lambda$ is $L_\lambda$-Lipschitz in the $\ell_2$ norm, we have
\beqan
    \left\lvert s_{\lambda_t}\left(f(x)\right) - s_{\lambda_t}\left(\tilde{\mu}_{t-1}(x)\right)\right\rvert \leq L_{\lambda_t}\norm{f(x)-\tilde{\mu}_{t-1}(x)}_2.
\eeqan
Since $\tilde{\mu}_0(x)=0$, $\tilde{\Gamma}_0(x,x)=\Gamma(x,x)$ and $\norm{f}_\Gamma \leq b$, we have
\beqn
\norm{f(x)-\tilde{\mu}_0(x)}_2=\norm{\Gamma_x^{\top}f}_2\leq \norm{f}_\Gamma\norm{\Gamma_x} = \norm{f}_\Gamma\norm{\Gamma_x^{\top}\Gamma_x}^{1/2} \leq b\norm{\tilde{\Gamma}_0(x,x)}^{1/2}.
\eeqn
Further, since $\log(1+ax) \leq a\log(1+x)$ holds for any $a \geq 1$ and $x \geq 0$, we obtain from Lemma \ref{lem:sum-of-pred-variances} and Lemma \ref{lem:approx-pred-var-bound} that
\beqa
\log\det\left(I_{nt}+\eta^{-1}G_t\right)&=&\sum_{s=1}^{t}\log \det\left(I_n+\eta^{-1}\Gamma_{s-1}(x_s,x_s)\right)\nonumber\\&\leq& \rho\sum_{s=1}^{t}\log \det\left(I_n+\eta^{-1}\tilde{\Gamma}_{s-1}(x_s,x_s)\right),
\label{eqn:log-det}
\eeqa
where $\rho=\frac{1+\epsilon}{1-\epsilon}$.
Let us now assume, for any $t \geq 1$, that 
\beq
(1-\epsilon)\Phi_{\cX_t}^{\top}\Phi_{\cX_t}-\epsilon\eta I \preceq  \Phi_{\cD_t}^{\top}\Phi_{\cD_t}\preceq (1+\epsilon)\Phi_{\cX_t}^{\top}\Phi_{\cX_t}+\epsilon\eta I.
\label{eqn:good-event}
\eeq
Then, from (\ref{eqn:log-det}) and Lemma \ref{lem:concentration-kernel-approx}, the following holds with probability at least $1-\delta/2$: 
\beq
    \forall t \geq 1, \forall x \in \cX, \quad  \left\lvert s_{\lambda_t}\left(f(x)\right) - s_{\lambda_t}\left(\tilde{\mu}_{t-1}(x)\right)\right\rvert \leq L_{\lambda_t}\tilde{\beta}_{t-1}\norm{\tilde{\Gamma}_{t-1}(x,x)}^{1/2},
    \label{eqn:concentration-approx}
\eeq
where $\tilde{\beta}_t=c_{\epsilon}b+\frac{\sigma}{\sqrt{\eta}}\sqrt{2\log(2/\delta)+\rho\sum_{s=1}^{t}\log\det\left(I_n+\eta^{-1}\tilde{\Gamma}_{s-1}(x_s,x_s)\right)}$, $t \geq 0$ and $c_{\epsilon}=1+\frac{1}{\sqrt{1-\epsilon}}$.
We can now upper bound
the \emph{instantaneous regret} at time $t \geq 1$ as 
\beqan
r_{\lambda_t}(x_t)&:=&s_{\lambda_t}\left(f(x^\star_{\lambda_t})\right)-s_{\lambda_t}\left(f(x_t)\right)\\
&\leq & s_{\lambda_t}\left(\tilde{\mu}_{t-1}(x^\star_{\lambda_t})\right)+L_{\lambda_t}\tilde{\beta}_{t-1}\norm{\tilde{\Gamma}_{t-1}(x^\star_{\lambda_t},x^\star_{\lambda_t})}^{1/2}-s_{\lambda_t}\left(f(x_t)\right)\\
&\leq& s_{\lambda_t}\left(\tilde{\mu}_{t-1}(x_t)\right)+L_{\lambda_t}\tilde{\beta}_{t-1}\norm{\tilde{\Gamma}_{t-1}(x_t,x_t)}^{1/2}-s_{\lambda_t}\left(f(x_t)\right)\\
&\leq & 2L_{\lambda_t}\tilde{\beta}_{t-1}\norm{\tilde{\Gamma}_{t-1}(x_t,x_t)}^{1/2}.
\eeqan
Here in the first and third step, we have used (\ref{eqn:concentration-approx}). The second step follows from the choice of $x_t$.
Since $\tilde{\beta}_t$ is a monotonically increasing function in $t$ and $L_{\lambda_t} \leq L$ for all $t$, we now have
\beqan
\sum_{t=1}^{T}r_{\lambda_t}(x_t) \leq 2L\tilde{\beta}_T\sum_{t=1}^{T}\norm{\tilde{\Gamma}_{t-1}(x_t,x_t)}^{1/2}
&\leq & 2L\tilde{\beta}_T\sqrt{\rho T\sum_{t=1}^{T}\norm{\Gamma_{t-1}(x_t,x_t)}}\\ &\leq& 2L\tilde{\beta}_T\sqrt{\rho(1+\kappa/\eta) T\sum_{t=1}^{T}\norm{\Gamma_{t}(x_t,x_t)}} ,
\eeqan
where the second last step is due to the Cauchy-Schwartz inequality and Lemma \ref{lem:approx-pred-var-bound}, and the last step is due to Lemma \ref{lem:pred-var-inequalities}.
A similar argument as in (\ref{eqn:log-det}) now yields
\beqan
\sum_{t=1}^{T}\log \det\left(I_n+\eta^{-1}\tilde{\Gamma}_{t-1}(x_t,x_t)\right) &\leq& \rho\sum_{t=1}^{T}\log \det\left(I_n+\eta^{-1}\Gamma_{t-1}(x_t,x_t)\right)\\ &=& \rho\log \det \left(I_{nT}+\eta^{-1}G_T \right)\leq 2\rho\gamma_{nT}(\Gamma,\eta).
\eeqan
We then have $\tilde{\beta}_{T} \leq c_\epsilon b+\frac{\sigma}{\sqrt{\eta}}\sqrt{2\left(\log(2/\delta)+\rho^2\gamma_{nT}(\Gamma,\eta)\right)}$.
Setting $q=\frac{6\rho\ln(4T/\delta)}{\epsilon^2}$, we now have from Lemma \ref{lem:dictionary-prop}, that with probability at least $1-\delta/2$, uniformly across all $t \in [T]$, the dictionary size $m_t \leq  6\rho q \left(1+\kappa/\eta\right)\sum_{s=1}^{t}\norm{\Gamma_s(x_s,x_s)}$ and (\ref{eqn:good-event}) is true. Taking an expectation over $\lbrace\lambda_i\rbrace_{i=1}^{T} \sim P_{\lambda}$ and using a union bound argument, we then obtain, with probability at least $1-\delta$, the cumulative regret
\beqn
		R_C^{\text{MT-BKB}}(T) \leq 2L\left(c_\epsilon b+\frac{\sigma}{\sqrt{\eta}}\sqrt{2\left(\log(1/\delta)+\rho^2\gamma_{nT}(\Gamma,\eta)\right)}\right)\sqrt{\rho(1+\kappa/\eta) T\sum_{t=1}^{T}\norm{\Gamma_{t}(x_t,x_t)}}\;.
	\eeqn
	We conclude the proof by noting that $\rho=\frac{1+\epsilon}{1-\epsilon} > 1$ and $c_{\epsilon}=1+\frac{1}{\sqrt{1-\epsilon}} \leq 2\rho$.

\section{Additional details on experiments}
\label{app:experiments}
\paragraph{Cumulative regret using linear scalarization} We sample from $P_\lambda$ as $\lambda=u/\norm{u}_1$, where $u$ is uniformly sampled from $[0,1]^n$.
We plot the time-average cumulative regret $\frac{1}{T}R_C(T)$ in Figure \ref{fig:plot-linear}.
\begin{figure}[H]
\centering
\subfigure[RKHS function]{\includegraphics[height=1.15in,width=1.75in]{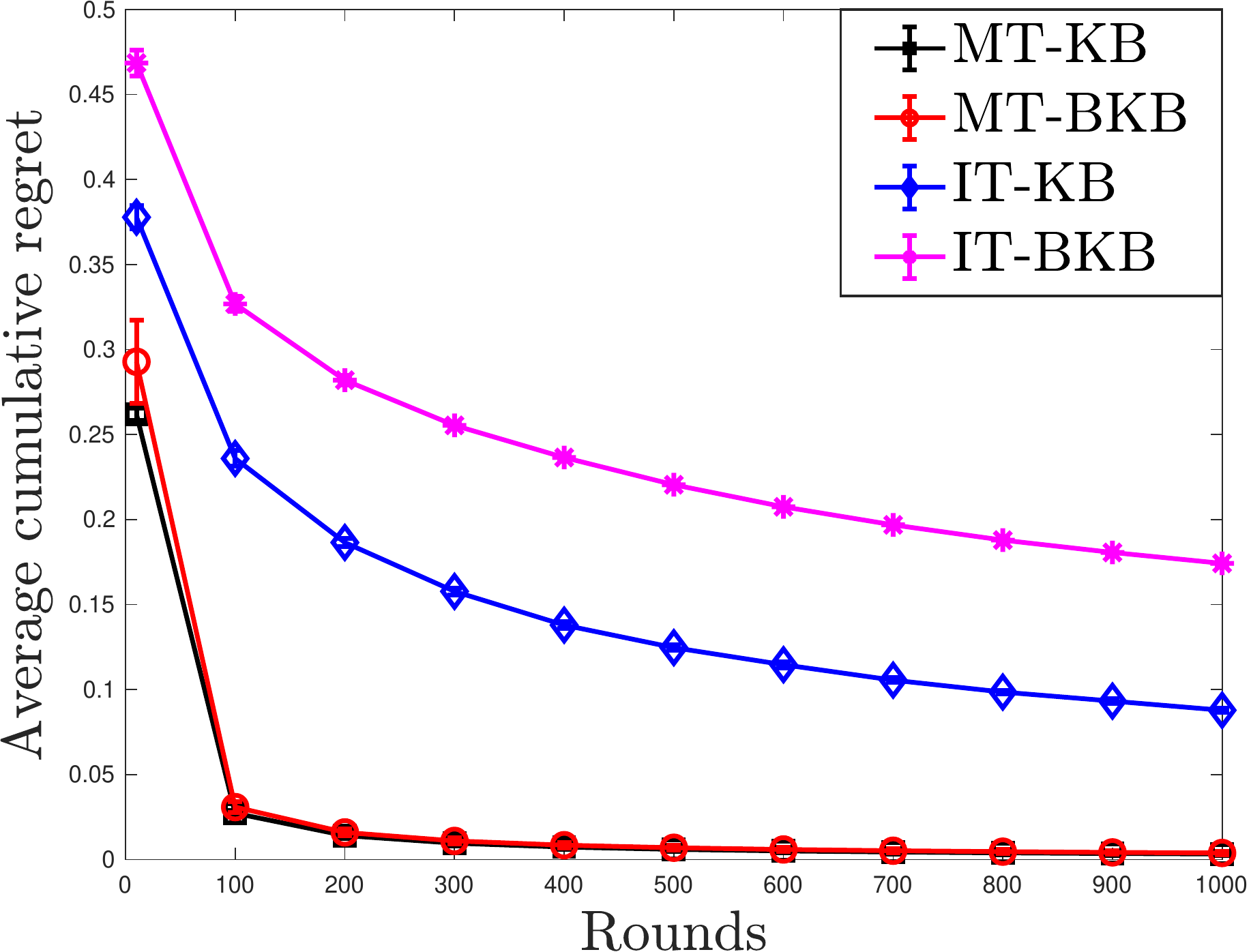}}
\subfigure[Perturbed sine function]{\includegraphics[height=1.15in,width=1.75in]{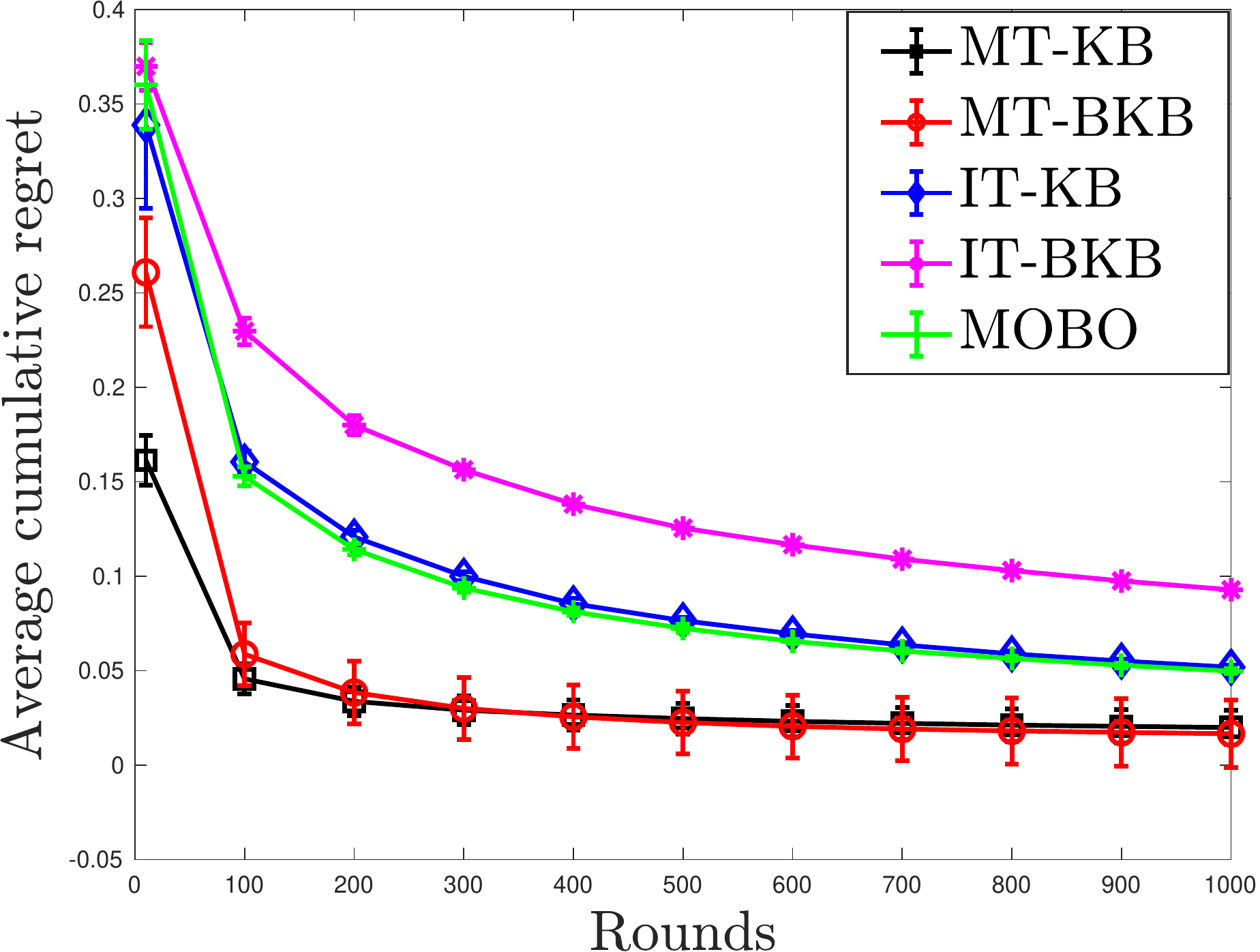}}
\subfigure[Sensor measurements]{\includegraphics[height=1.15in,width=1.75in]{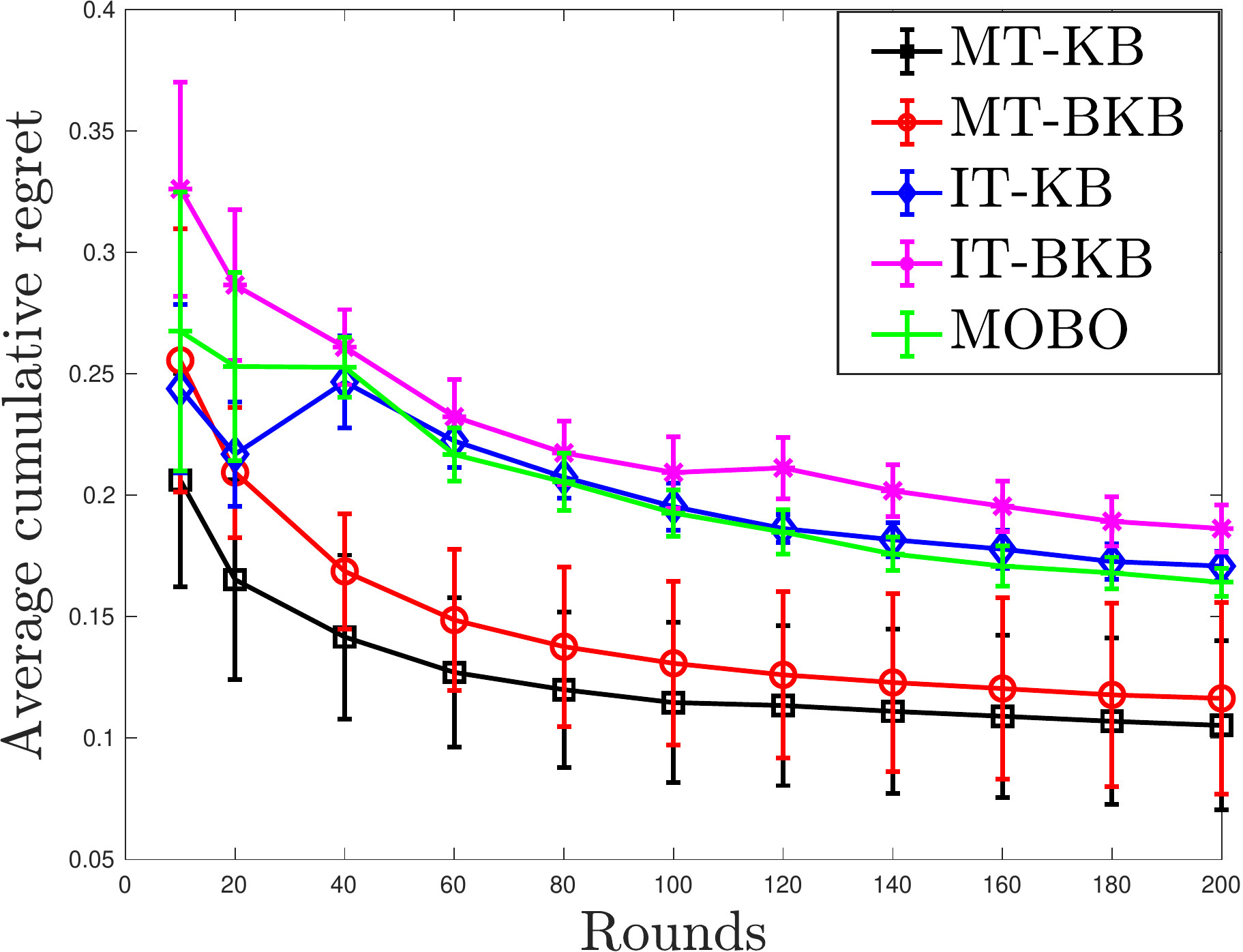}}
\caption{{\footnotesize Comparison of time-average cumulative regret of MT-KB and MT-BKB with IT-KB, IT-BKB and MOBO using linear scalarization.
}}
\label{fig:plot-linear} 
\end{figure}
\vspace*{-5mm}
\paragraph{Comparison of Bayes regret} We compare the Bayes regret $R_B(T)$ of MT-KB and MT-BKB with independent task benchmarks IT-KB, IT-BKB and MOBO using Chebyshev scalarization in Figure \ref{fig:plot-bayes}.
\begin{figure}[H]
\centering
\subfigure[RKHS function]{\includegraphics[height=1.15in,width=1.75in]{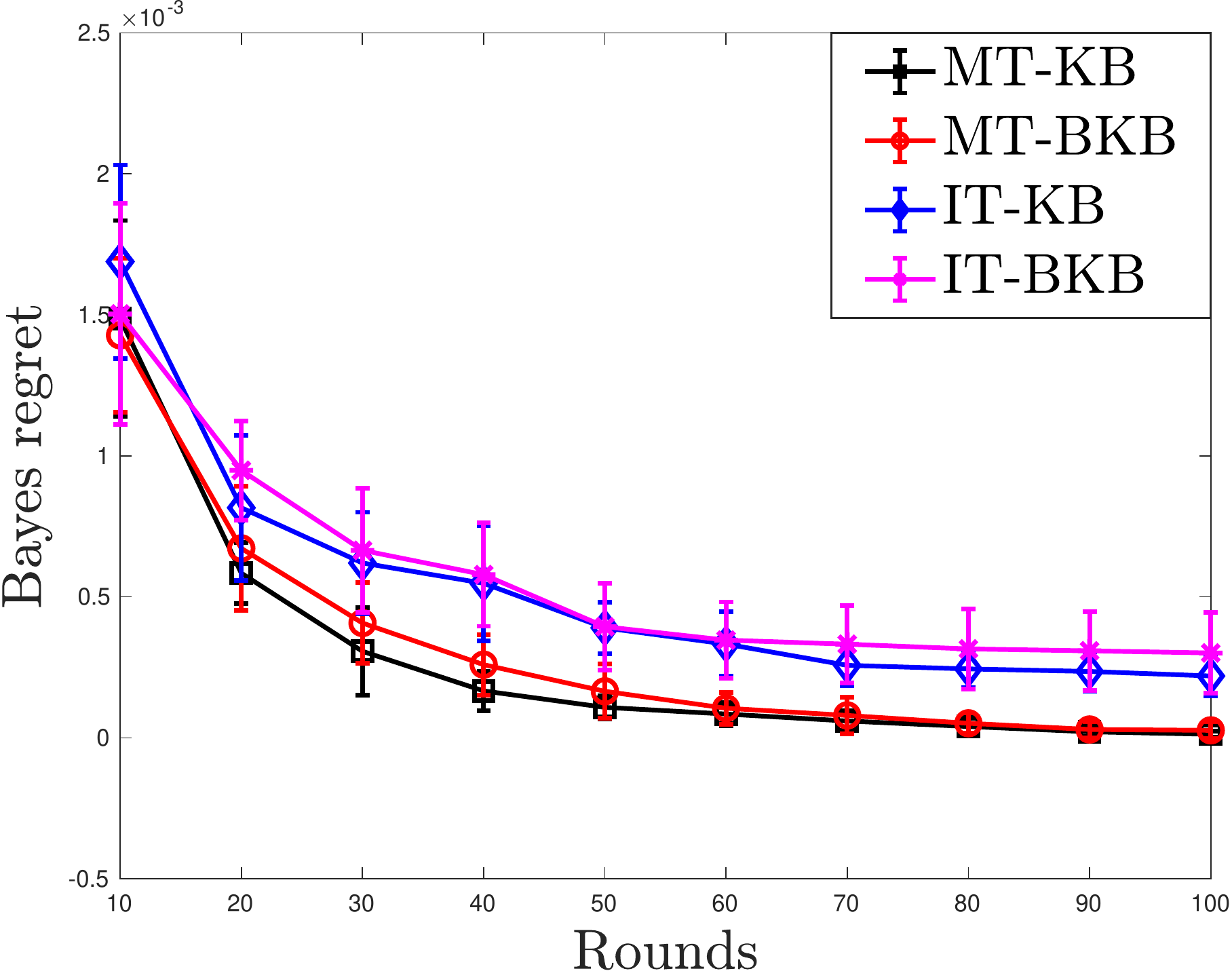}}
\subfigure[Perturbed sine function]{\includegraphics[height=1.15in,width=1.75in]{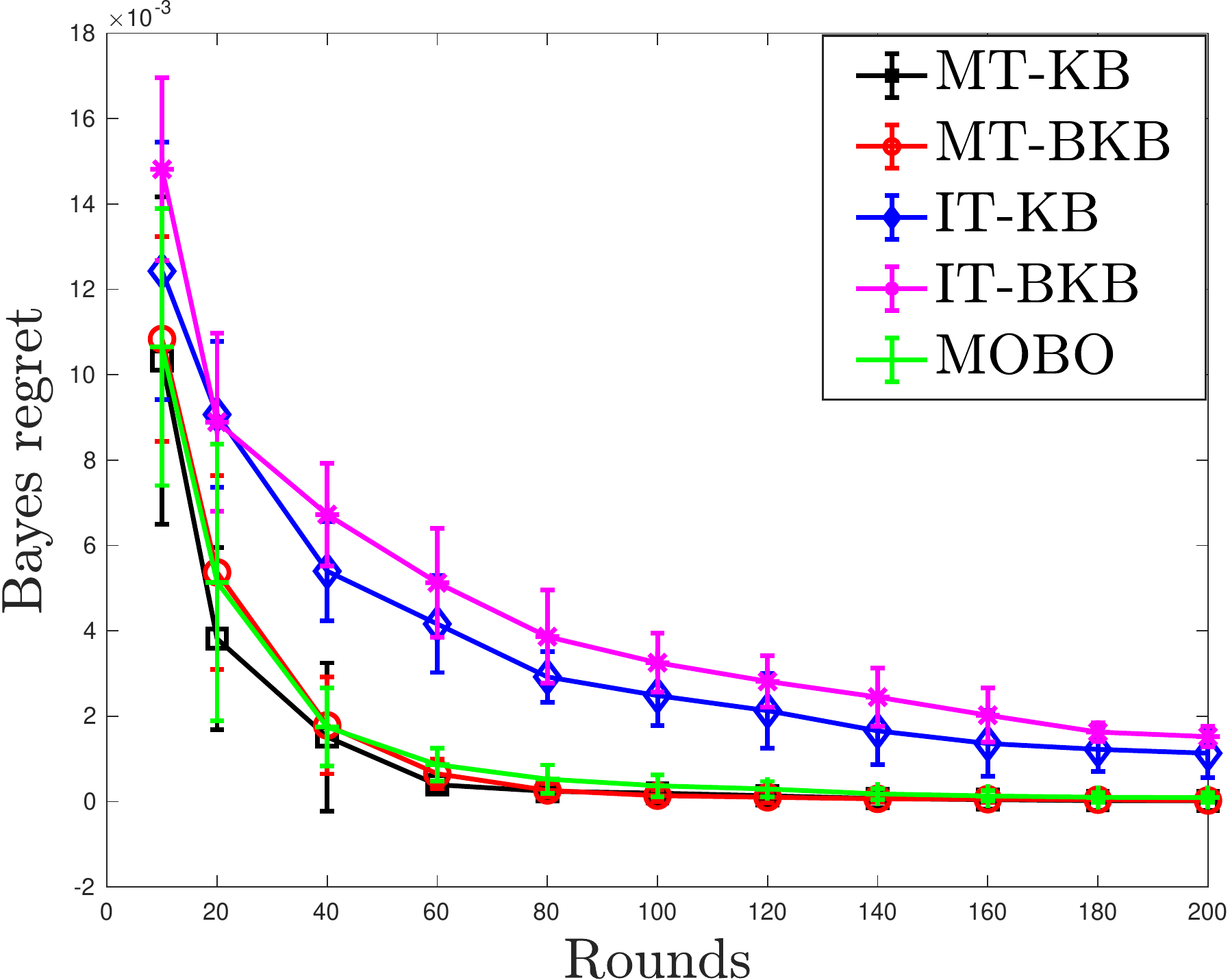}}
\subfigure[Sensor measurements]{\includegraphics[height=1.15in,width=1.75in]{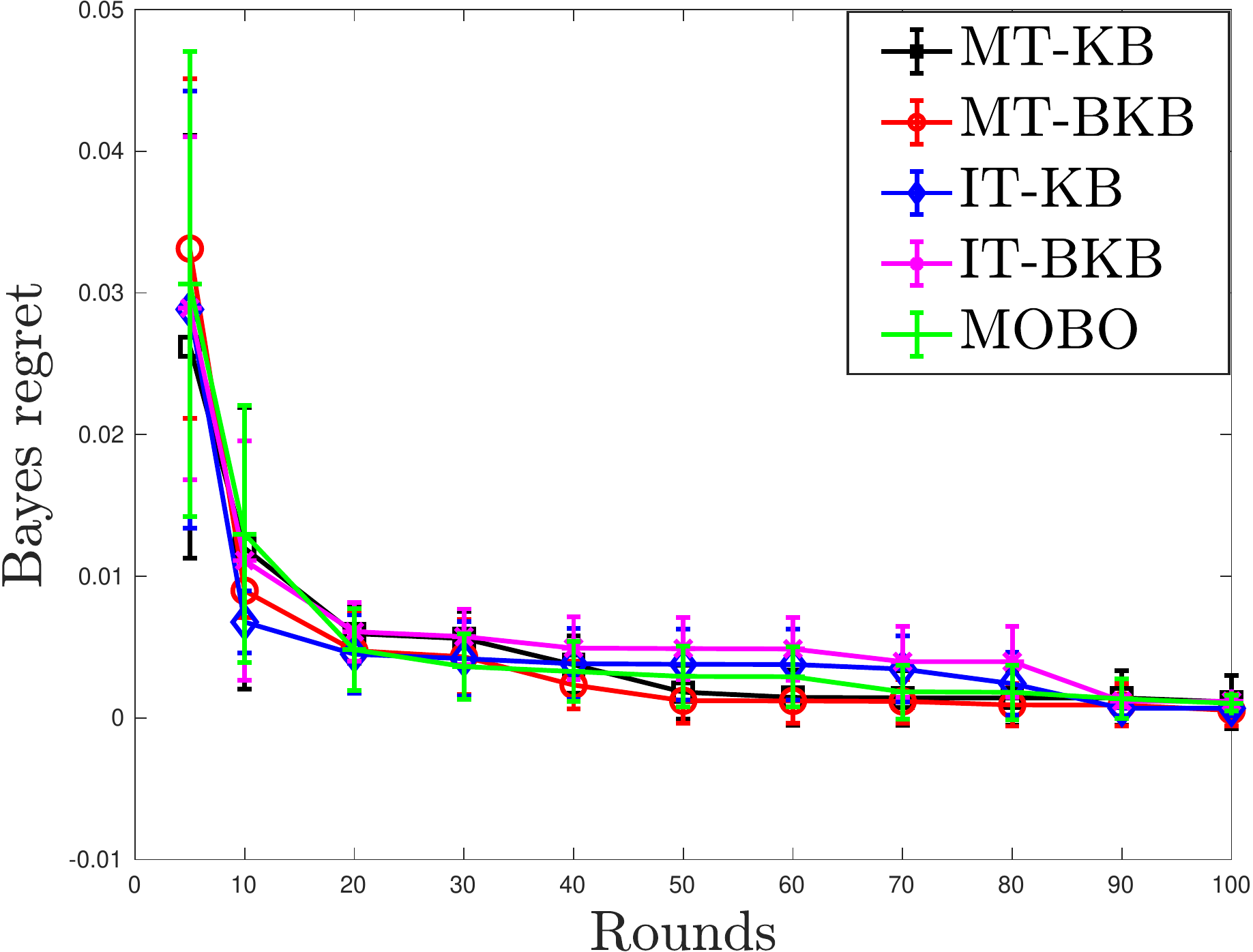}}
\caption{{\footnotesize Comparison of Bayes regret of MT-KB and MT-BKB with IT-KB, IT-BKB and MOBO using Chebyshev scalarization.
}}
\label{fig:plot-bayes} 
\end{figure}
\vspace*{-5mm}
\paragraph{Comments on parameters used} We set the confidence radii (i.e., $\beta_t$ and $\tilde\beta_t$) of MT-KB and MT-BKB exactly as given in Theorem \ref{thm:cumulative-regret} and Theorem \ref{thm:cumulative-regret-kernel-approx}, respectively. Similarly, for IT-KB and IT-BKB, we use respective choices of radii given in \citep{chowdhury2017kernelized} and \citep{calandriello2019gaussian} in the context of single task BO and suitably blow those up by a $\sqrt{n}$ factor to account for $n$ tasks. For MOBO, we use the UCB acquistion function and set the radius as specified in \citep{pariaflexible}. To make the comparison uniform across all experiments, we do not tune any hyper-parameter for any algorithm and for a particular hyperparameter, we always use the same value in all algorithms. The hyper-paramter choices are specified in Section \ref{sec:experiments}. We though believe that careful tuning of hyper-parameters might lead to better performance in practice.
\vspace*{-2mm}
\paragraph{A note on the sensor data} The data was collected at 30 second intervals for 5 consecutive days starting Feb. 28th 2004 from 54 sensors deployed in the Intel Berkeley Research lab. We have downloaded the data previously from the webpage \url{http://db.csail.mit.edu/labdata/labdata}. But the link appears to be broken now. We can share a copy of our downloaded version if asked to do so.
\end{appendix}

\end{document}